\documentclass[twoside]{article}
\usepackage[a4paper]{geometry}
\usepackage[latin1]{inputenc} % ou \usepackage[utf8]{inputenc}
\usepackage[T1]{fontenc} % ou \usepackage[OT1]{fontenc}

\usepackage{RR}

\usepackage{graphicx}
\usepackage{epsfig,color}
\usepackage{amsmath,amssymb,stmaryrd,graphicx,wasysym}
\usepackage{tabularx}
\usepackage{hyperref}
\usepackage{url}
\usepackage{prettyref}
\usepackage{xcolor}
\hypersetup{
    colorlinks = true,
    urlbordercolor = {magenta},
    citecolor = {blue},
}
\usepackage{algorithmicx,algpseudocode}
\usepackage{wrapfig}
\usepackage{pgf}
\usepackage{pgfcore}
\usepackage{pgffor}
\usepackage{paralist}
\usepackage{algorithm}
\usepackage{algorithmicx}
\usepackage{listings}
\RRNo{9447}
\RRdate{January 2022}

\RRetitle{Mixed Nondeterministic-Probabilistic Automata: Blending graphical probabilistic models with nondeterminism}
\RRtitle{Automates Hybrides Nond\'eterministes-Probabilistes: m\'elanger mod\`eles probabilistes graphiques avec le nond\'eterminisme}
\titlehead{Mixed Nondeterministic-Probabilistic Automata}
\RRauthor{Albert Benveniste\thanks{Inria Rennes, Campus de Beaulieu, 35042 Rennes cedex, France; e-mail: prenom.nom@inria.fr} \and Jean-Baptiste Raclet\thanks{IRIT, Universit{\'e} de Toulouse, Toulouse, France; e-mail: Jean-Baptiste.Raclet@irit.fr}}
\authorhead{A. Benveniste, J-B. Raclet}

\newcommand{\inserthere}{\hspace*{-1mm}
\begin{minipage}{4cm}\vspace*{-1.4cm}
\begin{itemize}
	\item[{\color{black}$\uparrow$}] {\color{black}factor graph}
	\item[] \smallskip
	\item[] 
	\item[$\leftarrow$] \bemph{formulas (\ref{sidtycfgdity})}
\end{itemize}
\end{minipage}}
\newtheorem{theorem}{Theorem}
\newtheorem{lemma}[theorem]{Lemma}

\newtheorem{convention}{Convention}
\newtheorem{definition}[theorem]{Definition}
\newenvironment{proof}{\textit{Proof:}~}{}
\newtheorem{discussion}{Discussion}

\newtheorem{example}{Example}

\newcommand{\psys}[1]{\left(\Omega_{#1},\pi_{#1}\right)}
\newcommand{\csys}[1]{C_{#1}}
\newcommand{\qsys}[1]{Q_{#1}}

\def\sg{[\joinrel\mathrel[}
\def\sd{\mathrel]\joinrel]}
\newcommand{\keyw}[1]{{\color{blue}#1}}
\newcommand{\progpara}{\keyw{\,{\|}\,}}
\newcommand{\source}[1]{\mbox{\sf is\_source}(#1)}
\newcommand{\sem}[1]{\sg \mathrel{#1} \sd}
\newcommand{\psem}[1]{\proba_{#1}}

\newcommand{\gsem}[1]{\cN\sem{#1}}
\newcommand{\mem}{p}
\newcommand{\nil}{\mathsf{nil}}

\newcommand{\plang}{{\sf ReactiveBayes}}

\newcommand{\supp}{\mathbf{supp}}

\newcommand{\produces}[2]{#1\,{\leadsto}\,#2}

\newcommand{\parents}{\mbox{\sc pa}}

\newcommand{\Bool}{\mathbb{B}}

\newcommand{\cF}{\mathcal{F}}

\newcommand{\sigalg}{$\sigma$-algebra}
\newcommand{\cons}{C}
\newcommand{\Cons}{\mathcal{C}}

\newcommand{\system}{S}
\newcommand{\bsystem}{\mathcal{S}}

\newcommand{\Systems}{\mathbb{S}}
\newcommand{\Kernels}{\mathbb{\kernel}}

\newcommand{\atom}{\mathbf{a}}

\newcommand{\pa}{\mbox{\textsc{pa}}}
\newcommand{\spa}{\mbox{\textsc{spa}}}

\newcommand{\mmdp}{Mixed Automaton}
\newcommand{\mmdps}{Mixed Automata}

\newcounter{cexample}

\newcommand{\ancestor}[1]{#1^{\uparrow}}

\newcommand{\removalbert}[1]{}
\newcommand{\finremovalbert}[1]{}

\newcommand{\myparagraph}[1]{~\\{\vspace*{-5mm}}{\\}\noindent\textrm{\emph{#1}}:}

\newcommand{\consist}[1]{#1^{\sf c}}

\newcommand{\prog}[1]{{\texttt{#1}}}

\newcommand{\wemph}[1]{{\color{white}#1}}
\newcommand{\bemph}[1]{{\color{blue}#1}}

\newcommand{\outproba}{\overline{\proba}}
\newcommand{\inproba}{\underline{\proba}}

\newcommand{\trace}{\mathbf{r}}

\newcommand{\simu}{\rho}

\newcommand{\NMPlift}[1]{\,#1^\Systems\,}

\newcommand{\prev}[1]{#1^{\!<}}

\newcommand{\trans}[4]{{#1}\stackrel{#2}{\longrightarrow}_{#4}{#3}}
\newcommand{\probatrans}[3]{{#1}\stackrel{#2}{{\leadsto}}{#3}}
\newcommand{\fullprobatrans}[5]{\probatrans{{\trans{#1}{#2}{#3}{#4}}}{}{#5}}

\newcommand{\transindex}[4]{{#1}\stackrel{#2}{\longrightarrow}_{#4}{#3}}

\newcommand{\la}{\leftarrow}

\newcommand{\shortproj}[1]{\mathbf{Pr}_{#1}}
\newcommand{\proj}[2]{\mathbf{Pr}_{#1\!}(#2)}

\newcommand{\margin}[2]{\mathbf{Margin}_{#1}\!\left(#2\right)}
\newcommand{\cond}[2]{\mathbf{Cond}_{#1}\!\left(#2\right)}

\newcommand{\monograph}[2]{#1\,\mbox{---}\,#2}

\newcommand{\para}{{\,\mathbin{\|}\,}}
\newcommand{\mpara}{\para}
\newcommand{\likelihood}{\ell}
\newcommand{\bayestimes}{;}

\newcommand{\depends}{\hookrightarrow}
\newcommand{\ra}{\rightarrow}

\newcommand{\Ra}{\Rightarrow}

\newcommand{\bx}{\mathbf{x}}
\newcommand{\bX}{\mathbf{X}}
\newcommand{\bP}{\mathbf{P}}

\newcommand{\bR}{\mathbb{R}}
\newcommand{\bN}{\mathbb{N}}

\newcommand{\eproof}{\hfill$\Box$ \smallskip}

\newcommand{\cG}{{\mathcal G}}
\newcommand{\cH}{{\mathcal H}}
\newcommand{\cB}{{\mathcal B}}

\newcommand{\cP}{{\mathcal P}}

\newcommand{\alphabet}{\Sigma}
\newcommand{\Alphabet}{\mathbf{\Sigma}}

\newcommand{\preset}[1]{{^{\bullet\!}{#1}}}
\newcommand{\ppreset}[2]{{^{\bullet{#1}\!}{#2}}}
\newcommand{\postset}[1]{{#1}^{\bullet}}
\newcommand{\inal}[1]{X_{{#1}}^{\rm in}}
\newcommand{\outal}[1]{X_{{#1}}^{\rm out}}

\newcommand{\action}{\alpha}

\newcommand{\Allvars}{\mathcal{X}}
\newcommand{\Vars}{X}

\newcommand{\mtrue}{\mbox{\tt T}}
\newcommand{\mfalse}{\mbox{\tt F}}
\newcommand{\true}{\mbox{\textsc{t}}}

\newcommand{\bE}{\mathbb{E}}

\newcommand{\diag}{{\it diag}}
\newcommand{\kernel}{K}
\newcommand{\proba}{\pi}
\newcommand{\Proba}{\Pi}
\newcommand{\Probas}{\mathcal{P}}

\newcommand{\beq}{\begin{eqnarray}}
\newcommand{\eeq}{\end{eqnarray}}
\newcommand{\beqq}{\begin{eqnarray*}}
\newcommand{\eeqq}{\end{eqnarray*}}
\newcommand{\bea}{\begin{array}}
\newcommand{\eea}{\end{array}}

\newcommand{\bet}{\begin{tabular}}
\newcommand{\eet}{\end{tabular}}

\newcommand{\eqdef}{\,=_{\rm def}\,}

\newcommand{\cL}{\mathcal{L}}

\newcommand{\cN}{{\cal N}}

\newcommand{\restrict}[2]{{#2}_{\left\downarrow{#1}\right.}}
\newcommand{\compat}{{\,{\bowtie}\,}}
\newcommand{\acompat}{{\,\bowtie_{\Alphabet}\;}}
\newcommand{\join}{{\,\sqcup\,}}
\newcommand{\ajoin}{{\,\sqcup_{\Alphabet}\;}}
\newcommand{\compress}[1]{\widetilde{#1}}
\newcommand{\uun}{\mathbf{1}}

\newtheorem{examp}{{\it Example}}
\newtheorem{condit}{Condition}
\RRabstract{Graphical models in probability and statistics are a core concept in the area of probabilistic reasoning and probabilistic programming---graphical models include Bayesian networks and factor graphs. In this paper we develop a new model of mixed (nondeterministic/probabilistic) automata that subsumes both nondeterministic automata and graphical probabilistic models. Mixed Automata are equipped with parallel composition, simulation relation, and support message passing algorithms inherited from graphical probabilistic models. Segala's Probabilistic Automata can be mapped to Mixed Automata.}
\RRresume{Les mod\`eles graphiques sont apparus comme utiles depuis les ann\'ees 1990, dans le domaine des statistiques et de la programmation probabiliste---on regroupe sous ce terme les graphes factoriels et le r\'eseaux Bay\'esiens. Dans ce rapport on pr\'esente le nouveau mod\`ele des \emph{syst\`emes} et \emph{automates mixtes} (probabilistes/nond\'eterministes). Ces mod\`eles \'etendent \`a la fois les mod\`eles graphiques et les automates probabilistes \`a la Segala-Lynch.}
\RRmotcle{graphes factoriels, r\'eseaux Bay\'esiens, nond\'eterminisme, automates probabilistes, programmation probabiliste}
\RRkeyword{factor graphs, Bayesian networks, nondeterminism, probabilistic automata, probabilistic programming}
	\RRprojet{Hycomes}  % cas d'un seul projet
\RCRennes
\begin{document}

\makeRR
\clearpage
\tableofcontents
\clearpage
\section{Introduction}
\label{leguiohu}

\subsection{Context}
\label{skdufcvsk}
Bayesian \emph{graphical modeling} and inference~\cite{dempster1968} expanded since the 1980's, with applications in numerous areas. {Graphical models} were introduced in probability and statistics to allow for a modular description of models~\cite{meent2018introduction}. Graphical models divide into two subfamilies: (directed) Bayesian Networks originally proposed by Judea Pearl \cite{DBLP:journals/ai/Pearl86} and (nondirected) Factor Graphs \cite{kindermann1980markov,Loeliger2004,meent2018introduction}. Probabilistic graphical modeling gave birth to an important sub-community of probabilistic programming~\cite{doi:10.1002/sim.3680,Plummer2003,meent2018introduction}.

\emph{Factor Graphs} allow for the modular specification of unnormalized probabilities, based on a nondirected bipartite graph $(V,F,E)$, where $V{\cup}{F}$ is the set of vertices and $E{\subseteq}V{\times}{F}$ is the set of edges; let $V_f$ be the subset of $v{\in}{V}$ such that $(v,f){\in}{E}$. $V$ is a set of random variables, and, to each \emph{factor} $f{\in}{F}$ is associated an unnormalized probability $p_f(V_f)$ for the tuple $V_f$ of random variables. This model defines the unnormalized probability distribution of $V$ as the product $P(V)=\prod_{f\in{F}}p_f(V_f)$---logarithms of probabilities are often considered instead and added, under the name of {potential}~\cite{kindermann1980markov}.

A \emph{Bayesian Network} is a tuple $(V,E,p)$, where: $V$ is a set of random variables; $(V,E)$ is an acyclic directed graph (for each $v{\in}{V}$, we let $\parents(v)$ denote its parents); $p(v | \parents(v))$ specifies, for each valuation of the parents $\parents(v)$, a conditional distribution for the variable $v$. The semantics of a Bayesian Network is that the joint distribution of $V$ factorizes as the product $P(V)=\prod_{v{\in}{V}}p(v | \parents(v))$. Bayesian Networks are thus causal graphical probabilistic models and the specification of causality comes extra to the specification of the underlying probability distribution, in the form of directed branches of the graph. As pointed out by Judea Pearl~\cite{Pearl09}, causality is an extra information relating random variables, not inferrable from their joint probability distribution.
\emph{Message passing} algorithms are a key tool for Factor Graphs, allowing to map a subclass of them to Bayesian networks, see~\cite{Loeliger2004} and Section 3.6 of~\cite{meent2018introduction}.
Through the union of underlying graphs and the compositional nature of probabilities specified by graphical probabilistic models, both frameworks of Bayesian Networks and Factor Graphs are naturally equipped with some kind of parallel composition.
All these features explain why graphical models are considered as an intermediate format targeted by some probabilistic programming tools, see, e.g.,~\cite{doi:10.1002/sim.3680,Plummer2003}, and~\cite{meent2018introduction}, chapter 3. 

\smallskip

One important issue is \emph{the combination of probabilistic and nondeterministic} behaviors. In statistical decision procedures, deciding whether the distribution of an observed sample belongs to subset $\cP_1$ or $\cP_2$ of probability distributions (where these subsets have empty intersection), exhibits nondeterminism in that the actual distribution is freely chosen within one of the two alternatives; this blending of probabilistic and nondeterministic behaviors is addressed in this case by using {generalized likelihood ratio (GLR) tests}~\cite{lehmann2005testing}. Also, the mixing of probabilistic behaviors and nondeterminism  is central in probabilistic programming~\cite{DBLP:series/mcs/McIverM05,DBLP:journals/toplas/ChatterjeeFNH18,DBLP:conf/isola/McIverM20}.
How to blend probabilistic and nondeterministic behaviors in general is, therefore, an important issue. 
%We consider this issue in our paper.
%These are the issues that we investigate in this paper, in the context of dynamical systems.
 
\smallskip

%\subsection{Probabilistic Programming}
\emph{Probabilistic Programming}~\cite{doi:10.1002/sim.3680,Plummer2003,JSSv076i01,DBLP:journals/toplas/OlmedoGJKKM18,DBLP:conf/birthday/KatoenGJKO15,meent2018introduction,rppl_pldi20} provides support for specifying statistical models with modularity and libraries for performing inference. 
%Probabilistic programming paradigm has several facets. Most emphasized is the provision for inference and learning algorithms. 
Some probabilistic languages generate \emph{likelihood functions}~\cite{Plummer2003,doi:10.1002/sim.3680,JSSv076i01} for use by inference algorithms, whereas other generate \emph{sampling} procedures~\cite{DBLP:journals/corr/abs-1206-3255,dippl}. Recently, G. Baudart et al.~\cite{rppl_pldi20} proposed \emph{reactive probabilistic programming} of dynamical systems as a conservative extension of {synchronous languages}~\cite{DBLP:journals/pieee/BenvenisteCEHGS03}, by enhancing the Hybrid Systems modeling language \href{http://zelus.di.ens.fr/}{Zelus}~\cite{DBLP:journals/pieee/BenvenisteBCCPP18} with probabilities. Objectives of Probabilistic Programming can be categorized as follows:
\begin{enumerate}
	\item \label{kjisgvcjyt} \emph{Modeling paradigm.} Blending probability and nondeterminism, composing, comparing (equivalence), are the main issues. 
	
	\item \label{jsefdhfgsvjh} \emph{Model for proof systems.} Calculi and their decidability and complexity are central issues in this objective.
	
	\item \label{skdufvbnfgklj} \emph{Support for statistical inference, decision, and learning.} Key pillars are all limit theorems of probability and statistics (law of large numbers, central limit theorem, large deviations). These theorems rely on stationarity (or time invariance) of the underlying probabilistic model. For models with no dynamics, independent identically distributed (i.i.d.) sets of data can be sampled from the model. For models with dynamics (e.g., Markov chains), runs can be observed and used to infer model characteristics. One central difficulty in this objective is the blending of nondeterminism with probabilities, as it generally breaks the stationarity of the underlying statistical model. For example, if two different statistical models are combined with a nondeterministic choice (or an if-then-else statement with nondeterministic guard), then stationarity of the overall model no longer holds. The solution is to recover stationary models by separating the two alternatives and not mixing them. This quickly becomes cumbersome if several such constructions are used in a model. This difficult and central issue is extensively discussed in~\cite{DBLP:conf/fsttcs/HurNRS15}, where it is shown that some major probabilistic programming tools may not correctly implement Monte-Carlo based learning algorithms such as Metropolis-Hastings. 
\end{enumerate}

\subsection{Contribution}
\label{fdkvjfhdvbkjfdh}
To illustrate our purpose, we begin with a toy example. 
Throughout this paper, all variables possess finite or denumerable type---this restriction is motivated by technical reasons explained later. Hence, types will not be declared when presenting examples. In ``if-then-else'' statements, it is understood that the control variable is Boolean.
Consider the following discrete time dynamical system (universal quantifier $\forall n$ is implicit):
\beq
S_1&:&\left\{\bea{l}
\prog{\keyw{observe}}\;u \\
x_0 = c_x \\
x_n = \varphi(u_n,x_{n-1}) \\
y_n = \mbox{if } f_n \mbox{ then } \psi(x_n,v_n) \mbox{ else } x_n
\eea\right.
\label{ujsydtcfjhy}
\eeq
Model (\ref{ujsydtcfjhy}) involves \emph{signals,} i.e., sequences, indexed by the natural integer $n$, of variables having the same type: for instance, signal $x_n$ denotes the sequence $\{x_k\mid k{\in}\bN\}$.
In (\ref{ujsydtcfjhy}), $f_n$ is a boolean signal indicating the occurrence of a failure and $v_n$ is a noise, i.e., some kind of disturbance. When a failure occurs, signal $x_n$ gets corrupted by noise $v_n$, which is captured by the (unspecified) function $\psi$; otherwise, $y_n=x_n$. Since model (\ref{ujsydtcfjhy}) involves the delayed signal $x_{n-1}$, an initial condition for this signal is specified by $x_0 = c_x$, where $c_x$ is some constant of same type as signal $x_n$. Model (\ref{ujsydtcfjhy}) looks like a  dynamical system as usual, with inputs $u,f$, and $v$, state $x$, and output $y$. 
%Depending on the value of boolean signal $f_n$, signal $x_n$ is, or is not, related to $v_n$.

We are interested in a different interpretation, however, by which model (\ref{ujsydtcfjhy}) specifies what is observed/unobserved: $u_n$ is observed at every instant (as stated in the first line), whereas other signals are unobserved (this is the default case). From this perspective, signals $f,v,x$, and $y$ are unknown and otherwise subject to (\ref{ujsydtcfjhy}). Thus, model (\ref{ujsydtcfjhy}) involves \emph{nondeterminism}. 

Next, consider the following stochastic model for noise $v_n$:
\beq
S_2&:&v_n\; \keyw{\sim} \;\mu 
\label{sdujcytsdfjy}
\eeq
where $v_n \,\keyw{\sim}\, \mu$ means that variable $v_n$ has distribution $\mu$ at each instant $n$. As an important convention of our modeling framework, statement $v_n \,\keyw{\sim}\, \mu$, taken in isolation, also means that the random sequence $v_n$ is \emph{independent, identically distributed (i.i.d.)}. No signal is observed in this model (capturing that we are considering an unobserved disturbance).

Having the two models $S_1$ and $S_2$, we like to \emph{compose} them, thus considering $S_1\|S_2$, defined as the conjunction of the two systems of equations (\ref{ujsydtcfjhy},\ref{sdujcytsdfjy}). $S_1\|S_2$ combines stochastic behavior with nondeterminism (since failure signal $f_n$ is still unknown and unobverved). As a consequence of this composition, the nature of signal $y_n$ may or may not involve randomness, due to the if-then-else statement occurring in $S_1$.

Consider next the following $S_3$ model specifying the behavior of the failure signal $f$:
\beq
S_3&:&\left\{\bea{l}
f_0=\mfalse \\
f_n = (\mathit{rf}_n \mbox{ or } f_{n-1}) \mbox{ and not } \mathit{bk}_{n} \\
\mathit{rf}_n \;\keyw{\sim}\; \mathrm{Bernoulli}(10^{-6})
\eea\right.
\label{erliguhoiu}
\eeq
In this model, ``root failure'' signal $\mathit{rf}$ is modeled as a Bernoulli sequence, i.e., $P(\mathit{rf}=\mtrue)=10^{-6}$; boolean signal $\mathit{bk}$ indicates that a ``backup sensor'' is provided. Thus, a failure is raised ($f=\mtrue$) if a root failure occurs, and it remains subsequently raised, until a backup sensor is provided. In $S_3$, no signal is observed, thus $\mathit{bk}$ is nondeterministic. Model (\ref{erliguhoiu}) is mixed probabilistic/nondeterministic. If $\mathit{bk}$ was specified as being observed, this model would become probabilistic in that, once the value of random signal $\mathit{rf}_n$ is known, the actual value of $f_n$ is determined. 

The next step is to further compose $S_1\|{S_2}$ with $S_3$. By convention of the parallel composition, as a consequence of composing the two statements ``$v_n\; \keyw{\sim} \;\mu$'' and ``$\mathit{rf}_n \;\keyw{\sim}\; \mathrm{Bernoulli}(10^{-6})$'', \emph{the two random sequences $v_n$ and $\mathit{rf}_n$ are mutually independent.}

As a safety issue, we could be interested in evaluating the risk of missing an alarm raised by having signal $y$ exceeding some threshold: an alarm is raised when $y_n> y_{\max}$. This alarm triggers some reconfiguration, not shown here. This reconfiguration action was designed to act under the hypothesis that the system is fault-free, i.e., $y_n=x_n$ always holds. Consider the following question: what is the ``risk'' that an alarm is missed when it should have occured, due to a fault? More precisely, 
\beq
\mbox{what is the risk that ``$x_n>y_{\max}$ and $y_n\leq y_{\max}$'' occurs?} \label{sldkfjhsdbkjh}
\eeq
So far we did not define what we mean by ``risk''. It cannot be measured by a probability, since $S_1\|S_2\|S_3$ mixes probability with nondeterminism. By ``risk'' we mean a pessemistic evaluation of this probability, with nondeterminism acting as an adversary.

Suppose, next, that we want to specify that signal $y$ is observed in system $S_1\|S_2\|S_3$. To this end, we consider the system
\beq
S_4&:&\prog{\keyw{observe}}\; y
\label{wseduydtrfduytfd}
\eeq
where no dynamics is otherwise specified. Parallel composition $S_1\|S_2\|S_3\|S_4$ expands as the following model:
\beq
&&\left\{\bea{l}
\prog{\keyw{observe}}\; u,y \\
x_0 = c_x \,,\, v_0=c_v \,,\, f_0 = \mfalse \\
x_n = \varphi(u_n,x_{n-1}) \\
y_n = \mbox{if } f_n \mbox{ then } \psi(x_n,v_n) \mbox{ else } x_n \\
f_n = (\mathit{rf}_n \mbox{ or } f_{n-1}) \mbox{ and not } \mathit{bk}_{n} 
\eea\right.
\label{xystrsadcutsr}
\\
&&\left\{\bea{l}
\mathit{rf}_n \;\keyw{\sim}\; \mathrm{Bernoulli}(10^{-6})  \\
v_n \;\keyw{\sim}\; \mu \\
(\mathit{rf}_n \mbox{ and } v_n \mbox{ are mutually independent i.i.d. signals})
\eea\right.
\label{auxyrafxutyar}
\eeq
The intended semantics of model (\ref{xystrsadcutsr},\ref{auxyrafxutyar}) is as follows: (\ref{auxyrafxutyar}) specifies the \emph{prior distribution} of the pair $(v,\mathit{rf})$ of random signals, where, by convention, the two signals are considered independent. (\ref{xystrsadcutsr}) defines a constraint on the tuple of variables involved in the system. The $\prog{\keyw{observe}}$ constraint on the pair $u,y$ states that its joint trajectory is given (through the sensors). As a consequence, the pair $(v,\mathit{rf})$ of random signals is now equipped with the \emph{posterior distribution} resulting from constraint (\ref{xystrsadcutsr}) being enforced. 

\medskip

 If we regard systems $S_1,\dots,S_4$ as boxes with wires (the involved signals), this modeling approach naturally leads to graphical models alike Factor Graphs. Indeed, this way of specifying mixed probabilistic/nondeterministic systems is fully modular: component models can be freely assembled to yield system models. Primitive statements are:
%\begin{itemize}
	%\item 
	1) declarations of prior distributions;
	%\item 
	2) declarations of constraints on signals through equations relating them, implicitly resulting in the definition of a posterior distribution; and
	%\item 
	3) a parallel composition in which composing prior distributions considers them independent and systems of equations are composed as usual.
%\end{itemize}
Closest to this approach are~\cite{rppl_pldi20} (born from synchronous programming~\cite{DBLP:journals/pieee/BenvenisteCEHGS03}) or~\cite{DBLP:conf/popl/GuptaJP99} (born from concurrent constraint  programming).

\medskip

As a semantic domain for the above modeling approach, we propose a framework subsuming graphical probabilistic modeling and supporting both probabilistic and nondeterministic behaviors. We focus our effort on semantics issues, such as: What is actually the probabilistic model specified? Given seemingly different system specifications, are they equivalent or do they differ? Can one define
a parallel composition of models? 
\emph{With reference to the context of probabilistic programming recalled in Section~$\ref{skdufcvsk}$, our work focuses on objective~$\ref{kjisgvcjyt}$ only, with no consideration of other objectives.} 

One of our contributions is the model of \emph{Mixed Automata}.
Its design relies on a very simple idea. An automaton is specified through its set of transitions $\trans{q}{\action}{q'}{}$, where $q$ and $q'$ are the current and next state, and $\action$ is the action triggering the transition. Upgrading this model to Probabilistic Automata~\cite{DBLP:conf/concur/LynchSV03} consists in upgrading transitions to $\trans{q}{\action}{\proba'}{}$, where $\proba'$ is the next probabilistic state (a probability distribution over the set $Q$ of states), from which the next state is derived by probabilistic sampling $\produces{\proba'}{q'}$. The final upgrade to Mixed Automata is by upgrading such transitions to $\trans{q}{\action}{\system'}{}$, where $\system'$ is now a \emph{Mixed System}  (or Mixed Probabilistic/Nondeterministic System in its extended name), from which the next state is derived by sampling $\produces{\system'}{q'}$.

Initially proposed in~\cite{BenvenisteLFG95}, Mixed Systems are pairs consisting of a private probability space and a visible state space, related through a relation. This pair specifies a posterior distribution, namely the conditional distribution given that the relation between states and random outcomes is satisfied. Visible states are exposed for possible interaction with other Mixed Systems. This allows to equip Mixed Systems with a parallel composition, on top of which a parallel composition for Mixed Automata can be defined.
We show that Mixed Systems naturally inherit a notion of \emph{graphical structure}, which subsumes both Bayesian Networks and Factor Graphs. 
Mixed Systems offer the counterpart of angelic/demonic nondeterminism~\cite{DBLP:journals/toplas/ChatterjeeFNH18} and hard/soft conditioning~\cite{DBLP:conf/lics/StatonYWHK16,DBLP:conf/icml/TolpinZRY21}, which are important notions in probabilistic programming.

Mixed Automata, defined on top of Mixed Systems, naturally inherit their associated graphical structure and parallel composition.
Mixed Automata are equipped with all the fundamental modular notions for Automata, namely the notions of (bi)si\-mu\-la\-tion relation and parallel composition.
In this paper, we show in addition that Mixed Automata subsume, in part, Segala's  Probabilistic Automata (PA)~\cite{SokolovaV04} and their variants. More precisely, we exhibit mappings from different PA models to Mixed Automata, preserving simulation equivalence. The parallel compositions, however, are most of the time different---we claim ours to be more useful than PA parallel composition when the two differ. In addition, in contrast to PA, our model of Mixed Automata naturally captures the notion of posterior (conditional) distribution and offers a notion of graphical model.
%, thus adequately supporting probabilistic programming.

\smallskip

The paper is organized as follows. Mixed Systems are introduced and further studied in Section~\ref{uwtrdytrd}. Mixed Automata are introduced and studied in Section~\ref{roe87wyhfgbsehlrigu}, and then compared to Probabilistic Automata in Section~\ref{45oiuepiruh}. Related work is discussed more broadly in Section~\ref{fdkjvblk}. Missing proofs are deferred to appendices. Focused bibliographical discussions are presented following each important notion. The reason is that the same mathematical notion occurs in different communities, under different names; so we felt it useful to relate them. Finally, Appendix~\ref{reoiughofieurf} presents hints for extending our approach to continuous probability distributions.

\section{Mixed Probabilistic/Nondeterministic Systems}
\label{uwtrdytrd}
\label{eriuyoiu}
$\Allvars$ shall denote an underlying set of \emph{variables,} of finite domain. Elements of $\Allvars$ are denoted by lower case letters $x,y,z\dots$, and finite subsets of $\Allvars$ are denoted by upper case letters $X,Y,Z$. We use set theoretic operations on sets of variables. Whenever convenient, we regard $X,Y,Z$ as tuples. The domain of $x$ is denoted by $Q_x$ and the domain of $X$ is $Q_X\eqdef\prod_{x\in{X}}Q_x$, we call it the \emph{state space}; the generic element of $Q_X$ is called a \emph{state} and is denoted by $q_X$ or simply $q$.

The pair
$(\Omega,\proba)$ shall denote a discrete probability space, i.e., $\proba$ is a countably additive function, from $2^\Omega$ to $[0,1]$, such that $\proba(\emptyset)=0$ and $\proba(\Omega)=1$. We simply write $\proba(\omega)$ instead of $\proba(\{\omega\})$. The \emph{support} of $\proba$ is the set $\supp(\proba)\eqdef\{\omega\mid\proba(\omega)>0\}$.
 For a subset $W{\subseteq}\Omega$ such that $\proba(W){>}0$, the \emph{conditional probability} $\proba(V{\mid}{W})\eqdef\frac{\proba(V\cap{W})}{\proba(W)}$ is well defined.

Finally, we will consider \emph{relations} (or \emph{constraints}) $\cons\subseteq\Omega\times{Q}$. Relations are composed by intersection.

%\footnote{\label{eriuyoiu}
\myparagraph{Disclaimer: in this paper, we consider only discrete probability spaces} This restriction is technically important, since it allows for a straightforward definition of conditional probabilities, and the notion of support of a probability is easily defined. For the general case, the notion of {conditional expectation} is always defined~\cite{DellacherieMeyer1978}, whereas conditional distributions require additional topological assumptions for their existence, and so does the notion of {support.} To keep our work  simpler, we decided not to cover those extensions. Appendix~\ref{reoiughofieurf} presents hints for extending our approach to continuous probability distributions. 
%See, nevertheless, Appendix~\ref{reoiughofieurf} for indications about such an extension and how the above difficulties modify the framework.
%}

%We first begin by reporting the background from~\cite{BenvenisteLFG95}, and then we provide new results.
%~
%\subsection{\bf\em The model}
%\label{erlgioehlguioh}
%
%\subsection{\bf\em Background from previous work and extended Factor Graphs}
\subsection{Mixed Systems, parallel composition, and Factor Graphs}
\label{lreiuhoeriuhli}
%This section is derived from~\cite{BenvenisteLFG95}. The link with factor graphs is, however, new.
In this section we introduce Mixed Systems and show that they extend and subsume in a unified framework: nondeterminism, probability spaces, and factor graphs. This section is inspired in part by~\cite{BenvenisteLFG95}.

\begin{figure}[!h]
	\centerline{\scalebox{0.8}{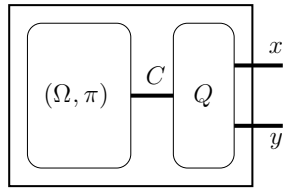}}
	\caption{Intuitive picturing of a Mixed System having two variables $x$ and $y$.}
	\label{iweuftygi}
\end{figure}
The intuition is illustrated on Figure~\ref{iweuftygi}, which will guide us for the different notions attached to Mixed Systems. A Mixed System will be a pair, consisting of a probability space $(\Omega,\proba)$ and a state space $Q$ collecting the configurations of a tuple of state variables (here: $x$ and $y$), related by a relation $\cons$. The probability space is ``private'', in that it is not directly exposed to any interaction with the environment. Interactions with the environment only occur through the state variables, thus seen as ``visible''. This distinction private/visible is shown on Figure~\ref{iweuftygi} by the outgoing pins $x,y$, which contrast with the absence of outgoing pin for the probabilistic box.

We are interested in understanding how Mixed Systems are ``executed'' (we call this the \emph{sampling}), and how state properties---which are not by themselves random---can still get some kind of probabilistic evaluation.
%\clearpage
\begin{definition}[Mixed System, definition and semantics]
  \label{slergiuhpiu} \
		\begin{enumerate}
		
		\item \label{lrfgblsdlikf} A \emph{Mixed System} (or \emph{system} for short) is a tuple $\system=(\Omega,\proba,X,\cons)$, where:	$(\Omega,\proba)$ is a probability space;	$X$ is a finite set of variables with domain  $\bea{c}Q=\prod_{x\in{X}}Q_x\eea$; and $\cons\subseteq{\Omega{\times}{Q}}$ is a relation. In the sequel, we write $$\omega\cons{q}$$ to mean $(\omega,q){\in}\cons$. 
		%The \emph{probabilistic support} of $\cons$ is $\cons_\proba=\cons\cap(\support(\proba){\times}{Q})$.
		
		\item \label{sdxhtrsd}
$\system$ is called \emph{consistent} if $\proba(\consist{\Omega})>0$, where
$\consist{\Omega}\eqdef\{\omega{\in}\Omega\mid\exists{q}:\omega\cons{q}\}$.
If $\system$ is consistent, its \emph{sampling} is well defined and consists in:
\begin{enumerate}
	\item sampling $\omega{\in}\Omega$ according to conditional probability $\consist{\proba}$, where:
\beq\bea{rcl}
\forall A\subseteq\Omega&:&
\consist{\proba}(A)\eqdef\proba(A\mid\consist{\Omega})=\frac{
\proba(A\cap\consist{\Omega})
}{
\proba(\consist{\Omega})
}\,,
\eea
\label{leriughelorui}
\eeq
\item and, then, nondeterministically selecting $q{\in}{Q}$ such that $\omega\cons{q}$.
\end{enumerate}
 This two-step procedure is denoted by $\produces{\system}{q}$.

\item \label{ausytdfxuajy}
If $\system$ is consistent, its \emph{probabilistic semantics} is defined as the pair  $\outproba,\inproba:2^Q\ra[0,1]$, where, for
any state property $A\subseteq{Q}$:
 %is given the following \emph{outer} and \emph{inner probabilities} $\outproba,\inproba:2^Q\ra[0,1]$, serving as probabilistic semantics:
\beq
\outproba(A)\eqdef\consist{\proba}({\Omega_{{\exists}A}}) &\mbox{where}&
{\Omega_{{\exists}A}}\eqdef\{\omega\in\Omega\mid\exists q\in{A}:\omega\cons{q}\}\,,
\label{jdywetfdewjty} \\
\inproba(A)\eqdef\consist{\proba}({\Omega_{{\forall}\!A}}) &\mbox{where}&
\Omega_{{\forall}\!A}\eqdef\{\omega\in\Omega\mid\forall q\in{A}:\omega\cons{q}\}\,,
\label{skdjfhkj}
\eeq
The following \emph{generalized likelihood} $\likelihood:2^Q\ra[0,1]$ is also of interest:
\beq\bea{c}
\likelihood(A) \eqdef \max_{\omega\in{\Omega_{{\exists}A}}}\;\consist{\proba}(\omega)\,.
\eea
\label{ewsfdtyfgujy}
\eeq
\end{enumerate}
In the sequel, we shall denote by $\Systems(X)$ the class of all (possibly inconsistent) Mixed Systems having $X$ as their set of variables.\eproof
\end{definition}
$\outproba$ defined by formula (\ref{jdywetfdewjty}) is not a probability on $Q$, but only an \emph{outer probability}\footnote{sometimes called also \emph{exterior} or \emph{upper} probability.}, i.e., a function $\outproba:2^Q\ra[0,1]$ such that $\outproba(\emptyset)=0$, $\outproba(Q)=1$, and $\outproba$ is {sub-additive}, meaning that it satisfies
\[\bea{c}\forall A, (A_n)_{n\in\bN} \mbox{ subsets of }Q : A\subseteq\bigcup_{n\in\bN}A_n \implies \outproba(A)\leq\sum_{n\in\bN}\outproba(A_n)\,.
\eea\]
 Note that (\ref{ewsfdtyfgujy}) resembles (\ref{jdywetfdewjty}) if we rewrite the latter as $\outproba(A)=\sum_{\omega\in{\Omega_{{\exists}A}}}\consist{\proba}(\omega)$. The same comments hold, mutatis mutandis, regarding the \emph{inner probability} $\inproba$, which is super-additive. Note that, 
\beq
\mbox{if $A=\{q\}$ is a singleton, then $\outproba(A)=\inproba(A)$.} \label{fjhksjhfj}
\eeq

\begin{example}\rm [Specializing to pure nondeterministic systems]
\label{wuyetdfwueity} \rm 
A pure nondeterministic system is specified as a subset $\widehat{\cons}\subseteq{Q}$ of the state space. To reformulate it as a Mixed System, simply take $(\Omega,\proba)$ trivial, i.e., $\Omega=\{\omega\}$, a singleton, equipped with the trivial probability such that $\proba(\omega)=1$, and define $\omega\cons{q}$ iff $q\in\widehat{\cons}$.\eproof
\end{example}
\begin{example}\rm [Specializing to pure probabilistic systems]
\label{wuydtfuty} \rm A pure probabilistic system is specified as a pair $(\Omega,\proba)$. To reformulate as a Mixed System, take $Q=\Omega$, and let $\cons$ be the diagonal of $\Omega\times{Q}$; finally, let $x$ be the variable with domain $Q$.\eproof
\end{example}
\begin{discussion}[blending nondeterminism and probability] \rm
	\label{isedkufghsvjuy} To capture the blending of nondeterminism and probability, outer probabilities are directly used in the Dempster-Shafer theory of evidence~\cite{dempster1967,dempster1968,shafer1976mathematical}.\footnote{Outer and inner probabilities were called upper and lower in ~\cite{dempster1967}.} Outer probabilities do not support key limit theorems for use in statistics, such as the law of large numbers, the central limit theorem, and more. Hence, whereas the theory of evidence comes with reasoning capabilities, it does not directly support learning or estimation. 
	
	%\albert{proposer un transcodage?}
	
	In formal methods for probabilistic systems (in the context of imperative programming), the blending of probability and nondeterminism was addressed by a number of authors, see, e.g.,~\cite{DBLP:journals/ijfcs/HartogV02,DBLP:series/mcs/McIverM05,DBLP:conf/lpar/BartheEGHSS15,DBLP:conf/isola/McIverM20,DBLP:conf/birthday/KatoenGJKO15,DBLP:journals/toplas/OlmedoGJKKM18,DBLP:journals/toplas/ChatterjeeFNH18,DBLP:journals/entcs/WangHR19}. Nondeterministic choice between alternatives is considered in~\cite{DBLP:conf/isola/McIverM20} and written $P\,{\sqcap}\,P'$, whereas probabilistic choice is specified as $P\,{_a\oplus}\,P'$ ($P$ is selected with probability $a$ and $P'$ with probability $1{-}a$) or $P\,{_a\oplus_b}\,P'$ ($P$ is selected with probability at least $a$ and $P'$ with probability at least $b$). The evaluation of formulas must specify how nondeterminism interplays with probabilities. A comprehensive approach was proposed in~\cite{DBLP:journals/toplas/ChatterjeeFNH18}, where \emph{demonic} and \emph{angelic} nondeterminisms are seen as adversarial and beneficial, respectively. These notions mirror the outer and inner probabilities used in Dempster theory. Unfortunately, outer and inner probabilities do not bring limit theorems of probability theory (law of large numbers, etc.), which are the core of machine learning.
	
Through formulas (\ref{jdywetfdewjty},\ref{skdjfhkj}) in Definition~\ref{slergiuhpiu},	the probabilistic semantics of Mixed Systems is defined as the associated outer and inner probabilities. Hence, Mixed Systems offer the calculus of the theory of evidence, and mirror the {demonic} and {angelic} types of nondeterminism. 
	On the other hand, since classical probability spaces are first class citizens of the model of Mixed Systems, this model also preserves the apparatus needed for machine learning. In Appendix~\ref{kuyfgksuy}, we develop a more detailed comparison of the semantics of Mixed Systems versus imperative probabilistic programming with {demonic} and {angelic} nondeterminism, following~\cite{DBLP:journals/toplas/ChatterjeeFNH18}.

%\footnote{In french, we say: \emph{avoir le beurre et l'argent du beurre.}} 

Finally, the generalized likelihood of f\rm ormula (\ref{ewsfdtyfgujy}) is the basis for inference, estimation, or machine learning, when multiple hypotheses or nuisance parameters are considered~\cite{lehmann2005testing}---we are not aware of any use of a mirror notion where ``min'' would be substituted for ``max''.\eproof
\end{discussion}
%\clearpage
\begin{example}\rm [outer probabilities]
	\label{kdugavkhvjhg} \rm Consider model $S_1\| S_2\| S_3$ of Section~\ref{fdkvjfhdvbkjfdh}. Pick an instant $n$ and let $S\eqdef S(n,x_{n-1},f_{n-1})$ be the Mixed System defined by (\ref{ujsydtcfjhy},\ref{sdujcytsdfjy},\ref{erliguhoiu}) for instant $n$ and given values for $x_{n-1},f_{n-1}$. With reference to (\ref{sldkfjhsdbkjh}), we wish to evaluate the probability that $x_n>y_{\max}$ and $y_n\leq y_{\max}$ occurs under adversarial nondeterminism. Denoting by $Q_v$ the domain of $v$ and by $\Bool$ the Boolean domain, the underlying probability space of $S$ is $(\Omega,\proba)$, where $\Omega=Q_v{\times}\Bool$ and $\proba=\mu{\times}\beta$, where $\beta$ is Bernoulli$(10^{-6})$. Domain $Q$ for the variables of $S$ is $Q=Q_x{\times}Q_y{\times}Q_u{\times}Q_v{\times}\Bool{\times}\Bool$, and relation $\cons$ is defined by the nonprobabilistic equations of $S$, i.e., (\ref{ujsydtcfjhy},\ref{erliguhoiu}) in which we discard the statement $\mathit{rf}_n\sim\beta$. Finally, let $\cons(u_n)$ denote the relation $\cons$ in which the value of $u_n$ is given ($u$ is observed).	
	Then 
	%$\outproba(x_n{>}y_{\max}\mbox{ and }y_n{\leq} y_{\max})
	%=
	%\consist{\proba}(W)$, where:
	\beq\bea{l}
	\outproba(x_n{>}y_{\max}\mbox{ and }y_n{\leq} y_{\max})
	=
	\consist{\proba}(W)\;,\;\mbox{ where}
	\\
W=	\left\{
	(v,\mathit{rf})
	\,\left|\,
\exists x_n,y_n,f_n,\mathit{bk}_n: \bea{l}
	x_n{>}y_{\max}\mbox{ and }y_n{\leq} y_{\max} ~\mbox{, and} \\ (x_n,y_n,f_n,\mathit{bk}_n,v,\mathit{rf}) \in \cons(u_n)
\eea	\right.\right\}
\eea
\label{ksdjfhsvdcbky}
	\eeq
Inspecting (\ref{ujsydtcfjhy},\ref{erliguhoiu}) shows that the condition defining set $W$ rewrites as
\[
x_n{>}y_{\max}\mbox{ and }\psi(x_n,v){\leq} y_{\max}\mbox{ and }f_n{=}\mtrue \mbox{ and }(x_n,y_n,f_n,\mathit{bk}_n,v,\mathit{rf}) \in \cons(u_n)\,.
\]
First, if $\varphi(u_n,x_{n-1}) \leq y_{\max}$ holds, then $W=\emptyset$. We thus assume in the sequel $\varphi(u_n,x_{n-1})>y_{\max}$. Thus we need to evaluate with respect to $\outproba$ the predicate 
\[
Z \;\eqdef~ \psi(x_n,v){\leq} y_{\max}\mbox{ and }f_n{=}\mtrue \mbox{ and }(x_n,y_n,f_n,\mathit{bk}_n,v,\mathit{rf}) \in \cons(u_n)\,.
\]
Condition $f_n{=}\mtrue$ is equivalent to the conjunction of the following two conditions: 1) $\mathit{bk}_n{=}\mfalse$ (backup sensor is not available), 2) $f_{n-1}{=}\mtrue$ or $\mathit{rf}_n{=}\mtrue$. Recall that the value of $f_{n-1}$ is given. We thus distinguish the following two cases:
\begin{enumerate}
\item \label{ksdjhvsbkmy} $f_{n-1}{=}\mtrue$: then, $f_n{=}\mtrue$ whatever the value of $\mathit{bk}_n$ is, and, using (\ref{ksdjfhsvdcbky}):
	\[
	\outproba(x_n{>}y_{\max}\mbox{ and }y_n{\leq} y_{\max}) =
	\mu\bigl\{v\mid \psi(\varphi(u_n,x_{n-1}),v)\leq y_{\max}\bigr\}\,.
	\]
	\item \label{ikwsuyfgvku} $f_{n-1}{=}\mfalse$: then, $f_n{=}\mtrue$ if and only if $\mathit{rf}_n{=}\true$ and $\mathit{bk}_n{=}\mfalse$. Thus, chosing $\mathit{bk}_n{=}\mfalse$ ensures that: $\mathit{rf}_n{=}\true$ and $\psi(\varphi(u_n,x_{n-1}),v)\leq y_{\max}$ together yield $y_n{\leq} y_{\max}$. Alternatively, $\mathit{bk}_n{=}\mtrue$ prevents the condition $y_n{\leq} y_{\max}$ from occurring. By definition of the outer probability (\ref{jdywetfdewjty}), we finally get, using (\ref{ksdjfhsvdcbky}):
	\[
	\outproba(x_n{>}y_{\max}\mbox{ and }y_n{\leq} y_{\max})
	=
	\beta(\mathit{rf}{=}\mtrue)\times
	\mu\bigl\{v\mid \psi(\varphi(u_n,x_{n-1}),v)\leq y_{\max}\bigr\}\,,
	\]
	which corresponds to the probabilistic evaluation of the predicate ``$x_n{>}y_{\max}$ and $y_n{\leq} y_{\max}$'' if the nondeterministic alternative $\mathit{bk}_n{=}\mfalse/\mtrue$ is interpreted as demonic~\cite{DBLP:journals/toplas/ChatterjeeFNH18}.\eproof
\end{enumerate}
\end{example}
\begin{discussion}[Conditioning and its variations] \rm
	\label{sdkfugkuy} 
	Conditioning is generally not considered in the field of probabilistic automata. It is, however, central in probabilistic programming, see, e.g.,~\cite{DBLP:journals/toplas/OlmedoGJKKM18,DBLP:journals/pacmpl/DahlqvistK20,DBLP:conf/icml/TolpinZRY21} for studies in which conditioning is the main subject. The $\prog{\keyw{observe}}$ primitive, pervasive in all tools, is used to specify posterior distributions given constraints (as we do in Definition~\ref{slergiuhpiu}). The litterature on probabilistic programming distinguishes between \emph{hard} (also called \emph{deterministic}) and \emph{soft} (also called \emph{stochastic}) conditioning~\cite{DBLP:conf/lics/StatonYWHK16,DBLP:conf/icml/TolpinZRY21}. In the basics of probability theory, however, the only notion is that of \emph{conditional expectation}~\cite{DellacherieMeyer1978}, from which other notions are derived, e.g., conditional probability, transition probability or stochastic kernel, and disintegration (or regular version of conditional expectation). Deriving such notions is straightforward in our case, since we restrict ourselves to discrete probability spaces. We will discuss this further when extending Bayesian networks to Mixed Systems, in Section~\ref{selrujgnhelru}.
	%\eproof
\end{discussion}
\begin{discussion}[consistency] \rm
	\label{sikufhbikuy} 		
Inconsistency formalizes self-contradiction, for Mixed Systems. 
The condition ``$\proba(\consist{\Omega})>0$'' in statement~\ref{sdxhtrsd} of Definition~\ref{slergiuhpiu} means that $\consist{\Omega}$ has non-empty intersection with the support of $\proba$, defined as the set of $\omega$'s of positive probability: $\proba({\omega})>0$. This simple definition for the support, which is only valid for discrete probabilities, allows us to propose a simple definition for the notion of consistency. When continuous probability spaces are considered (like the Gaussian), the above definition for the support no longer holds. The right definition relies on topological properties. As a consequence, our elementary definition of consistency would no longer apply. This is next illustrated on our running example.
%\eproof
\end{discussion}
\begin{example}\rm [consistency] \rm
	\label{skujfycsgdkujy} 
Consider model (\ref{xystrsadcutsr},\ref{auxyrafxutyar}). Statement~\ref{sdxhtrsd} of Definition~\ref{slergiuhpiu} defines consistency as the existence of a state $q$ in relation through $
\cons$ with an $\omega$ belonging to the support of $\proba$, which is fairly simple. Suppose, for a while, that $u_n,x_n,y_n,v_n$ possess real domain, $\mu(dv)=\chi(v)dv$, where $dv$ denotes the Lebesgue measure, density $\chi$ is continuous and everywhere positive, and function $v\mapsto\psi(x,v)$ is bijective and bicontinuous for every fixed $x$. Then, fixing the value of $y_n$, for a given pair $(u_n,x_{n-1})$, will fix the value of $v_n$ if $f_n{=}\mtrue$ in the equation defining $y_n$. With reference to Example\,\ref{kdugavkhvjhg}, the only difference is that a parallel composition with the statement $\prog{\keyw{observe}}\;y$ was added. So, it still makes sense to consider the two cases\,\ref{ksdjhvsbkmy} and\,\ref{ikwsuyfgvku} of Example\,\ref{kdugavkhvjhg}. In case\,\ref{ksdjhvsbkmy}, we get
$W= \left\{
(\mathit{rf},v)
\mid
\psi(\varphi(u_n,x_{n-1}),v){=}y_n
\right\}$, whence $(\beta{\times}\mu)(W){=}0$. Deducing inconsistency would be nonsense, however, since the support of $\mu$ is $\bR$. This illustrates that our pedestrian definition of consistency no longer works if real variables and  distributions having densities with respect to Lebesgue measure are considered.
%\eproof
\end{example}

\subsubsection{Equivalence}
In this section, we study equivalence. To this end, we introduce the following operation of compression, on top of which equivalence is defined:
\begin{definition}[compression]
  \label{lighlalegfr}
  For $\system=(\Omega,\proba,X,\cons)$ a Mixed System, we define
  the following equivalence relation on $\Omega$, i.e., $\sim\;\subseteq\Omega{\times}\Omega$ is such that:
	\beq
		(\omega,\omega')\in\;\sim&\mbox{ if and only if: }&
	\forall{q}\in{Q}: \omega\cons{q}\Leftrightarrow\omega'\,\cons{q}\,.
  \label{eoguheogihio}
  \eeq
	As usual, we write $\omega \sim   \omega'$ to mean $(\omega,\omega')\in\;\sim$.
	The \emph{compression} of $\system$, denoted by
$\compress{\system}=(\compress{\Omega},\compress{\proba},X,\compress{\cons})$,
is then defined as follows:
\begin{itemize}
	\item $\compress{\Omega}$ is the quotient $\Omega{/}{\sim}$, which
elements are written $\compress{\omega}$;
\item
$\compress{\omega}\compress{\cons}q$ iff ${\omega}{\cons}q$
for $\omega\in\compress{\omega}$; and
\item
\mbox{$\compress{\proba}(\compress{\omega})=\sum_{\omega\in\compress{\omega}}\proba(\omega)$}.
\end{itemize}
Say that $\system$ is \emph{compressed} if it coincides with its compression.\eproof
\end{definition}
Distinguishing $\omega$ and $\omega'$ is impossible if
$\omega{\sim}\omega'$.
%By construction,
%$\compress{\system}$ and ${\system}$ possess identical samplings and identical probabilistic semantics: $\overline{(\compress{\proba})}=\outproba$ and $\underline{(\compress{\proba})}=\inproba$.
Equivalence is defined on top of compression (see item~\ref{lrfgblsdlikf} of Definition~\ref{slergiuhpiu} for notation $\cons_\proba$):
%We say that two systems are equivalent if their compressed forms are isomorphic:
%
\begin{definition}[equivalence]
  \label{oeruihytersd}
  Two compressed mixed systems $\system$ and $\system'$ are \emph{equivalent} if they possess identical sets of variables $X{=}X'$, and  there exists a bijective map $\varphi:\cons_\proba\mapsto\cons'_{\proba'}$ satisfying the following conditions for every pair $(\omega,q)\in\Omega\times{Q}$, where $(\omega',q')\eqdef\varphi(\omega,q)$:
	\beq
	\omega\,\cons_{\proba}\,q \Leftrightarrow 
	\omega'\,\cons'_{\proba'}\,{q'} ~~;~~
	\proba'(\omega')=\proba(\omega) ~~;~~ q'=q\,.
	\label{dhsydjtdjyt}
	\eeq
	$\system$ and $\system'$ are \emph{equivalent}, written $\system{\equiv}\system'$, if their compressions
	are equivalent.\eproof
\end{definition}
The following result expresses that mixed system equivalence preserves probabilistic semantics:
\begin{lemma}
	\label{kerufgybob} 
Any two equivalent mixed systems, $S_1\equiv{S_2}$, possess identical probabilistic semantics: $\outproba_1=\outproba_2$ and $\inproba_1=\inproba_2$. 
\end{lemma}
\begin{proof} 
It is enough to prove the lemma in the following two cases: 1) $S_1$ and ${S_2}$ are both compressed, and 2): $S_2=\compress{S}_1$. The result is immediate for case 1), so we focus on case 2).
Let $Q$ be the common domain of $X_1=X_2$ and $A\subseteq{Q}$ be a state property. Then,
\[\bea{rcl}
\outproba_1(A) &=& \consist{\proba}_1\bigl(\{
\omega_1
\mid \exists{q}\in{A}: \omega_1\cons_1{q}
\}\bigr) \\
&=& \consist{\proba}_1\bigl(\{\omega_1
\mid \omega_1\in\compress{\omega}_1 \mbox{ and }
 \exists{q}\in{A}: \compress{\omega}_1\compress{\cons}_1{q}
\}\bigr)\\
&=& \consist{\compress{\proba}}_1\bigl(\{\compress{\omega}_1 \mid
 \exists{q}\in{A}: \compress{\omega}_1\compress{\cons}_1{q}
\}\bigr) = \overline{\compress{\proba}}_1(A) = \outproba_2(A)\,.
\eea\]
A similar proof holds for inner probabilities.
\eproof
\end{proof}

\begin{discussion}[equivalence] \rm
	\label{skduyfvsgkuy} 
	Floyd/Hoare/Dijkstra logic of pre- and postconditions for imperative languages was extended to encompass probability and nondeterminism with pGCL (probabilistic Guarded Command Language)~\cite{DBLP:journals/jcss/Kozen81,DBLP:journals/pacmpl/DahlqvistK20,DBLP:series/mcs/McIverM05,DBLP:conf/birthday/KatoenGJKO15,DBLP:journals/toplas/OlmedoGJKKM18,DBLP:conf/isola/McIverM20}. The semantics is defined as the probability of weakest preconditions under demonic nondeterminism. McIver-Morgan notions of refinement and equivalence follow from this semantics.
	%: for two commands $S$ and $T$, $S\sqsubseteq{T}$ holds if and only if, for every predicate $R$ over final states, the weakest preconditions of $S$ is weaker than that of $T$, where the comparison is formulated in terms of the above semantics. 
	This approach is also used to define equivalence of probabilistic programs, see, e.g., Section 3.1 of~\cite{meent2018introduction}.
	
	As pointed in Discussion~\ref{isedkufghsvjuy}, the above semantics parallels our consideration of outer/inner probabilities in point~\ref{ausytdfxuajy} of Definition~\ref{slergiuhpiu}. Compared to McIver-Morgan notion of equivalence, the notion of equivalence we propose in Definition~\ref{oeruihytersd} is more basic and direct. It implies equivalence of the evaluation of state properties using outer/inner probabilities.\eproof
%	\albert{REVOIR Equivalence of systems in probabilistic models, e.g., Probabilistic Automata is extensively studied, albeit with no relation to probabilistic programming, see Section~\ref{fdkjvblk} on related work. In turn, program equivalence is an issue of practical importance that is often discussed in the literature on probabilistic programming, see, e.g., Section 3.1 of~\cite{meent2018introduction}. Regarding formal studies, Dahlqvist and Kozen~\cite{DBLP:journals/pacmpl/DahlqvistK20} propose a categorical semantics of a probabilistic language with recursion. McIver and Morgan~\cite{DBLP:series/mcs/McIverM05} study an imperative probabilistic language supporting conditioning and probabilistic choice, and provide a semantics for it. For a language of the same family, Olmedo et al.~\cite{DBLP:journals/toplas/OlmedoGJKKM18} propose a notion of contextual equivalence of programs. To our knowledge, however, no direct and simple definition of equivalence exists that would support dynamical systems and their parallel composition.}
\end{discussion}

\subsubsection{Marginal}
For $(X,Y)$ a pair of random variables with joint distribution $P(x,y)$, the distribution of $X$ is given by the \emph{marginal} of $P$, namely: $P(x)\eqdef\sum_yP(x,y)$.

We extend this notion to Mixed Systems, by viewing it as a hiding operation, see Figure~\ref{ksduvbkudh}.
\begin{figure}[!h]
	\centerline{\scalebox{0.8}{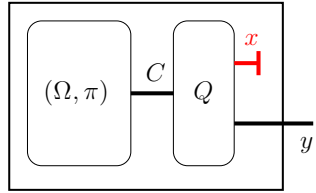}}
	\caption{The marginal on $y$ for the Mixed System of Figure~\ref{iweuftygi} is by hiding $x$ (in {\color{red}red}).}
	\label{ksduvbkudh}
\end{figure}
For $\cons\subseteq\Omega\times{Q}$ a relation where $Q$ is the domain of tuple $X$,  $Y\subseteq{X}$ a subset of variables, and $Z=X-Y$, we denote by
	\[
	\shortproj{Y}:2^{\Omega\times{Q}}\ra 2^{\Omega\times{Q_Y}}:~
	\proj{Y}{\cons}\eqdef\{(\omega,q_Y)\mid\exists q_Z:\omega\cons(q_Y,q_Z)\}
	\]
	the projection of $\cons$ over $Y$.
\begin{definition}[marginal]
		\label{eruigtyweouy} Let $\system=(\Omega,\proba,X,\cons)$ be a Mixed System, and let ${Y}\subseteq{X}$ be a subset of variables. The \emph{marginal} of $\system$ on $Y$, denoted by $\margin{Y}{\system}$, is the Mixed System
		$
			\margin{Y}{\system} \eqdef (\Omega,\proba,Y,\proj{Y}{\cons})
		$.\eproof
\end{definition}
		Even if $\system$ was itself compressed, due to the projection of relation $\cons$, the Mixed System defining the marginal in Definition~\ref{eruigtyweouy} may require a compression.

		\begin{example}\rm [Link with the classical notion of marginal for probabilities] \label{wsqytrdwuytr}\rm
				Let us apply Definition~\ref{eruigtyweouy} to the purely probabilistic system of Example~\ref{wuydtfuty}, namely
		$\system_{\sf proba}=(Q,\proba,\{X,Y\},\diag)\,,$
		having two variables $X,Y$, corresponding state space $Q$, and $\Omega=Q$ with $\cons=\diag$, the diagonal. This is the model of a pair $(X,Y)$ of visible variables with joint probability distribution $\proba(x,y)$, where $x$ and $y$ denote values for $X$ and $Y$, respectively. The projection of $\diag$ on $Y$ is
		$$\proj{Y}{\diag}=\{(x,y,y')\mid y=y'\}\,.$$
		Thus, $(x,y)\sim(x',y')$ if and only if $y=y'$. Thus, when using the formula of Definition~\ref{eruigtyweouy} to define $\margin{Y}{\system_{\sf proba}}$, the private probability space $(Q,\proba)$ must be compressed as $\compress{\proba}(y)=\sum_x\proba(x,y)$, showing that our notion of marginal boils down to the classical notion for probabilities in this case.\eproof
\end{example}

\subsubsection{Parallel composition}
Mixed Systems are equipped with a parallel composition: common state variables are unified (thus causing synchronization constraints); on the other hand, probabilistic parts remain local and independent, conditionally to the satisfaction of synchronization constraints. This is illustrated on Figure~\ref{jeyftfgjy}.
\begin{figure}[ht]
	\centerline{\scalebox{1}{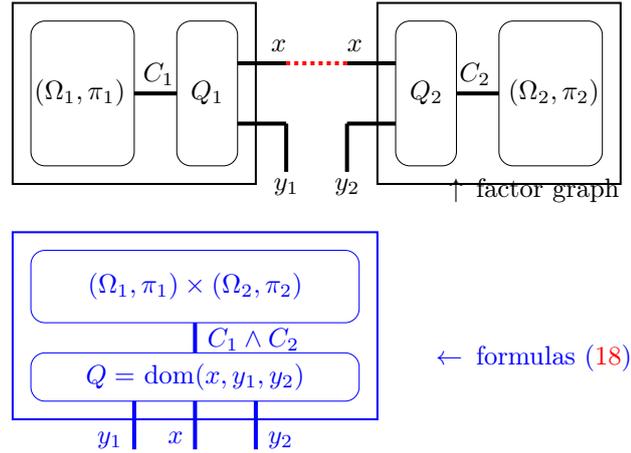}}
	\caption{Illustrating the parallel composition, for $y_1,y_2$ local variables and $x$ shared. The factor graph, capturing the connection via identical wires, is depicted on the top in black; the definition using formulas (\ref{sidtycfgdity}) is shown in blue.}
	\label{jeyftfgjy}
\end{figure}

%In contrast to other previous and forthcoming notions (marginal, conditional, kernel), parallel composition is not the extension of an existing notion. The only existing composition between probability spaces is their Cartesian product, making the two components independent, i.e., free of interaction. In contrast, our notion of parallel composition captures interaction between components.

Formally, let $I$ be a finite set, and, for each $i{\in}I$, let $X_i$ be a finite set of variables with domain $Q_i$, and set $X=\bigcup_{i{\in}I}{X_i}$ with domain $Q$. Say that tuple $(q_i)_{i{\in}I}$ is \emph{compatible}, written
\beq
\compat_{i{\in}I}\,{q_i}\,, \label{leriufygeoi}
\eeq
 if $q_i(x)=q_j(x)$ for any pair $(i,j)$ of indices and every shared variable $x\in{X_i}\cap{X_j}$. If $\compat_{i{\in}I}\,{q_i}$, their \emph{join}
\beq
\join_{i{\in}I}\,{q_i}\in{Q}  \label{leriughpiu}
\eeq
is defined  by $\join_{i{\in}I}\,{q_i}(x)=q_i(x)$ whenever $x\in{X_i}$. 
%Relation $\compat$ is reflexive, symmetric, and satisfies
%\beq
%\left.\bea{r}
%q_1\compat{q_2} \\ (q_1\join{q_2})\compat{q_3}
%\eea\right\} &\implies& \left\{\bea{lcl}
%q_1\compat{q_3} &,& q_1\compat(q_2\join{q_3}) \\
%q_2\compat{q_3} &,& q_2\compat(q_1\join{q_3})
%\eea\right.
%\label{dksdufygkuy}
%\eeq
%
\begin{definition}[parallel composition and Factor Graph]
	\label{lrtguiodtrhleguip}
	The  \emph{parallel composition} $\system_1{\mpara}\system_2$ of two mixed systems $\system_1$ and $\system_2$ is the Mixed System $\system$ such that:
%\vspace{-0.1cm}
\beq
        \bea{rcl}
        {X}&=&{X_1}{\cup}{X_2} \,,\,
	\Omega=\Omega_1{\times}\Omega_2 \,,\,\proba=\proba_1{\times}\proba_2 \mbox{~~(cartesian product), and} \\
        \cons&=&\bigl\{
    \bigl((\omega_1,\omega_2),q_1\join{q_2}\bigr)
    \mid
    q_1\compat{q_2} \mbox{ and } \omega_1\cons_1{q_1} \mbox{ and } \omega_2\cons_2{q_2}
  \bigr\}\,,
	\eea
	\label{sidtycfgdity}
	\eeq
	%where $\wedge$ denotes conjunction.
        %
				We attach to parallel composition $\bsystem=\mpara\!_{i\in{I}}\system_i$ its \emph{Factor Graph} $\cG_\bsystem$, which is a \emph{nondirected} bipartite graph whose set of vertices collects systems and variables:
				\[\bea{c}
				\bigl\{\system_i\mid{i}\in{I}\bigr\}\cup\bigl\{x\mid x\in\bigcup_{i\in{I}}{X_i}\bigr\},
				\eea\]
				and $\cG_\bsystem$ has edges $({\system_i},x)$, for $i\in{I}$ and $x\in{X_i}$, also denoted by $\monograph{\system_i}{{x}}$.\eproof
\end{definition}
The composition of two consistent systems may be
inconsistent.
Let
\beq
\nil&=&(\{1\},\delta_1,\emptyset,\{(1,\epsilon)\}) \label{htdwercfhstydr}
\eeq
be the \emph{nil} system, with trivial probability space $(\Omega,\proba)=(\{1\},\delta_1)$ and no visible variable; its state space is the singleton $Q_\nil=\{\epsilon\}$ where $\epsilon$ is some distinguished element, and its relation is the singleton $\cons=\{(1,\epsilon)\}$. The nil system is neutral for parallel composition: $\nil{\mpara}\system\equiv\system$ holds, for every $\system$.

Factor Graphs obey the following rule, where $\cup$ denotes the union of graphs:
				\beq
				\cG_{\system_1\|\system_2} &=& \cG_{\system_1}\cup\cG_{\system_2}\,.
				\label{eltrojheltor}
				\eeq
The associativity and commutativity of this parallel composition is immediate, as it is directly inherited from the same properties satisfied by the Cartesian product of probability spaces and the conjunction of relations. Factor Graphs and the parallel composition of Mixed Systems are useful in decomposing large but sparse systems, into a parallel composition of smaller, locally interacting, subsystems. 
\begin{lemma}
%[equivalence is a congruence]%[\cite{BenvenisteLFG95}]
  \label{wjdetyfuuy}
  $\system_1\equiv\system'_1$ implies 
  $\system_1{\mpara}\system_2\equiv\system'_1{\mpara}\system_2$, expressing that parallel composition preserves equivalence.
\end{lemma}
See Appendix~\ref{riuygfttyiohjih} for the proof.

\subsection{Bayesian Calculus and Bayesian Networks}
\label{selrujgnhelru}
So far Factor Graphs and related algorithms are able to capture joint distributions relating different statistical data, but they cannot capture causality, as argued by Judea Pearl~\cite{Pearl09}. Actually, Judea Pearl states that causality requires extra, structural, information that must be added to the specification of probability distributions: directed graphs are used to this end.

Another issue is that of incremental sampling of a compound system: whereas the sampling of a parallel composition is generally global (or using the sophisticated iterative methods used, e.g., in~\cite{JSSv076i01}), one could ask whether it could be performed incrementally. 

In statistics based on graphical models, these questions are answered by considering, in addition to Factor Graphs, so-called Bayesian Networks~\cite{DBLP:journals/ai/Pearl86}. Bayesian networks specify causality information by means of directed graphs, which bring the extra information advocated by J. Pearl to talk about caussality. Bayesian networks also naturally support incremental execution.
In this section, we show how these concepts supporting causality and incremental sampling, can be extended to Mixed Systems.

As a preamble, we recall some  facts from basic probability theory.
For a pair $(X,Y)$ of random variables with joint distribution $P(x,y)$, usual Bayes formula writes $P(x,y){=}P(y)P(x|y)$, where $P(y)\eqdef\sum_xP(x,y)$ is the \emph{marginal distribution} of $Y$ and $P(x|y)$ is the \emph{conditional distribution} of $X$ given that $Y{=}y$, assigning,  to each value $y$ of $Y$, a probability for $X$. Sampling $P(x|y)$ consists in 1) nondeterministically selecting a value for $y$, and then 2) with this value of $y$, sampling $X$ according to $P(x|y)$. $P(x|y)$ is called a \emph{transition probability,} or a \emph{probability kernel} or \emph{stochastic kernel,} depending on the contexts and communities: $y\mapsto P(x|y)$ maps any value for $Y$ to a probability distribution for $X$. We now extend these notions to Mixed Systems.

\subsubsection{Mixed Kernel}
We begin by extending the notion of probability kernel to that of Mixed Kernel. The starting idea consists in defining a Mixed Kernel as a function, mapping every $Y$-state of a set $Y$ of variables, to a system having $X$ as its set of variables.
For the notations used in the sequel, the reader is referred to the beginning of Section~\ref{uwtrdytrd}.

\begin{definition}[Mixed Kernel]
	\label{skjhfgkukduy}
	%\begin{enumerate}
		%\item
		A \emph{Mixed Kernel} (or simply \emph{kernel}) is a map 
		\[
		\kernel:Q_{X}\ra\Systems(X')\,,
		\]
	where $X$ and $X'$ are two finite sets of variables such that $X\cap{X'}=\emptyset$, called the sets of \emph{inputs} and \emph{outputs} of kernel $\kernel$. In the sequel, we shall denote these two sets $X$ and $X'$ by $\inal{\kernel}$ and $\outal{\kernel}$, respectively.
		
		The \emph{probabilistic semantics} of $\kernel$ is the pair of maps
		\beq
		q_{\rm in}\mapsto\bigl(\outproba(q_{\rm in}),\inproba(q_{\rm in})\bigr) 
		\label{dfksjgksjhjh}
		\eeq
		where $q_{\rm in}$ is a value for the input variables $\inal{\kernel}$, and $\outproba(q_{\rm in})$ and $\inproba(q_{\rm in})$ are the outer and inner probabilities associated to Mixed System $\kernel(q_{\rm in})$.\eproof
\end{definition}
For $q{\in}Q$ and $\cons\subseteq\Omega{\times}Q$, we write
%will use the notations
\beq
\cons_q &\eqdef& \{\omega{\in}\Omega\mid\omega\cons{q}\}, \mbox{ and }\cons_\omega \eqdef \{q{\in}Q\mid\omega\cons{q}\}\,.
\label{ltriguhruiph}
\eeq
\begin{convention}\rm
	\label{dwejdtuyt} A kernel $\kernel$ whose input set $X$ is empty identifies with the Mixed System $\system=\kernel(\epsilon)$ it defines, where $Q_X$ is the singleton $\{\epsilon\}$. Vice-versa, any system $\system$ identifies with the kernel $\kernel$ whose input set $X$ is empty and $\kernel(\epsilon)=\system$.\eproof
\end{convention}
%\albert{PAS UTILISE? By Lemma~\ref{ieduryfguy}, without loss of generality, we can assume (up to system equivalence) that all Mixed Systems $\kernel(q_{\rm in}),q_{\rm in}{\in}Q_{\inal{\kernel}}$, share the same private probability space:
%\beq
%\kernel(q_{\rm in})&=&(\Omega,\proba,X,\cons^{q_{\rm in}})
%\label{fsuywsegiukfi}
%\eeq
%where $\cons^{q_{\rm in}}$ is a relation that depends on the input state $q_{\rm in}$. \emph{We assume this form $(\ref{fsuywsegiukfi})$ for kernels in the sequel.} }
%\clearpage
\subsubsection{Bayesian Network}
		\begin{definition}[Bayesian Network] \label{lerifugeroigf}
Let $\cN{=}(\Vars\cup\Kernels,\depends)$ be a {directed} acyclic bipartite graph, where $\Vars$ and $\Kernels$ are finite sets of variables and Mixed Kernels, and $\depends\,\subseteq(\Vars{\times}\Kernels)\cup(\Kernels{\times}\Vars)$ is the set of edges. For $\kernel\in\Kernels$, we denote by ${\preset{\kernel}}$ and ${\postset{\kernel}}$ the sets of variables $x\in\Vars$ such that $x\depends\kernel$ and $\kernel\depends{x}$, respectively. $\cN$ is called a \emph{Bayesian Network} if satisfies the following conditions:
\beq
\forall\kernel\in\Kernels&\implies& \inal{\kernel}\subseteq\preset{\kernel} \mbox{ and }\outal{\kernel}=\postset{\kernel}\,.
\label{ifuglruifkjegk}
\\
\forall\kernel_1,\kernel_2\in\Kernels,\kernel_1\neq\kernel_2&\implies& \postset{\kernel_1}\cap\postset{\kernel_2}=\emptyset
\eeq
For convenience, we will denote by
\beq
\kernel_1\bayestimes\kernel_2 \label{ksduhcsvbkjhy}
\eeq
a Bayesian network $\cN{=}(\Vars{\cup}\Kernels,\depends)$ whose set $\Kernels$ contains only two Mixed Kernels $\kernel_1$ and $\kernel_2$, such that $\kernel_1\depends\kernel_2$ and $X=\inal{\kernel_1}{\cup}\outal{\kernel_1}{\cup}\inal{\kernel_2}{\cup}\outal{\kernel_2}$.\eproof
	\end{definition}
This notion is illustrated on Figure~\ref{wsuydrfdcsjytr} for two Mixed Kernels communicating via variable $x$ (compare with Figure~\ref{jeyftfgjy}).
\begin{figure}[!h]
	\centerline{\scalebox{1}{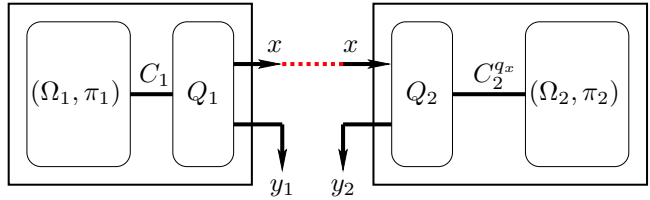}}
	\caption{Bayesian Network $\system_1\bayestimes\kernel_2$. Mixed Kernel $\kernel_2$ has input $x$.}
	\label{wsuydrfdcsjytr}
\end{figure}
\noindent To Bayesian Network $\cN{=}(\Vars{\cup}\Kernels,\depends)$, we associate the partial order \mbox{$(\Vars{\cup}\Kernels,\preceq)$}, where $\preceq$ is the transitive closure of $\depends$. In the  following definition, for $q$ a valuation of the set $X$ of variables and $\kernel$ a kernel belonging to $\Kernels$, $\restrict{\preset{\kernel}}{q}$ and $\restrict{\postset{\kernel}}{q}$ denote the restriction of $q$ to the variables belonging to $\preset{\kernel}$ and $\postset{\kernel}$, respectively. 
\begin{definition}[incremental sampling and probabilistic semantics]
	\label{freuiyfkeruy} 
The \emph{incremental sampling} of Bayesian Network $\cN$ is defined by structural induction over $\preceq$ as follows:
	\begin{enumerate}
		\item Initial condition: we assume a value for every variable $x\in\min(\Vars\cup\Kernels)$, where $\min$ refers to $\preceq$; we set $\Vars_-=\min(\Vars\cup\Kernels)\cap\Vars$ and $\Kernels_-=\emptyset$;
		\item Induction hypothesis: $\Vars_-\cup\Kernels_-\subseteq\Vars\cup\Kernels$ is a downward closed subset of vertices of $\cN$ such that 
		\begin{enumerate}
		\item $\postset{\Kernels}_-\subseteq\Vars_-$;
			\item every variable $x\in\Vars_-$ holds a value, whereas every $x\not\in\Vars_-$ does not;
		\end{enumerate}
		
		\item \label{dydjytjhg} Induction step: while ${\Vars}_-{\neq}\Vars$, do: 
		\begin{enumerate}
		\item let $\Kernels^*\subseteq\Kernels{-}\Kernels_-$ collect the kernels $\kernel$ such that $\preset{\kernel}{\subseteq}{\Vars}_-$ and $\postset{\kernel}{\neq}\emptyset$;
		\item	for every $\kernel{\in}\Kernels^*$, every variable belonging to $\preset{\kernel}$ holds a value, hence we can
			sample Mixed System $\kernel(q_{\preset{\kernel}})$, which returns a value for $q_{\postset{\kernel}}$; 
		\item \label{ksjhfvgks} doing this for all $\kernel{\in}\Kernels^*$ yields a value for every variable belonging to ${\Vars}_-\cup\postset{\Kernels^*}\supset{\Vars}_-$ (the inclusion is strict); 
		\item set ${\Kernels}_-:={\Kernels}_-\cup\Kernels^*$ and ${\Vars}_-:={\Vars}_-\cup\postset{\Kernels^*}$ and return to~\ref{dydjytjhg}.
		\end{enumerate}
		\item Done. 
	\end{enumerate}
	Sampling $\cN$ thus returns a value $q\in{Q_X}$ for every variable belonging to $X$, we denote this by $\cN\leadsto{q}$. 
	The \emph{probabilistic semantics} of $\cN$ is the map $q\mapsto\outproba(q)$, associating to every $q{\in}Q_X$ such that $\cN\leadsto{q}$, its \emph{probabilistic score}
	\beq
	%\bea{rcl}
	\outproba(q) &=& \prod_{\kernel\in\Kernels}~\outproba\bigl(\kernel,\restrict{\preset{\kernel}}{q}\bigr)\bigl(\restrict{\postset{\kernel}}{q}\bigr)\,.
	%\eea
	\label{kslufhsdbkj}	
\eeq
In (\ref{kslufhsdbkj}), $\outproba(\kernel,\restrict{\preset{\kernel}}{q})(\restrict{\postset{\kernel}}{q})$ is the score assigned to state  $\restrict{\postset{\kernel}}{q}$ by the outer probability associated to mixed system $\kernel(\restrict{\preset{\kernel}}{q})$.\eproof
\end{definition}
Since inclusion ${\Vars}_-\cup\postset{\Kernels^*}\supset{\Vars}_-$ in step \ref{ksjhfvgks} is strict, the inductive procedure terminates in finitely many steps. Note that, by (\ref{fjhksjhfj}), there is no need to consider $\inproba(q)$.
%
%\albert{Illustrer par un exemple?}
%
The inductive procedure of Definition~\ref{freuiyfkeruy} is formalized in Algorithm~\ref{alg:cap}.
\begin{algorithm}
\caption{Incremental sampling of Bayesian Network $\cN$}\label{alg:cap}
\begin{algorithmic}
\Require $\forall x\in\min(\Vars\cup\Kernels)$, $x$ is defined
\Ensure $\forall x \in X$, $x$ is defined
\State $\Vars_-\gets\Vars\cap\min(\Vars\cup\Kernels)$ and $\Kernels_-\gets\emptyset$
\While{${\Vars}_-{\neq}\Vars$}
\State $\Kernels^* \gets \{ \; \kernel \;\mid\; \preset{\kernel}{\subseteq}{\Vars}_- \mbox{ and } \postset{\kernel}{\neq}\emptyset \; \}$
\ForAll{$\kernel{\in}\Kernels^*$}
\State $\mbox{sample}(\kernel(\restrict{\preset{\kernel}}{q}))$
\EndFor
\EndWhile
\State ${\Kernels}_- \gets {\Kernels}_-\cup\Kernels^*$
\State ${\Vars}_- \gets {\Vars}_-\cup\postset{\Kernels^*}$
\end{algorithmic}
\end{algorithm}

\begin{definition}[Bayesian network equivalence]
	\label{esolfhsbdkuy} 
Let $\cN_1$ and $\cN_2$ be two Bayesian networks such that $X_1=X_2$. Say that $\cN_1$ and $\cN_2$ are \emph{probabilistically equivalent,} written $\cN_1\equiv_P\cN_2$, if they possess equal probabilistic semantics: $\outproba_1=\outproba_2$.\eproof
\end{definition}
By Lemma~\ref{kerufgybob}, $S\equiv{S'}$ implies $S\equiv_P{S'}$, when regarding mixed systems $S$ and $S'$ as Bayesian networks.

\begin{example}\rm [Finite Markov chain as a Bayesian Network]
\label{eorfyegfko} \rm Recall that a finite sequence of random variables $X_1,X_2,\dots,X_n$ is called a Markov chain if the joint distribution of $(X_0,X_1,\dots,X_n)$ factorizes as
$\proba(X_0{=}x_0,\dots,X_n{=}x_n) = \mu(x_0)\prod_{i=1}^nP(x_i| x_{i-1})$, where probability $\mu$ over $\bX$, the state space of the Markov chain, is the \emph{initial condition} and $P(x'| x)$ is the \emph{transition kernel,} i.e., for $x$ fixed, $x'\mapsto P(x'| x)$ is a probability over $x'$. Markov chains are thus a particular case of the Bayesian Networks proposed in Definition~\ref{lerifugeroigf}.\eproof
\end{example}
		We next extend, to mixed systems, the notion of conditional distribution. To this end, we will use the following notation: for $Y$ a set of variables and $q_Y\in{Q_Y}$,
		\beq
		(Y{=}q_Y)
	\label{elriuei}
		\eeq
		denotes the Mixed System defined as follows: $\Omega$ is the singleton $\{1\}$ with trivial probability on it, $Y$ is the set of variables, and $\cons=\{(1,q_Y)\}$ is a singleton, expressing that $Y$ is constrained to take the value $q_Y$.
\begin{definition}[conditional]
		\label{kerufygoieuyg} \label{lugtihlitu} Let $\system=(\Omega,\proba,X,\cons)$ be a Mixed System, and let ${Y}\subseteq{X}$ be a subset of variables. The \emph{conditional} of $\system$ on $Y$, denoted by $\cond{Y}{\system}$, is the kernel defined by $\cond{Y}{\system}(q_Y)\eqdef (\,Y{=}q_Y)\mpara\system$.\eproof
\end{definition}

\myparagraph{Link with the classical notion}
Consider the following particular case for $\system$: $\Omega{=}Q$, and $\cons$ is the diagonal of $\Omega{\times}Q$. Then, $\system$ specifies the joint distribution $\proba$ for tuple $X$ of random variables. Decompose $X=Y{\cup}{Z}$ where $Y{\cap}{Z}=\emptyset$. Compressing $\margin{Y}{\system}$ yields the marginal distribution of $Y$. Compressing $(Y{=}q_Y)\mpara\system$ yields the conditional distribution $\proba(q_Z|q_Y)$. Therefore, Definitions~\ref{eruigtyweouy} and~\ref{kerufygoieuyg} extend the notions of marginal and conditional existing on purely probabilistic systems. %\eproof
	\begin{discussion}[more on conditioning] \rm
	\label{olisubvsiu} When probability and nondeterminism are blended, the notion of Mixed Kernel serves the same purpose as \emph{soft} or \emph{stochastic} conditioning~\cite{DBLP:conf/icml/TolpinZRY21}, since it implements the stochastic conditioning $p(x\mid y\sim D)$ discussed in the introduction of~\cite{DBLP:conf/icml/TolpinZRY21}.\eproof
\end{discussion}
Generally, sampling the parallel composition $\system_1{\mpara}{\system_2}$ yields a result which differs from the incremental sampling of $\system_1\bayestimes\cond{X_1}{\system_2}$ (by Convention~\ref{dwejdtuyt} we can regard $\system_1$ as a kernel and consider this incremental sampling). Nevertheless, the following result holds (see Definition~$\ref{oeruihytersd}$ regarding isomorphic samplings):
%
		%\begin{lemma}
		%\label{lriguwhoiu}
		%If Mixed Systems $\system_1$ and $\system_2$ satisfy $\margin{X_1\!}{\system_1{\mpara}\system_2}\equiv\system_1$, then, $\system_1{\mpara}\system_2$ and $\system_1\bayestimes\cond{X_1}{\system_2}$ possess isomorphic samplings with $\varphi$ being the identity.
	%\end{lemma}
	%\begin{proof} See Appendix~\ref{ltroghporui}.
	%\end{proof}

%The following result is an immediate corollary of Theorem~\ref{oriufgheorliuy}:
	\begin{theorem}[Bayes formula] \label{reoifueorwui} Let $\system=(\Omega,\proba,X,\cons)$ be a Mixed System and  ${Y}\subseteq{X}$ a subset of variables. Then, the following \emph{Bayes formula} holds:\footnote{This theorem and formula (\ref{lugtihlitu}) correct the erroneous construction of the conditional $\cond{Y}{\system}$ in  Appendix A of~\cite{BenvenisteLFG95}.} 
	\[
	\system\equiv_P\margin{Y}{\system}\bayestimes\cond{Y}{\system}\,.
	\]
	\end{theorem}
		\begin{proof}
	See Appendix~\ref{lertughetpr} for the proof.\eproof
	\end{proof}
	
	The right hand side of Bayes' formula is illustrated on Figure\,\ref{dklfvjhdfbkjh}.
	\begin{figure}[ht]		
	\centerline{
	\begin{tabular}{lcr}
	marginal: hiding $X$ && conditional: kernel with input $Y$ \\
		 \scalebox{0.8}{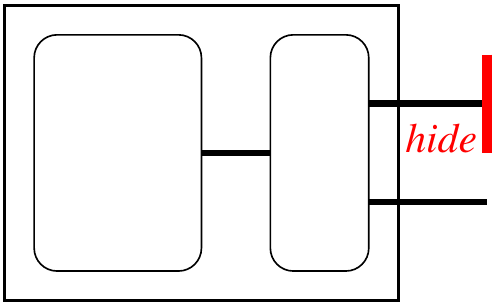} &\raisebox{1.2cm}{\Large{;}}&
		\scalebox{0.8}{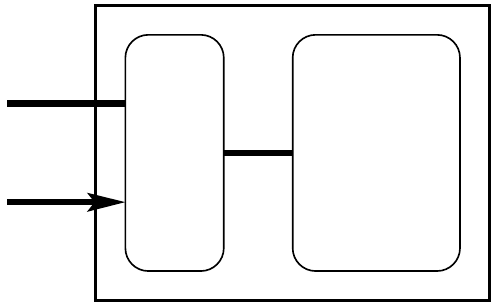}
	\end{tabular}
}
\caption{Illustrating the right hand side of Bayes' formula: the  output $Y$ of the system on the left is connected to the input $Y$ of the kernel on the right.}
\label{dklfvjhdfbkjh}
\end{figure}

By Definition~\ref{lrtguiodtrhleguip}, parallel composition
$
\bsystem=\prod_{\system\in\Systems}\system
$
defines a \emph{Factor Graph} $\cG_\bsystem$, having nondirected bipartite edges $\monograph{\system}{{x}}$, for every $\system\in\Systems$ and every visible variable $x$ of $\system$.
Message passing algorithms transform certain Factor Graphs associated to a parallel composition of several Mixed Systems, to Bayesian Networks while preserving the sampling. This provides such Factor Graphs with an incremental sampling:
\begin{theorem}[message passing algorithm]
	\label{eriuhpwiu} If Factor Graph $\cG_\bsystem$ of system $\bsystem$ is a  tree, we can transform it to a Bayesian Network $\cN_{\bsystem}$ while preserving its probabilistic semantics.
\end{theorem}
See Appendix~\ref{litrughuihyutrd} for a proof.
 %The basic step in message passing algorithms is formula (\ref{lreiugfepiu}).

\paragraph{Message passing algorithms for computing generalized likelihoods.} The purpose of probabilistic languages~\cite{doi:10.1002/sim.3680,JSSv076i01,DBLP:conf/icse/GordonHNR14} is not only (actually, not so much) sampling, but rather estimation/inference. Of course, in addition to performing incremental sampling, Bayes' formula also allows evaluating probabilities of properties incrementally. Then, a counterpart of Bayes' formula exists for performing maximum likelihood estimation incrementally---it is known in the pattern recognition literature as the Viterbi algorithm~\cite{DBLP:journals/tit/Forney72,Rabiner86anintroduction}. Theorem~\ref{eriuhpwiu} shows that message passing algorithms also allow for an incremental evaluation of generalized likelihoods.

\subsection{The {\sf ReactiveBayes} minilanguage}
In this section, we use the model of Mixed Systems to specify the semantics of the informal language we used in the introduction when discussing our running example. To make this precise, we formalize this informal language through the ``\plang'' syntax presented hereafter. 

To prevent from decidability issues in constraint solving, domains of variables and random variables are all assumed finite. Finally, to simplify our presentation of the syntax, domains are omitted. 

\subsubsection{Syntax}
Here is the syntax, where $\keyw{\prog{keywords}}$ are highlighted in blue:
\beq
\bea{lcl}
 e &::=& c \mid x 
%\mid \omega 
\mid (e,e) \mid op(e) \mid f(e)
\mid \keyw{\prog{pre}}\; x \mid \keyw{\prog{init}}\; x=c
%\\
 %e_v &::=& c \mid x  \mid (e_v,e_v) \mid op(e_v) \mid f(e_v)
%\mid \keyw{\prog{pre}}\; x
\\
%S &::=& \omega \,\keyw{\sim}\, P(e_v) \mid e=e \mid \keyw{\prog{observe}}\; x \mid \; S \progpara S
S &::=& x \,\keyw{\sim}\, P(e) \mid e=e \mid \keyw{\prog{observe}}\; x \mid \; S \progpara S
\eea
	\label{leughloiulikuh}
\eeq
\begin{itemize}
	\item 
An expression $e$ is a constant $c$, a {variable} $x$, 
%a {random} variable $\omega$, 
an external operator application $op(e)$, a function application $f(e)$, or a delayed version $\keyw{\prog{pre}}\; x$ for the variable $x$. Initial condition $\keyw{\prog{init}}\; x=c$ is required whenever $\keyw{\prog{pre}}\; x$ occurs in the program; it fixes the initial value for $x$.
%\item 
%A visible expression $e_v$ is the same, except that random variables are forbidden. 
\item 
%A Mixed System $S$ is the declaration of a \emph{prior distribution} $P(e_v)$ for random variable $\omega$, where \emph{distribution $P(e_v)$ has, optionally, parameters set by visible expression} $e_v$,
A Mixed System $S$ is the declaration of a \emph{prior distribution} $P(e)$ for variable $x$, thus making it random; {distribution $P(e)$ has, optionally, parameters set by expression} $e$,
an \emph{equation} $e=e$, the declaration that \emph{variable $x$ is actually observed}, or the parallel composition thereof. For each term $P$ we assume a semantics denoted by $\psem{P}$, which is a probability.
\end{itemize}
%The following \emph{well-formedness condition} is required:
%\beq
%\mbox{Mixed systems share no random variable,} \label{ksdjfhyvbskduy}
%\eeq
%i.e., if $S_1$ and $S_2$ are two declared systems, no random variable $\omega$ can occur in both. By the semantics of parallel composition, see $(\ref{leirughepui},v)$ below, Condition (\ref{ksdjfhyvbskduy}) ensures that systems interact through variables $x$ only, whereas random variables $\omega$ remain private. 
No provision is given by syntax (\ref{leughloiulikuh}) for writing equations relating systems. In particular, fixpoint equations $S= S'\progpara S$ cannot be expressed: \plang\ does not offer full recursion. However, statements $\prog{\keyw{pre}}$ and $\prog{\keyw{init}}$ provide a limited form of recursion, supporting dynamical systems. This will be made clear in Section~\ref{ikfuygskfuj}, where the semantics of full \plang\ will be given.
\begin{examp}\rm 
	\label{fkdjhsdfkudyk} Mixed System $S_1$, specified by model (\ref{ujsydtcfjhy}) writes
\beqq
\bea{l}
%~\prog{\keyw{where}}
  {\wemph{\|}} \prog{~\keyw{observe}~u}
%\\  \prog{\keyw{rec}~\keyw{init}~y = y0}
\\  {\progpara}\prog{~\keyw{init}~x = x0}
\\  {\progpara}\prog{~y = phi(u,\keyw{pre}~x)}
\\  {\progpara}\prog{~x = {if} fail {then} psi(y,noise) {else} y}
\eea\eeqq
Mixed System $S_2$, specified by model (\ref{sdujcytsdfjy}) writes
\beqq
\bea{l}
 {\wemph{\|}} \prog{~\keyw{init} noise = n0}
\\  {\progpara}\prog{~{noise} = chi(\keyw{pre} noise,w)}
\\ {\progpara}\prog{~{w}}~ \keyw{\sim} ~\prog{mu}
\eea
\eeqq
And so on. 
The global model is $S_1\progpara S_2\progpara S_3\progpara S_4$.\eproof
\end{examp}

\subsubsection{Semantics}
We now give the semantics of the static fragment of \plang, namely ignoring in (\ref{leughloiulikuh}) the statements $\keyw{\prog{pre}}$ and $\keyw{\prog{init}}$. $\sem{S}$ denotes the semantics of \plang\ program $S$:
\beq
\hspace*{-5mm}\bea{ccrcl}
(i)&& \sem{\keyw{\prog{observe}}\; x} &=& (\cdot,\cdot,\{x\},x=c)
\\ 
(ii)&& \sem{x \,\keyw{\sim}\, P} &=&(\Omega_x,\psem{P},\{x\},x=\omega_x)
\\ 
(iii)&& \sem{x \,\keyw{\sim}\, P(e)} &=&
%\displaystyle
c\mapsto\sem{x \,\keyw{\sim}\, P(c)} \mbox{, where } c=e
\\ 
(iv)&& \sem{e= e'} &=& (\cdot,\cdot,\mbox{vars}(e)\cup\mbox{vars}(e'),e=e')
\\ 
(v)&& \sem{S_1 \,\progpara\, S_2} &=& \sem{S_1}\mpara\sem{S_2}
\eea
\label{leirughepui}
\eeq
In $(i)$, the semantics has no probabilistic part, and a single visible variable $x$ whose value $c$ is given, but left unspecified.
In $(ii)$, probability distribution $P$ is fixed; the semantics consists of the probability space $(\Omega_x,\psem{P})$, where $\Omega_x$ is a private copy of the domain of $x$ equipped with probability $\psem{P}$ and having generic element $\omega_x\in\Omega_x$; equation $x=\omega_x$ exposes $\omega_x$ for further interactions through $x$. In $(iii)$, the probability depends on an expression $e$, whose generic value is denoted by $c$; the semantics is the kernel mapping $c$ to $\sem{x \,\keyw{\sim}\, P(c)}$. Line $(iv)$ defines the semantics of equations; ``vars$(e)$'' denotes the set of variables involved in expression $e$; the semantics has no probabilistic part. Finally, $(v)$ makes this semantics structural.
Thanks to formula (\ref{eltrojheltor}) of Definition~\ref{lrtguiodtrhleguip}, it also defines the Factor Graph representing $S$.

The following fragment of (\ref{leirughepui}) is mapped to a Bayesian Network. In the following formulas, $\gsem{S}$ denotes the Bayesian Network defined by $S$, when it exists:
\begin{equation}\bea{cclcl}
(i)&& \gsem{\keyw{\prog{observe}}}&=&\{\source{x}\}
\\
(ii)&& \gsem{x \,\keyw{\sim}\, P}&=&\{x\}
\\
(iii)&& \gsem{x \,\keyw{\sim}\, P(e)}&=&
\mbox{vars}(e)\ra{\sem{x \,\keyw{\sim}\, P(e)}}\ra{x}
\\
(iv)&& \gsem{x=e}&=& \mbox{vars}(e)\ra{\sem{x=e}}\ra{x}
\\
(v)&& \gsem{S_1 \,\progpara\, S_2}&=&\gsem{S_1}\cup\gsem{S_2}
\eea
\label{ltrughiuyh}
\end{equation}
In $(i)$, $\source{x}$ denotes $x$ flagged with the condition that it must remain a source node in any of its environments. Application of Rule $(v)$ is subject to the following success conditions:
\begin{condit}[success conditions]
	\label{fukfysdkuy} \
	\begin{enumerate}
		\item The union $\gsem{S_1}\cup\gsem{S_2}$ possesses no circuit and satisfies the conditions of Definition~$\ref{lerifugeroigf}$, and
		\item The result keeps satisfying all inherited conditions $(i)$.
	\end{enumerate}
\end{condit}
%The success conditions of rule $(v)$ are that: (1) the union $\gsem{S_1}\cup\gsem{S_2}$ possesses no circuit and satisfies the conditions of Definition~\ref{lerifugeroigf}, and (2) it keeps satisfying all inherited conditions $(i)$. 
These conditions ensure that parallel compositions are incremental.
%Definition~\ref{lrtoghrlitu}.
The message Passing algorithm presented in Theorem~\ref{eriuhpwiu} allows source-to-source rewriting for mapping tree shaped non-directed Factor Graphs to directed Bayesian Networks. 
%
%This is illustrated on the parallel composition $\prog{S1}\;\progpara\;\prog{S2}\;\progpara\;\prog{S3}\;\progpara\;\prog{S4}$. 
\begin{examp}
%[${S_1}\;\progpara\;{S_2}\;\progpara\;{S_3}\;\progpara\;{S_4}$]
\rm
	\label{skduskufyk} 
The following picture displays, on the top, the Factor Graph associated to ${S_1}\;\progpara\;{S_2}\;\progpara\;{S_3}\;\progpara\;{S_4}$, and, on the bottom, the Bayesian Network resulting from applying the message passing algorithm---for better readability we show only shared variables:
\newcommand{\graphhor}{
~
{S_2}\frac{~~~}{~}f_n\frac{~~~}{~}{S_1}\frac{~~~}{~}v_n\frac{~~~}{~}{S_3}}
%
%\hspace*{14mm}\raisebox{-0.9cm}{\scalebox{0.9}{\input{graphS.pdf_t}}} \hfill \scalebox{0.9}{$\bea{c}\cond{\prog{y}}{\prog{S4}} \\ \uparrow \\ \prog{y} \\ \uparrow \\
		%~~~~\cond{\prog{fail}}{\prog{S3}}\la\prog{fail}\la
	%{\prog{SS1}}\ra\prog{noise}\ra\cond{\prog{noise}}{\prog{S2}}
	%\eea$}
	%
	\begin{center}{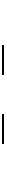}
	\\ ~ \vspace*{-2mm} \\
	{$\bea{c}\cond{y}{{S_4}} \\ \uparrow \\ y_n \\ \uparrow \\
		~\cond{f}{{S_3}}\la{f_n}\la
	{\mathbf{S}_1}\ra{v_n}\ra\cond{v}{{S_2}}
	\eea$}
\end{center}
	\noindent where $\mathbf{S}_1\eqdef{S_1}\mpara\margin{{v}}{{S_2}}\mpara\margin{{f}}{{S_3}}\mpara\margin{{y}}{{S_4}}$.\eproof
\end{examp}
\begin{discussion}[if-then-else] \rm
	\label{sleifsubliu} 
In Example~\ref{fkdjhsdfkudyk}, system $S_1$ involves an ``if-then-else'' statement. Syntax (\ref{leughloiulikuh}), however, does not involve such statements. This means that ``if-then-else'' statements are seen by syntax (\ref{leughloiulikuh}) as one instance of ``$f$'', to which no particular attention is paid. The semantics of this ``$f$'' obviously depends on the value of the Boolean control signal. However, neither the factor graph, nor the Bayesian network associated to $S_1$, depend on which branch is active in this ``if-then-else'' statement. This is harmless if the focus is on modeling. Considering ``if-then-else'' and paying attention to it is definitely needed in probabilistic reasoning~\cite{DBLP:journals/toplas/ChatterjeeFNH18}, see Appendix~\ref{kuyfgksuy}. The same holds when performing inference or learning~\cite{DBLP:conf/fsttcs/HurNRS15}; see also the discussion of objective~\ref{skdufvbnfgklj} of probabilistic programming on page~\pageref{skdufvbnfgklj}.\eproof
\end{discussion}
So far we have presented models involving no dynamics. In the next section we move to our proposed formal model for dynamical systems:
%reactive probabilistic programming~\cite{rppl_pldi20}: 
Mixed Automata.

%\subsection{\textbf{\emph{Comparison with imperative probabilistic programming}}}

\section{Mixed Automata}
\label {roe87wyhfgbsehlrigu}
The  idea is simple:
we upgrade notions, from automata, to Probabilistic Automata, and to Mixed Automata:
\begin{enumerate}
	\item
Transitions $\trans{q}{\action}{q'}{}$, where $q$ and $q'$ are states and $\action$ is an action, correspond to automata.
\item
Upgrading them to \mbox{$\trans{q}{\action}{\proba'}{}\produces{}{q'}$}, where $\proba'$ is the next probabilistic state and $\produces{}{}$ denotes probabilistic sampling, yields Simple Probabilistic Automata following Segala and Lynch~\cite{SegalaL94,DBLP:conf/concur/LynchSV03}.
\item
Upgrading them further to $\trans{q}{\action}{\system'}{}\produces{}{q'}$, where $\system'$ is a Mixed System and $\produces{}{}$ denotes sampling, yields Mixed Automata.
\end{enumerate}

\subsection{Definition and properties}
The formal definition is introduced next. It uses the notation $\Systems(X)$, introduced at the end of Definition~\ref{slergiuhpiu}. We assume an underlying alphabet ${\Alphabet}$ of \emph{actions}.
\begin{definition}[Mixed Automaton]\label{def-iohmdp} A \emph{Mixed Automaton} is a tuple
$$M=({\alphabet},{X},q_0,\ra),$$ 
where: $\alphabet\subseteq\Alphabet$ is a finite set of actions,
$X$ is a finite set of variables having  domain $Q=\prod_{x\in{X}}Q_x$, $q_0 \in Q$ is the
  \emph{initial} state, and $\ra\;\subseteq\,{Q}{\times}{\alphabet}{\times}\Systems(X)$ is the \emph{transition relation}.
	We write
  $$\mbox{$\trans{q}{\action}{\system}{}$ (or
  {$\trans{q}{\action}{\system}{M}$ when we wish to make $M$ explicit)}}$$
	to mean
  $(q,\action,\system) \in \;\ra$.
We require that $M$ shall be \emph{deterministic}:
%\footnote{Since Mixed Systems are targets of Mixed Automata transitions, Mixed Automata still capture nondeterminism.} 
\beq
\mbox{for any pair
$(q,\action)\in{Q}{\times}\alphabet$, $\trans{q}{\action}{\system}{}$
and $\trans{q}{\action}{\system'}{}$ implies $\system {=} \system'$.}
\label{hgfcjgcvj}
\eeq
The \emph{sampling} of $M$ is its set of \emph{runs} $\trace$, which are finite sequences of chained transitions:
\beq
\trace&=& \fullprobatrans{q_0}{\action_1}{\system_1}{}{q_{1}}
\fullprobatrans{}{\action_2}{\system_2}{}{q_{2}}
\dots
\fullprobatrans{q_{k-1}}{\action_{k}}{\system_{k}}{}{q_{k}}\,,
\eeq
where Mixed Systems $\system_1,\dots,\system_k$ are consistent, and $\produces{\system}{q}$ is the sampling introduced in Definition~$\ref{slergiuhpiu}$.\eproof
\end{definition}
The transitions of Mixed Automata target  Mixed Systems, which combine nondeterminism with probabilities. Therefore, Mixed Automata capture nondeterminism despite Condition (\ref{hgfcjgcvj}).
\begin{example}\rm [comparing with classical notions] 
	\label{kfsudyfgkjuy} \rm 
	Let $(X_n)_{n{\geq}0}$ be a Markov chain with state space $Q$, initial state $q_0$, and transition probability $P(q'\,|\,{q})$. We can reformulate it as the  Mixed Automaton $M=({\alphabet},{X},q_0,\ra)$, where: $\alphabet$ is the singleton $\{\alpha\}$; variable $X$ has   domain $Q$; $\ra$ maps $(q,\action)$ to the purely probabilistic Mixed System of Example~\ref{wuydtfuty}, representing probability $q'\mapsto P(q'\,|\,{q})$ for given state $q$.\eproof
\end{example}
%Compression and equivalence of Mixed Systems carry over to Mixed Automata
Like automata and Probabilistic Automata, Mixed Automata come equipped with a notion of  parallel composition, built on top of the parallel composition of Mixed Systems. The simplest idea is that the transitions of parallel composition $M_1\mpara{M_2}$ will take the form $\trans{q_1{\join}q_2}{\action}{\system_1'\mpara\system_2'}{}\produces{}{q'_1{\join}q'_2}$, where $q'_1{\join}q'_2$ and $\system_1'\mpara\system_2'$ are defined in (\ref{leriughpiu}) and Definition~\ref{lrtguiodtrhleguip}, respectively. In this simple construction, synchronizing the two transitions is by having them perform the same action $\action$. 

To be able to define the semantics of our \plang\ minilanguage, we will, however, need the more flexible synchronization mechanism of ``compatible actions''---this is known to be only a technical extension. We thus assume that the underlying alphabet ${\Alphabet}$ of actions is equipped with a commutative and associative \emph{join} partial operation \mbox{$\ajoin:\Alphabet\times\Alphabet\ra\Alphabet$}, where $\action_1\ajoin\action_2$ is defined whenever the two actions are \emph{compatible,} written $\action_1\acompat\action_2$. In the composition of Mixed Automata, the components synchronize on compatible actions and move to the parallel composition of target systems by performing the join of the two actions:
\begin{definition}[parallel composition]\label{rtghltughtui}
  Let $M_1$ and $M_2$ be two Mixed Automata having compatible initial states $q_{0,1}{\compat}q_{0,2}$. Their
  \emph{parallel composition} $M_1{\mpara}{M_2}$ has alphabet $\alphabet_1{\cup}\alphabet_2$, set of
  variables $X_1{\cup}X_2$, and initial state
  $q_{0,1}{\join}q_{0,2}$. Its transition relation $\longrightarrow_M$ is the minimal
  relation satisfying the following condition, where $\system_1 {\mpara} \system_2$ was defined in
	Definition~$\ref{lrtguiodtrhleguip}$:
  \[\!\!
	\left.\bea{c}
\trans{q_i}{\action_i}{\system_i}{M_i} \mbox{ for }i=1,2 \\
q_1\compat{q_2} \mbox{ ~and~ } \action_1\acompat{\action_2}\eea\right\}
\implies \trans{q_1\join{q_2}}{\action}{\system_1 {\mpara} \system_2}{M}, \mbox{ where } \action=\action_1\ajoin\action_2\,.\mbox{ \eproof}
\]
%\smallskip
\end{definition}
The next important notion is that of (bi)simulation, which is central in automata theory. We upgrade it, from the basic notion for automata up to the extended notion for Mixed Automata:
%The lifting of simulation relations, from states to probabilistic states, was already considered by authors~\cite{Seg06-CONCUR}. An easy extension of this reasoning allows us to define the lifting of simulation relations, from states to Mixed Systems. This in turn will provide us with the notion of (bi)simulation, for Mixed Automata. More precisely:
\begin{enumerate}
	\item
In the context of automata, relation $\leq$ on pairs of states is a \emph{simulation} if it satisfies~\cite{Seg06-CONCUR}:
		\beqq
	\left.\bea{c}
	\trans{q_1}{\action}{q'_1}{} \\ q_1\leq{q_2}
	\eea\right\}\implies \exists q'_2:\left\{\bea{c}
	\trans{q_2}{\action}{q'_2}{} \\ q'_1 \leq {q'_2}
	\eea\right.
	\eeqq
	\item
	This definition is upgraded to Probabilistic Automata as follows~\cite{Seg06-CONCUR}:
		\beqq
	\left.\bea{c}
	\trans{q_1}{\action}{\proba'_1}{} \\ q_1\leq{q_2}
	\eea\right\}\implies \exists \proba'_2:\left\{\bea{c}
	\trans{q_2}{\action}{\proba'_2}{} \\ \proba'_1\leq^{P}{\proba'_2}
	\eea\right.
	\eeqq
	where $\leq^P$ is the \emph{lifting of $\leq$ to pairs of probabilistic states.} We have:
	\beq
	\mbox{
	\begin{minipage}{7.7cm}
		 $\proba'_1\leq^{P}{\proba'_2}$ ensures, for each $q'_1$ such that $\proba'_1\leadsto{q'_1}$, the existence of $q'_2$ satisfying $\proba'_2\leadsto{q'_2}$ and $q'_1\leq{q'_2}$.
	\end{minipage}
	} \label{lrtoghpruih}
	\eeq
	\item
	This definition will be further upgraded to Mixed Automata as follows:
		\beq
	\left.\bea{c}
	\trans{q_1}{\action}{\system'_1}{} \\ q_1\leq{q_2}
	\eea\right\}\implies \exists \system'_2:\left\{\bea{c}
	\trans{q_2}{\action}{\system'_2}{} \\ \system'_1\leq^{S}{\system'_2}
	\eea\right.
	\label{rtpioghpo}
	\eeq
	where $\leq^S$ is the \emph{lifting of $\leq$ to pairs of Mixed Systems.} We request:
	\beq
	\mbox{
	\begin{minipage}{7.9cm}
		 $\system'_1\leq^{S}{\system'_2}$ shall ensure, for each $q'_1$ such that $\system'_1\leadsto{q'_1}$, the existence of $q'_2$ satisfying $\system'_2\leadsto{q'_2}$ and $q'_1\leq{q'_2}$.
	\end{minipage}
	} \label{lrigtuhpoui}
	\eeq
\end{enumerate}
	Such a lifting is introduced next. Let $\system_1$ and $\system_2$ be two Mixed Systems.
	\begin{definition}[lifting relations on Mixed Systems states]
  \label{hrgfuihsk} Let $\simu\subseteq{Q_1}{\times}{Q_2}$ be any state relation. Mixed System relation {$\NMPlift{\simu}\subseteq\Systems(X_1){\times}\Systems(X_2)$} is the \emph{lifting} of $\simu$ if
  there exists a \emph{weighting} function
  \mbox{$w:\Omega_1{\times}\Omega_2\ra[0,1]$} such that:
	%\footnote{Note that this definition by itself does not require $\simu$ to be a preorder, it applies to any relation. Of course, we will then apply the lifting to simulation preorders.} 
	%\vspace{-0.1cm}
\begin{enumerate}
\item \label{sggouigh} For every triple $(\omega_1,\omega_2,q_1)\in\Omega_1{\times}\Omega_2{\times}Q_1$ such
  that $w(\omega_1,\omega_2)>0$ and $\omega_1\,\cons_1\,{q_1}$, there
  exists $q_2\in{Q_2}$ such that $q_1\,\simu\,{q_2}$, and
   $\omega_2\,\cons_2\,{q_2}$;
\item \label{leiurlyui} Weighting $w$ projects to $\proba_1$ and $\proba_2$:
  $$\mbox{\hspace*{16mm}}\sum_{\omega_2}w(\omega_1,\omega_2){=}\proba_1(\omega_1) \mbox{ and }\sum_{\omega_1}w(\omega_1,\omega_2){=}\proba_2(\omega_2).\mbox{\hspace*{17mm}\eproof}$$
\end{enumerate}
\end{definition}
By construction, this definition for the lifting of state relations to relations on Mixed Systems satisfies (\ref{lrigtuhpoui}).
%This notion is new.
Note the existential quantifier in Condition~\ref{sggouigh}.  By
Condition~\ref{leiurlyui}, $w$ induces a probability on
$\Omega_1{\times}\Omega_2$. We write $\system_1\NMPlift{\simu}\system_2$
to mean $(S_1,S_2)\in\NMPlift{\simu}$.

\begin{discussion}[lifting and coupling] \rm
	\label{kwsedufyvgbskuy} Our lifting is a direct extension of the technique used in~\cite{Seg06-CONCUR} for Probabilistic Automata. In the context of probabilistic reasoning, the same technique was also extensively studied under the name of \emph{probabilistic coupling}~\cite{DBLP:conf/lpar/BartheEGHSS15,DBLP:journals/corr/abs-1710-09951}. Weighting function $w(\omega_1,\omega_2)$ of Definition~\ref{hrgfuihsk} transposes probabilistic coupling to our model of Mixed Automata in which nondeterminism and probability are combined. In a different community, 	``stochastic non-determinism'' was extensively studied through the notion of \emph{Non-deterministic labelled Markov process} in~\cite{DBLP:journals/mscs/DArgenioTW12,DBLP:journals/logcom/DoberkatT17}, in a categorical framework; the second reference encompasses continuous distributions (beyond discrete).%\eproof
\end{discussion}
\begin{lemma}
  \label{egfuioehrpo} $\system_1\NMPlift{\simu}\system_2$ and
  $\system'_1{\equiv}\system_1$ together imply
  $\system'_1\NMPlift{\simu}\system_2$.
\end{lemma}
See Appendix~\ref{elrgfuilyu} for a proof.\eproof

\begin{definition}[simulation]\label{def:simulation}
 Given two Mixed Automata
  $M_1,M_2$, we say that $M_2$ \emph{simulates} $M_1$,
  written $M_1 {\leq} M_2$, if they possess a \emph{simulation,} i.e., a relation $\leq \
  \subseteq \, Q_1 {\times} Q_2$ such that $q_{0,1} {\leq} q_{0,2}$\ and,
 for every pair $q_1 {\leq} q_2$ and every transition
$\transindex{q_1}{\action}{\system_1}{1}$,
  there exists a transition
  $\transindex{q_2}{\action}{\system_2}{2}$
  such that
  $\system_1\NMPlift{\leq}\system_2$, where $\NMPlift{\leq}$ denotes the lifting of $\leq$.
$M_1$ and $M_2$ are called \emph{simulation equivalent} if they simulate each other. $M_1$ and $M_2$ are called \emph{bisimilar} if there exists a relation $\sim\, \subseteq \, Q_1 {\times} Q_2$ such that both $\sim$ and its transpose are simulations.\eproof
\end{definition}
%
%\albert{
\begin{discussion}[simulation equivalence vs. bisimilarity]\rm 
	\label{kdufhvb} Despite the condition (\ref{hgfcjgcvj}) that the transition relation shall be deterministic, the two notions of ```simulation equivalence'' and ``bisimilarity'' differ. The reason is that nondeterminism is hidden behind the Mixed Systems targeted by transitions. Actually, we will prove in our forthcoming Theorem~\ref{elriguhui} that Segala's Probabilistic Automata~\cite{Seg06-CONCUR,SegalaL94,DBLP:conf/concur/LynchSV03}, which possess nondeterministic transition relations, can be embedded into Mixed Automata while preserving simulations.\eproof
\end{discussion}
%}

The notion of simulation and its derived constructs are the core topic of the literature on automata and their probabilistic extensions. The reader is referred to the next section for a bibliographical discussion.
\begin{lemma}
	\label{glrtukghtrllsdukfg} Parallel composition preserves simulation:
	 $M'_1{\leq}{M_1}$ and $M'_2{\leq}{M_2}$ together imply $M'_1{\mpara}{M'_2}\leq{M_1}{\mpara}{M_2}$.
\end{lemma}
See Appendix~\ref{wroifuwgopi} for a proof.\eproof
\begin{discussion}[Mixed Automata are causal in time]\rm
	\label{dlfviujnbvl} 
Mixed Automata 
\linebreak
remain a \emph{causal} model in time, since the current transition depends on the past, not on the future. Consequently, Mixed Automata cannot be used to specify acausal estimation problems, e.g., estimating unmeasured variable $z_k$ based on observations of $X_0,\dots,X_k,\dots,X_N$.	
To perform this, we must ``unfold time as space'', i.e., regard $X_0,\dots,X_N$ as a $(N{+}1)$-tuple of variables, not as successive occurrences in time of variable $X$. 
Note that the transition relations of Mixed Automata inherit, from Mixed Systems, the Bayesian Calculus and the notions of Factor Graph and Bayesian Network.
%\eproof
\end{discussion}

%\clearpage
\subsection{Mixed Automata for the semantics of {\sf ReactiveBayes} }
\label{ikfuygskfuj}
In this section we first give the semantics of the full \plang\ minilanguage (\ref{leughloiulikuh}) in terms of Mixed Automata. Recall that the semantics of the static part of the language was given in (\ref{leirughepui},$(i)$--$(v)$).

\myparagraph{Notations}
To every variable $x$, we associate its successive \emph{previous versions} $\preset{x},\ppreset{2}{x},\ppreset{3}{x},\dots$, where
\beq
\ppreset{(n+1)}{x}\eqdef\preset{(\ppreset{n}{x})}&\mbox{ and }&Q_{\preset{x}}=Q_x\,.
\label{sjcuhsgdvjcuhy}
\eeq
Then, we define 
\beq
\preset{e}(x) &\eqdef& e(\preset{x})
\label{jdseytdsejy}
\eeq
as being the expression $e$ in which every variable $x$ is replaced by its previous version $\preset{x}$.
We will use the Mixed System $({x}{=}q_x)$, defined in (\ref{elriuei}): this system has trivial probabilistic part, variable $x$, and enforces the value $q_x$ for it.\eproof 

We begin with  delay $\prog{\keyw{pre}}$ and initialization $\prog{\keyw{init}}$:
\beq\bea{ccrcl}
(vi) &&\sem{\keyw{\prog{pre}}\; x} &=& \bigl(\,\{\true\}\,,\,\{x,\preset{x}\}\,,\,-\,,\{\trans{q}{\true}{(\preset{x}{=}q_x)}{}\mid\forall{q}\in{Q}\}\bigr)
\\ [2mm]
(vii) &&\sem{\keyw{\prog{init}}\; x=c} &=&
\bigl(\,\{\true\}\,,\,\{x\}\,,\,c\,,\,\trans{c}{\true}{\nil}{} \mbox{ and }\trans{\epsilon}{\true}{\nil}{}\,\bigr)
\\ [2mm]
\eea
\label{ekuygoutrhpgu}
\eeq
The semantics of $\prog{\keyw{pre}}$ is stated in $(vi)$. It is the Mixed Automaton with trivial action alphabet (singleton $\{\true\}$), two variables $x$ (receiving the current value) and $\preset{x}$ (delivering the previous value), an undefined initial state, and the set of  transitions
$$\trans{q}{\true}{(\preset{x}=q_x)}{}\,,$$
where $q$ ranges over the set of all states and $q_x$ is the $x$-coordinate of $q$---this transition relation formalizes the constraint that $(\prog{pre}\;x)_n$ holds the value of $x_{n-1}$.

Since the initial state is undefined in the delay statement, a specification of the initial value is required, by using initialization statement $\prog{\keyw{init}}$. Its semantics is stated in $(vii)$, where $\nil$ is the trivial Mixed System defined in (\ref{htdwercfhstydr}). This Mixed Automaton possesses $x$ as its only variable, $c\in{Q_x}$ as its initial state, and otherwise does nothing, i.e., sets no constraint on its environment.

\medskip

So far we have completed the semantics of \plang\ as defined in (\ref{leughloiulikuh}), for which actions were not used---only the trivial ``true'' action was used in the semantics. Since Mixed Automata is a richer framework, it can support the following richer language involving state machines, by adding the following syntax, with reference to (\ref{leughloiulikuh}):
\beq
\bea{lcl}
\action &::=& \preset{e} \mbox{, where $e$ has Boolean type}
\\
A &::=&
\keyw{\prog{{on}}}\;\action\; \keyw{\prog{then}}\; S \; \keyw{\prog{else}}\; S \mid \; A \progpara A
\eea
	\label{lersoghepug}
\eeq
Actions $\action$ are previous versions of expressions of Boolean type.
In the additional statement ``$\keyw{\prog{{on}}}\;\action\; \keyw{\prog{then}}\; S \; \keyw{\prog{else}}\; S$'', actions $\action$ and $\neg\action$ trigger the transition leading to the first and second system, respectively. If $\action$ is the constant ``true'', we simply write ${S}$ instead of ``$\keyw{\prog{{on}}}\;\mbox{true}\; \keyw{\prog{then}}\; S$''.

We now give the corresponding semantics ($\true$ denotes the Boolean value ``true'', and we refer the reader to Definition~\ref{slergiuhpiu} regarding $\nil$ and the distinguished state $\epsilon$):
\beq\hspace*{-8mm}\bea{ccrcl}
(viii) &&\sem{\prog{\keyw{on}}\;\action\; \keyw{\prog{then}}\; S \; \keyw{\prog{else}}\; S'} &=& \frac{
\mbox{$S$ and $S'$ have previous state $p$}
}{
\left(\bea{c}\{\action,\neg\action\}
\;,\;
X\cup{X'}
\;,\;
\cdot
\\ [1mm]
\bigl\{\trans{p}{\action}{S}{},\trans{p}{\neg\action}{S'}{}\bigr\}
\eea\right)
}
\\ [10mm]
(ix) && \sem{A_1 \,\progpara\, A_2} &=& \sem{A_1}\;\mpara\;\sem{A_2}
\eea
\label{ortghuplgh}
\eeq
The right hand side of $(viii)$ is an inference rule meaning {``\emph{numerator entails denominator}''}. By (\ref{jdseytdsejy}) and the syntax for actions in (\ref{lersoghepug}), action $\action$ in $(viii)$ is evaluated by using the previous state $p$. 
At a given instant, the previous state is known, and can thus be used as the source state of the two transitions. The initial state is left unspecified. 
Focus on the parallel composition $(ix)$. With reference to Definition~\ref{rtghltughtui}, we now formalize the compatibility relation $\acompat$ and the join operator $\ajoin$:
\beq
\mbox{
	  $\action_1\acompat\action_2$ always holds, and $\action_1\ajoin\action_2\eqdef\action_1\wedge\action_2$.
}
\label{letroighui}
\eeq
%In particular, if $x$ is a variable of $A$, then ${\keyw{\prog{init}}\; x=c} \progpara A$ results in enforcing $x$ to have initial value $c$ in $A$.
%

\section{Comparison with Segala's Probabilistic Automata}
\label{45oiuepiruh}

Probabilistic Automata (\pa)~\cite{Seg06-CONCUR,SegalaL94,DBLP:conf/concur/LynchSV03} were originally proposed by Segala and Lynch. To simplify our comparison, we discuss here the version of \pa\ with no consideration of internal actions. According to the classification made by Sokolova and de Vink~\cite{SokolovaV04}, we study the link with both the Simple (Segala) Probabilistic Automata and the (Segala) Probabilistic Automata. For the former, actions are selected  and then a transition to a probabilistic state is selected nondeterministically. For the latter, both the action and a  state are jointly selected, probabilistically. This distinction is referred to as reactive vs. generative models in~\cite{SokolovaV04}.

%\albert{REVOIR 
Simple Probabilistic Automata existed way before the work of Segala and Lynch~\cite{Seg06-CONCUR,SegalaL94,DBLP:conf/concur/LynchSV03}, in the community of applied mathematics and probability theory, where they are known under the name of \emph{Markov Decision Processes (MDP)}~\cite{bellman1957markovian,puterman2014markov}. In this context, the main considered problem is the synthesis of an \emph{optimal policy} to minimize some expected cost function on trajectories of the system. The minimization is over \emph{scheduling policies,} which are causal rules for selecting the next action given the past trajectory. Once this policy has been fixed, the resulting dynamics is a Markov Chain. Studies on (bi)simulation were more recently developed for MDP's~\cite{DBLP:journals/ai/GivanDG03}, 
%Givan, R., Dean, T., & Greig, M. (2003). Equivalence notions and model minimization in markov decision processes. Artificial Intelligence, 147, 163–223. 
and further developed to support robustness by defining metrics between finite MDP's~\cite{DBLP:journals/siamcomp/FernsPP11}.
%Metrics for Finite Markov Decision Processes. Norm Ferns, Prakash Panangaden, Doina Precup

%Since the mathematical literature on MDP does not consider issues of (bi)simulation and parallel composition, we prefer to keep the name of Simple Probabilistic Automata for our discussion in this section.
%}

In the following, $\Probas(Q)$ denotes the set of all probability distributions over the set $Q$.
Formally, we consider a tuple $P=(\alphabet,Q,q_0,\ra)$, where $\alphabet$ is the finite alphabet of actions, $Q$ is a finite  state space, $q_0{\in}Q$ is the initial state, and the probabilistic transition relation $\ra$ is defined in two different ways:
\beq
\mbox{\emph{Simple Probabilistic Automaton} (\spa)}&:& \ra\,\subseteq\,{Q}{\times}\alphabet{\times}\Probas(Q)
\label{otuighuith}
\\
\mbox{\emph{Probabilistic Automaton} ~\,(\pa)}&:& \ra\,\subseteq\,{Q}{\times}\Probas(\alphabet{\times}Q)
\label{lwreiukghlhliu}
\eeq
In the following definitions, relation $\leq^\Probas$ is the lifting, to probability distributions over $Q{\times}Q'$, of the relation $\leq$ over $Q{\times}Q'$---for the definition of the lifting $\leq^\Probas$, the reader can use Definition~\ref{hrgfuihsk} adapted by ignoring relations $\cons_1$ and $\cons_2$.

\subsubsection*{Details for \spa, model $(\ref{otuighuith})$} We write $\trans{q}{\action}{\mu}{P}$ to mean $(q,\action,\mu)\in\;\ra$ and $\produces{\mu}{q'}$ to mean that sampling $\mu$ returns next state $q'$. The sampling is: if $P$ is in state
$q{\in}Q$, performing $\action{\in}\alphabet$ leads to some target set
of probability distributions over $Q$, of which one is selected,
nondeterministically, and used to draw at random the next state $q'$. A \emph{simulation relation} is a relation $\leq\;\subseteq{Q{\times}Q'}$ such that, for any $q\leq{q'}$, the following holds: if $\trans{q}{\action}{\mu}{P}$, there exists $\mu'$ such that $\trans{q'}{\action}{\mu'}{P'}$ and $\mu\leq^\Probas{\mu'}$. The \emph{parallel composition} of \spa~\cite{DBLP:conf/concur/LynchSV03} is defined by: $P_1{\mpara}P_2=(\alphabet,Q,q_0,{\ra})$, where $\alphabet=\alphabet_1{\cup}\alphabet_2$, $Q=Q_1{\times}Q_2$, $q_0=(q_{0,1},q_{0,2})$, and $\trans{(q_1,q_2)}{\action}{\mu_1{\times}\mu_2}{}$ holds iff $\trans{q_i}{\action}{\mu_i}{i}$ for $i=1,2$.

\subsubsection*{Details for \pa, model $(\ref{lwreiukghlhliu})$}  We write $\trans{q}{}{\mu}{P}$ to mean $(q,\mu)\in\;\ra$ and $\produces{\mu}{(\action,q')}$ to mean that sampling $\mu$ jointly returns action $\action$ and next state $q'$.
The sampling is: $P$ being in state $q{\in}Q$ leads to some target set
of probability distributions over $\alphabet{\times}Q$, of which one is selected,
nondeterministically, and used to draw at random the next pair $(\action,q')$ of action and state. A \emph{simulation relation} is a relation $\leq\;\subseteq{Q{\times}Q'}$ such that, for any $q\leq{q'}$, the following holds: if $\trans{q}{}{\mu}{P}$, there exists $\mu'$ such that $\trans{q'}{}{\mu'}{P'}$ and $\mu\leq^\Probas{\mu'}$.

The \emph{parallel composition} $P=P_1{\mpara}P_2$ faces the following difficulty: there is a conflict between (1) the probabilistic choice of actions $\action_1$ and $\action_2$ in each component, and (2) the synchronization constraint on the pair $(\action_1,\action_2)$ possibly required by the parallel composition.

This difficulty does not exist if no synchronization constraint exists, e.g., if the composition of actions $\action=\action_1.\action_2$ is always defined. In this case, the parallel composition is straightforward: $\trans{(q_1,q_2)}{}{\mu_1{\times}\mu_2}{P}$ iff $\trans{q_i}{}{\mu_i}{P_i}$ holds for $i=1,2$. This kind of parallel composition does not capture synchronization, however.

In contrast, if strong synchronization is imposed, e.g., by requiring that $\action_1{=}\action_2$ whenever one of the two actions is shared by the two components---this is the policy followed in our model of Mixed Automata---, then the above conflict exists.
This conflict is usually resolved by adding a probabilistic scheduling policy specified through an auxiliary probability distribution, see the detailed discussion in~\cite{SokolovaV04} and references therein. A typical approach to compose the two transitions $\trans{q_i}{}{\mu_i}{P_i}\produces{}{(\action_i,q'_i)}, i=1,2$ is the following:
\begin{itemize}
	\item If synchronization constraint $\action_1=\action_2=\action$ happens to be satisfied, then the two transitions synchronize and $(q_1,q_2)$ leads to $(\action,(q'_1,q'_2))$ with probability $\mu_1(\action,q'_1)\times\mu_2(\action,q'_2)$.
	\item If both actions $\action_1$ and $\action_2$ are local $\action_i\not\in\alphabet_1{\cap}\alphabet_2, i=1,2$, then the synchronization constraint is not violated. However, since only one action is permitted at a time in \pa, one among the two transitions must be elected while the other one is freezed. This is achieved by tossing a (possibly biased) coin with parameter $\sigma\in(0,1)$, so that $(q_1,q_2)$ leads to $(\action_1,(q'_1,q_2))$ with probability $\mu_1(\action,q'_1)\times\sigma$ and $(q_1,q_2)$ leads to $(\action_2,(q_1,q'_2))$ with probability $\mu_2(\action_2,q'_2)\times(1-\sigma)$.
	\item Other cases are forbidden.
\end{itemize}
Collecting the outcomes that are not forbidden results in a transition of the form $\trans{q}{}{\bar{\mu}}{P}\produces{}{(\action,q')}$, where $\bar{\mu}$ is \emph{unnormalized}. A subsequent normalization is performed to get the final definition $\trans{q}{}{{\mu}}{P}\produces{}{(\action,q')}$ for the transitions of the parallel composition. The definition of this parallel composition thus requires specifying an additional probability distribution (the parameter $\sigma$ of the biased coin). Other variants for solving the same conflict all need such additional probability distributions---typically referred to as \emph{schedulers}.

\begin{discussion}[who comes first: nondeterminism or probability?] ~\\
\label{kruykuys} \rm 
%\emph{On the blending of nondeterminism and probability: who comes first?}
The following question arises~\cite{DBLP:journals/entcs/WangHR19}: should nondeterminism be resolved \emph{prior} or \emph{after} probabilistic sampling? Since the selection of the performed action followed by that of one probability from a subset of $\Probas(Q)$ (for \spa{s}), or the selection of one probability from a subset of $\Probas(\alphabet{\times}Q)$ (for \pa{s}) is performed prior to probabilistic sampling, both \spa\ and \pa\ models follow the first alternative. Our model of Mixed Automata follows a schyzophrenic approach: actions are selected first, leading to a Mixed System in which nondeterminism is resolved at last (See point~\ref{sdxhtrsd} in Definition~\ref{slergiuhpiu})---one can thus say that nondeterminism is resolved ``first-and-last''. As we shall see in our forthcoming comparison, the main difference between our model and models from the \pa\ family is not in this ``prior vs. after'' issue, but rather in our handling of conditioning and parallel composition.\eproof
\end{discussion}

\subsubsection*{Comparison results}

The following theorems relate \spa\ and \pa\ to \mmdps\ (proofs are constructive).
\begin{theorem}[SPA vs. \mmdps] \
  \label{erlgfuierhlpiu} \label{lesriughp}
	\begin{enumerate}
		\item  \label{reufygeky} There exists a mapping $P{\mapsto}{M_P}$, from \spa\ to \mmdps, preserving both simulation and parallel composition: $P_1{\leq}{P_2}$ iff $M_{P_1}{\leq} M_{P_2}$, whereas $M_{P_1{\mpara}{P_2}}$ and
  $M_{P_1}{\mpara}M_{P_2}$ are simulation equivalent.
	\item \label{reoiufui} There exists a reverse mapping $M{\mapsto}{P_M}$, from \mmdps\ to \spa, preserving simulation. No reverse mapping exists, however, that preserves parallel composition.
	\end{enumerate}
\end{theorem}
See Appendices~\ref{leirugheltorwu} and \ref{loitrghkliughip} for proofs of Statements~\ref{reufygeky} and~\ref{reoiufui} of this theorem. 
The two mappings $P{\mapsto}{M_P}$ and $M{\mapsto}{P_M}$ are not opposite, which makes it possible for the two statements not to contradict  each other. The non-existence of a reverse mapping $M{\mapsto}{P_M}$ preserving parallel composition highlights that the difference in the parallel compositions, for SPAs vs. for Mixed Automata, is deep.
\begin{theorem}[PA vs. \mmdps]
	\label{elriguhui}  There exists a mapping $P{\mapsto}{M_P}$, from \pa\ to \mmdps, preserving simulation. Parallel composition, however, is not preserved.
\end{theorem}
See Appendix~\ref{klueirghirugh} for a proof.

Due to Statement~\ref{reoiufui} of Theorem~\ref{erlgfuierhlpiu} and the existence of an embedding \spa$\ra$\pa~\cite{SokolovaV04} preserving simulation, a reverse mapping exists, from \mmdps\ to \pa.

%\myparagraph{Discussion}
%The difficulty in defining the parallel composition of \pa\ is the conflict between (1) the probabilistic choice of actions $\action_1$ and $\action_2$ in each component, and (2) the synchronization constraint $\action_1=\action_2$ required by the parallel composition. This conflict is usually resolved by adding a probabilistic scheduling policy specified through an auxiliary probability distribution~\cite{SokolovaV04}. In contrast, our parallel composition for \mmdps\ resolves this conflict by considering the conditional distribution given the synchronization constraint. Our approach suits probabilistic programming, which makes extensive use of prior and posterior probabilities. The \pa\ approach, however, does not.

In~\cite{SokolovaV04}, it is proved that \spa\ can be embedded into \pa, by simply ``pushing'' actions, from occurring prior to probabilistic choice to being part of probabilistic choice (in which case alternatives to emitting action $\action$ sum up to probability $1$). So, it seems unnecessary to study the embeddings \spa\,$\ra$\,Mixed Automata and \pa\,$\ra$\,{Mixed Automata} separately, since mapping the second one seems sufficient. This is, however, not a good idea, since the two embeddings differ, in that parallel composition is preserved for \spa\ but not for \pa.
%}

\begin{discussion}[More on comparing \spa/\pa\ and Mixed Automata]
	\label{kuygskuy}\rm 
 %regarding Probabilistic Programming}
So far Theorems~\ref{erlgfuierhlpiu} and~\ref{elriguhui} compare \spa/\pa\ and Mixed Automata regarding the core notions of \pa, namely simulation and parallel composition. 
%The core concepts of probabilistic programming are, however, different: conditioning and graphical models (causal, Bayesian networks; and acausal, factor graphs) are the core concepts. 
Conditioning is not at all considered in \pa\ theories---this indeed is the reason for them to have problems when handling synchronization in the parallel composition. We do not see how factor graphs can be reflected in \pa\ theories.
In contrast, these concepts 
%of probabilistic programming 
are naturally supported by our model of Mixed Automata. In addition, our model offers the classical concepts of \pa\ theories, namely simulation and equivalence.\eproof
%: these notions are useful at formalizing program equivalence in probabilistic programming, a delicate issue.
\end{discussion}

\section{{Other related work}}
\label{fdkjvblk}

So far we discussed work closely related to the different topics we covered. In this section we broaden our discussion by considering side topics relevant to our study.

Regarding semantic studies, we did not address \emph{denotational semantics}---our sampling (Definition~\ref{slergiuhpiu}) is an operational semantics. By denotational semantics, we mean a mathematical characterization of the set of all traces that can be produced by the considered system. The subject was indeed addressed in core mathematical probability theory---it was not called this way---with the Kolmogorov extension theorem: this theorem gives the denotational semantics of a sequence of independent identically $\mu$-distributed random variables as a probability space $(\Omega,\cF,\proba)$, where $\Omega$ is the set of trajectories, $\cF$ the associated product $\sigma$-algebra, and $\proba=\mu^\bN$, whose existence and uniqueness follows from this extension theorem. Since the $1970$'s, mathematicians in probability theory gave a denotational semantics (this term was not used) to stochastic differential equations in a very general setting, see e.g, the seminal paper~\cite{stroock1972}. In our context of nondeterministic/probabilistic dynamical systems, the task was not really investigated by mathematicians, and one should rather look at the literature closer to computer science. The seminal paper by Kozen \cite{DBLP:journals/jcss/Kozen81} defines two kinds of semantics of simple imperative probabilistic programs. The first semantics has finite horizon $[0,S]$ where $S$ is a stopping time (causally defined random time) and closely follows probability theory with its construction of probability spaces of program traces; the second semantics, advocated by the author, is more denotational, uses Scott-like techniques of continuous linear operators on a Banach space of measures, and supports infinite traces, see also~\cite{TixKeimelPlotkin2009,DBLP:conf/icse/GordonHNR14}. This approach was extended in~\cite{DBLP:conf/lics/JonesP89,DBLP:conf/birthday/KatoenGJKO15,DBLP:journals/pacmpl/DahlqvistK20} in order to provide semantics to the \emph{observe} statement present in most modern probabilistic programming languages.
%%%%%%%%%%%%%%%%%%%%%%
In~\cite{DBLP:conf/esop/BorgstromGGMG11}, the semantics of a functional language supporting mixtures of continuous and discrete distributions and dedicated to certainly terminating programs, is
specified as \emph{measure transformers}, describing how the program itself propagates the distribution of the probabilistic inputs.
%%%%%%%%%%%%%%%%%%%%%

Major probabilistic programming languages do offer \emph{recursion}~\cite{JSSv076i01,meent2018introduction}, all of them offer \emph{while loops}. These features raise the issue of possible non termination. Non terminating while loops are the essence of~\cite{rppl_pldi20}. We did not consider \emph{recursion} in its full generality, but only under the limited form of non-terminating time-recursion, with Mixed Automata. Actually, time-recursion is the most widely used form of recursion considered in statistics and learning.

\emph{Inference and learning} are the main concerns of probabilistic programming. Due to the generality of the considered models, Monte-Carlo based inference algorithms are preferred~\cite{DBLP:conf/icfp/BorgstromLGS16,DBLP:conf/fsttcs/HurNRS15,JSSv076i01,DBLP:journals/corr/abs-1206-3255,dippl}. Nondeterminism, which is supported by probabilistic languages, breaks the stationarity (or time-invariance) of the specified statistical models. This is a source of difficulties when invoking limit theorems of probability theory to support learning algorithms~\cite{DBLP:conf/fsttcs/HurNRS15}. We did not consider learning in this work. Clearly, our model of Mixed Automata would face the same challenge if inference were considered. Extension of model based IOCO \emph{testing} with probabilities was considered in~\cite{DBLP:journals/fac/GerholdS18}---this is a different subject than statistical testing in the sense of~\cite{lehmann2005testing}.

%\myparagraph{Variants of probabilistic automata} 
In Section~\ref{45oiuepiruh},
we have shown that  \mmdps\ subsume \pa. Tutorial~\cite{SokolovaV04}
investigates more variants of \pa. We conjecture that similar results hold for
these as well: mappings exist that preserve simulation but not parallel
composition.
Abstract Probabilistic Automata~\cite{DBLP:journals/iandc/DelahayeKLLPSW13} are an \emph{interface model,} aiming to support specification, not programming. In addition to parallel composition, Abstract Probabilistic Automata offer \emph{refinement} and possess Probabilistic Automata as their \emph{models}, two concepts irrelevant to our study.

%\myparagraph{Equivalence of probabilistic automata} A large part of Section 3.1 of book~\cite{meent2018introduction} is devoted to discussing examples of equivalent probabilistic programs. Regarding the approaches based on probabilistic automata, two probabilistic automata are said to be equivalent if each word is accepted with the same probability by both automata. It was shown by Tzeng in~\cite{tzeng1992polynomial} that quivalence for probabilistic automata is thus decidable in polynomial time. Alternatively, equivalence of probabilistic programs have been studied as equivalence of probabilistic lambda terms~\cite{DBLP:conf/icfp/BorgstromLGS16}, bisimulation of (labelled) Markov processes~\cite{desharnais2002bisimulation}, refinement order on imperative probabilistic programs~\cite{DBLP:series/mcs/McIverM05}.

%\myparagraph{Concurrent probabilistic models}
%

In our work we have considered only automata, whose dynamics is indexed by discrete time $n$. \emph{Equipping true concurrency models with probability} was classical for some net models. Free choice (or confusion free) nets are models for which this is rather simple; since choices remain local and statically defined, it is easy to turn them into probabilistic choices. For event structures with confusion, however, this is no longer the case: concurrency interferes with choice, making the latter dynamically defined. This makes it intricate, to equip choices with probabilities while maximally preserving concurrency. 
First constructions were proposed in~\cite{DBLP:conf/fossacs/AbbesB05,DBLP:journals/iandc/AbbesB06,DBLP:journals/tcs/AbbesB08,DBLP:conf/fossacs/AbbesB09}, based on the notion of \emph{branching cell,} capturing the above difficulty. Infinite event structures are supported (with restrictions) for which the law of large numbers is proved. 
Drawbacks are: 1) different sequences of events corresponding to the same configuration may be given different probabilities, and 2) the overall probability is globally defined, hence no parallel composition can be proposed. 
A different construction was proposed for occurrence nets in~\cite{DBLP:journals/lmcs/BruniMM19,DBLP:journals/tcs/BruniMM20}, addressing the above drawback. The net is augmented with ``negative places'', thus enforcing supplementary causalities with the result of deferring choices until they become local. Through the notion of statically defined \emph{s-cell,} the so augmented net can be given probabilistic choices meeting full concurrency, and parallel compositions of such nets is supported. In turn, the construction of the negative places works for finite nets only. In~\cite{DBLP:journals/tcs/BruniMM20}, a link of such augmented nets is established with Bayesian networks, thus providing a result similar to ours in Section~\ref{selrujgnhelru}.
Finally,~\cite{DBLP:journals/corr/abs-1305-5239,DBLP:journals/iandc/Abbes17,DBLP:journals/deds/Abbes19} study trace monoids by equipping them with probabilities derived from local specifications, using analytic combinatoric techniques. As far as we know, this is the only approach supporting true concurrency with probabilistic choice and parallel composition, for infinite traces. Concurrency makes everything definitely harder.

\section{{Conclusion}}
\label{eraspgfuisehpiu}

%In this paper w
We developed the model of Mixed (Probabilistic-Nondeterministic) Automata that subsumes  nondeterministic automata, probabilistic automata, and graphical probabilistic models. In a Mixed Automaton, transitions are triggered by actions and map states to Mixed Systems, from which the next state is sampled.

Mixed Systems are stateless and involve no dynamics. They combine nondeterminism and probability in a simple setting, providing an elegant theory of equivalence and a parallel composition. We proposed the notion of Mixed Kernel equipped with an incremental composition. We generalized Bayes formula by extending, to Mixed Systems and Mixed Kernels, the notions of marginal and conditional probabilities. The parallel composition of Mixed Systems naturally brings a notion of graphical structure, which subsumes Factor Graphs; similarly, the incremental composition of Mixed Kernels supports an extension of Bayesian Networks. Message passing algorithms allow for transforming tree-shaped Factor Graphs to Bayesian Networks, as already known for the classical notions. To summarize, our model extends graphical probabilistic models to a framework in which nondeterminism and probabilities can be freely combined. This framework also subsumes Dempster's belief theory. 
%This part of our model provides an elegant solution to various problems in probabilistic programming, such as program equivalence and program rewriting.

On top of Mixed Systems, we defined Mixed Automata and equipped them with a simulation and a parallel composition where probabilistic parts of systems can interact. This is in contrast to existing models of probabilistic automata, which do not support conditioning.
It would make sense to develop an \emph{interface theory} having Mixed Automata as models, along the lines of Abstract Probabilistic Automata~\cite{DBLP:journals/iandc/DelahayeKLLPSW13}. We believe that the simplicity of Mixed Systems makes them an interesting candidate for the semantics of probabilistic programs---there is still a long way to go before justifying this claim. 

%Mixed Automata are an adequate domain for a denotational semantics of Reactive Probabilistic Programming~\cite{rppl_pldi20}. They are flexible and simple and provide at hand all the basic tools for statistical inference---limit theorems of probability theory are still valid and can be invoked. So far, we did not provide a denotational semantics of program traces. This would require ``unfolding'' Mixed Automata to construct a mixed nondeterministic-probabilistic powerdomain for their sets of traces. This additional work would add very little regarding statistical tools.

To avoid technicalities, we decided to restrict ourselves to the consideration of finite or denumerable probability spaces. This makes the definition of support of a probability and conditional probability straightforward. Since conditioning is the heart of our approach, relaxing this restriction is far from obvious, with a deep revisiting of the notion of \emph{consistency} for Mixed Systems. In the last appendix of~{\cite{REPORT}}, we give hints for such an extension.
%\section{\albert{faire le RR}}
%We give in Appendix~\ref{reoiughofieurf} indications about how to develop such an extension.

%In this paper, we proposed a comprehensive mathematical grounding of probabilistic programming. 
We did not investigate decidability and complexity issues, however, neither we paid attention to effectiveness.
%Finding appropriate abstractions for making our approach practical even with continuous distributions, remains to be done---the background of synchronous languages (where all data types are supported by compilation algorithms) should help for this.
Handling constraints $\cons$ is the first difficulty. To reason on control, we could keep solving simple (e.g., Boolean) constraints, e.g., by distinguishing, in our model syntax, if-then-else statements. Other constraints may be abstracted by their associated directed or nondirected bipartite graph. Then, techniques such as the \emph{conditional dependency graphs} of synchronous languages~\cite{DBLP:journals/pieee/BenvenisteCEHGS03} could be adapted. 
%The second difficulty is the probabilistic part. Closed formulas for marginal and conditional distributions exist for specific families of probabilities. Beyond this, inference or estimation techniques are used in the statistics and AI literature~\cite{doi:10.1002/sim.3680,Plummer2003,JSSv076i01,meent2018introduction} and could be borrowed to our context.

We did not investigate either the design of learning and inference algorithms, a central motivation of probabilistic programming. When considering this subject, we would encounter the problem of correct Monte-Carlo sampling in learning algorithms, extensively studied in~\cite{DBLP:conf/fsttcs/HurNRS15}. In our context, this amounts to 1) identifying time-invariant model fragments, 2) applying limit theorems to them, and finally, 3) combining the results to derive learning algorithms for Mixed Systems or Automata models.

\myparagraph{Acknowledgements} The reviewers are gratefully thanked for pointing weaknesses and suggesting important bibliographical items while commenting the original version.

\bibliographystyle{plain}\bibliographystyle{spmpsci}      % mathematics and physical sciences
\bibliography{proba}
\clearpage
\appendix
%\input{MixedAutomata_appendix}
%\albert{REVOIR LES PREUVES}

In this supplementary material, we first collect all the missing proofs. Then, with reference to Footnote~\ref{eriuyoiu}, we include a short discussion of how to extend our model by relaxing the restriction that probability spaces should all be at most denumerable.

\section{Addendum and Proofs Regarding Mixed Systems}
%We reproduced here missing proofs for the new results.

\subsection{Comparison with imperative probabilistic programming, see Discussion~\ref{isedkufghsvjuy}}
\label{kuyfgksuy}
In this appendix, we compare our model of Mixed Systems with imperative probabilistic programming following the approach promoted by Mc Iver and Morgan~\cite{DBLP:series/mcs/McIverM05,DBLP:conf/isola/McIverM20}. This line of work addresses probabilistic extensions of Hoare logic for imperative programs, focusing on evaluating the probability of weakest preconditions of properties. We like to compare our approach with one aspect of this work, namely the modeling of the blending of probability and nondeterminism---this is only a minor aspect of the work of Mc Iver and Morgan, which focuses on decidability issues and computational cost of their proposed logic. 

\subsubsection{Demonic/angelic nondeterminism}
We chosed to base our comparison on a different work in the same direction:~\cite{DBLP:journals/toplas/ChatterjeeFNH18}, which provides the most extensive developement on \emph{demonic/angelic} blending of probability and nondeterminism in the language \textsc{Apps}. 
We do not claim to cover all aspects of \textsc{Apps}, since the focus of this reference is on the checking of almost sure termination using supermartingale techniques. Since our scope is more modest in this appendix, 
we will only develop an informal comparison based on the following example corresponding to Fig.\,2 of~\cite{DBLP:journals/toplas/ChatterjeeFNH18}, reproduced here as Figs.\,\ref{kufsgkuyf} and~\ref{krfhsrkufykuy}.
\begin{figure}[h!]
\begin{minipage}{4cm}
$
\bea{l}
x := 0; \\
\prog{\keyw{while}}\; x \geq 0 \;\prog{\keyw{do}} \\
   \hspace*{1em}\prog{\keyw{if}}\; \prog{\keyw{prob}}(0.6)\; \prog{\keyw{then}} \\
	 \hspace*{1em}\hspace*{1em}    \prog{\keyw{if}}\; \prog{\keyw{angel}} \; \prog{\keyw{then}} \\
		\hspace*{1em}\hspace*{1em}\hspace*{1em}	    x := x+1 \\
		\hspace*{1em}\hspace*{1em}	 \prog{\keyw{else}} \\
		\hspace*{1em}\hspace*{1em}\hspace*{1em}	    x := x-1 \\
		\hspace*{1em}\hspace*{1em}	 \prog{\keyw{fi}} \\
	 \hspace*{1em}\prog{\keyw{else}} \\
	   \hspace*{1em}\hspace*{1em}  \prog{\keyw{if demon then}} \\
		\hspace*{1em}\hspace*{1em}\hspace*{1em}	    x := x+1 \\
			\hspace*{1em}\hspace*{1em} \prog{\keyw{else}} \\
			\hspace*{1em}\hspace*{1em}\hspace*{1em}    x := x-1 \\
			\hspace*{1em}\hspace*{1em} \prog{\keyw{fi}} \\
   \hspace*{1em}\prog{\keyw{fi}} \\
\prog{\keyw{od}}	
\eea
\bea{c} ~\\
\left.\bea{c}\vspace*{12mm}
\eea
\right\}Q_1 \\ ~\\ 
\left.\bea{c}\vspace*{12mm}
\eea
\right\}Q_2
\eea 
$ 
\end{minipage} ~~
\begin{minipage}{7.8cm}
Verbatim from~\cite{DBLP:journals/toplas/ChatterjeeFNH18}: There is only one program variable $x$ and no random variables. There is a while loop, where given a probabilistic choice, one of two statement blocks $Q_1$ or $Q_2$ is executed. The block $Q_1$ (resp., $Q_2$) is chosen to execute stochastically w.r.t. the probabilistic choice ($Q_1$ \emph{is selected with probability} $0.6$). The statement block $Q_1$ (resp., $Q_2$) is an angelic (resp., demonic) conditional statement to either increment or decrement $x$.

Following~\cite{DBLP:journals/toplas/ChatterjeeFNH18}, call $Q_3$ the body of the $\prog{\keyw{while}}$ loop of this example: $\prog{\keyw{while}}\; x \geq 0 \;\prog{\keyw{do}}\;Q_3$.
\end{minipage}
\caption{Example of Fig.\,2 of~\cite{DBLP:journals/toplas/ChatterjeeFNH18}}
	\label{kufsgkuyf} 
	\centerline{\includegraphics[width=10cm]{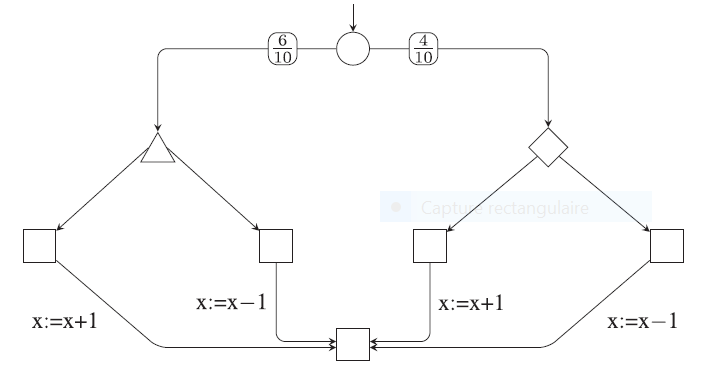}}
	\caption{Semantics: SGS (Stochastic Game Structure) of $Q_3$, Fig.\,6 of~\cite{DBLP:journals/toplas/ChatterjeeFNH18}. The execution begins with the probabilistic choice. The left branch (corresponding to $Q_1$) is selected according to demonic nondeterminism figured by a triangle, and the right branch (corresponding to $Q_2$) is selected according to angelic nondeterminism, figured by a diamond.}
	\label{krfhsrkufykuy}
\end{figure}

The program and its semantics are self-speaking. A key point here is the role of demonic and angelic nondeterminisms, and their combination in this program. Let us consider the post-condition 
\beq
P: x \mbox{ gets increased by one by performing } Q_3\,.
\label{weksuysviu}
\eeq
The question is: how do we assess $P$? Under demonic choice, $P$ is violated if there exists some branch in the nondeterministic choice under which $P$ is violated. Under angelic choice, $P$ is violated if for all branches in the nondeterministic choice, $P$ is violated. Inspecting Fig.\,\ref{krfhsrkufykuy} shows that $P$ is violated if and only if $Q_1$ is selected. Thus the probabilistic score that $P$ is violated is $0.6$---we do not use the term ``probability'' since $P$ combines both probabilistic and nondeterministic features, and cannot be given a true probability.

Can we cast this example into Mixed Systems?

\subsubsection{Casting this example to Mixed Systems?}
Consider the following attempt by defining the Mixed System $S_{Q_3}=\{(\Omega,\proba),\cons,\{x,x'\}\}$, where:
\begin{itemize}
	\item $\Omega=\{Q_1,Q_2\}$ and $\proba(\omega{=}Q_1)=0.6,\proba(\omega{=}Q_2)=0.4$; 
	\item Variable $x,x'$ correspond to the statuses of variable $x$ of $Q_3$ from Fig.\,\ref{krfhsrkufykuy}, before and after executing $Q_3$; the value of $x$ is assumed and the value of $x'$ will be established by sampling $S_{Q_3}$;
	\item It remains to define relation $\cons$ involving $\omega,x,x'$. To mimic Fig.\,\ref{krfhsrkufykuy}, we would like to write something like 
	\[\bea{rr}
	x'=\mathbf{if}\; \omega=Q_1 & \mathbf{then ~~ angel} \; x'\in\{x-1,x+1\}
	\\
	&\mathbf{else \; demon} \;  x'\in\{x-1,x+1\}
	\eea
	\]
\end{itemize}
Unfortunately, \textbf{angelic/demonic choice} are not concepts of our Mixed Systems model following Definition~\ref{slergiuhpiu}. With regard to probabilistic evaluation of state properties (item~\ref{ausytdfxuajy} of Definition~\ref{slergiuhpiu}), we could specify whether we use $\outproba$ (mirroring demonic) or $\inproba$ (mirroring angelic). Still, this does not allow to combine both alternatives for different parts of the system. 

We propose to refine  Definition~\ref{slergiuhpiu} so that both types of nondeterminism can be freely combined. Let us investigate this on the above example.
%\albert{
Consider the Mixed System 
\beq
\system=(\Omega,\proba,X,\cons)\,,
\label{kaujhfvbdsmh}
\eeq
where: 
\begin{itemize}
	\item $\Omega=\{Q_1,Q_2\}$ and $\proba(\omega=Q_1)=0.6,\proba(\omega=Q_2)=0.4$; 
	\item Variable $x,x'$ correspond to the statuses of variable $x$ of $Q_3$ from Fig.\,\ref{krfhsrkufykuy}, before and after executing $Q_3$; the value of $x$ is assumed and the value of $x'$ will be established by sampling $S_{Q_3}$;
	\item Relation $\cons$ is (yet informally) defined by
\beq
\omega\cons{x'} \mbox{ iff }\left\{\bea{lcr}
\omega=Q_1 
&\wedge& {\bf angel} ~x'\in\{x{-}1,x{+}1\} 
\\
\mbox{or}
\\
\omega=Q_2 
&\wedge& {\bf demon} ~ x'\in\{x{-}1,x{+}1\}
\eea\right.
\label{kuvakjvk}
\eeq
\end{itemize}
This definition for $\cons$ is informal, since keywords \textbf{demon} and \textbf{angel} have no mathematical meaning by themselves. We will give a semantics to (\ref{kuvakjvk}) by assigning, to each state predicate, a \emph{probabilistic score} $\proba^*$. More precisely, we define $\proba^*(\neg P)$, the probabilistic score of predicate $\neg P$,  by the following formula:
\beq\bea{rcr}
\proba^*(\neg P) &\eqdef& \consist{\proba}\left(\left\{\omega\mid\omega{=}Q_1 \wedge \exists x'{\in}\{x{-}1,x{+}1\}: \neg P\right\}\right) \\ [0.5mm]
&& + \consist{\proba}\left(\left\{\omega\mid\omega{=}Q_2 \wedge \forall x'{\in}\{x{-}1,x{+}1\}: \neg P\right\}\right)
\eea
\label{ksdjchvkj}
\eeq
In this formula, we give a semantics to \textbf{angel} in (\ref{kuvakjvk}) by using the existential quantifier, i.e., we use the outer probability to evaluate the corresponding state predicate; we give a semantics to \textbf{demon} in (\ref{kuvakjvk}) by using the universal quantifier, i.e., we use the inner probability to evaluate the corresponding state predicate. Now, for this example, $\consist{\proba}=\proba$ since, with relation (\ref{kuvakjvk}), for both choices $\omega=Q_1$ and $\omega=Q_2$, related values for state $x'$ exist. Formula (\ref{ksdjchvkj}) finally yields $\proba^*(\neg P)=0.6$.
%}

The above coding applies only to a restricted class of relations $\cons$. In formula (\ref{ksdjchvkj}), we exploited the fact that, in relation $\cons$ defined by (\ref{kuvakjvk}), a partition of $\Omega$ is performed first (probabilistic choice), and then, each branch of this choice involves a pure state predicate, independent from $\omega$. 

Here follow some hints to extend this link beyond the particular example. Our starting point is the semantics of \textsc{Apps}, which is expressed in terms of \emph{Stochastic Game Structures} (SGS), see Definition~2.3 of~\cite{DBLP:journals/toplas/ChatterjeeFNH18}. Since Mixed Systems do not support recursion, we consider only the subclass of SGS that are DAGs. Picking a probabilistic location $\ell$ of this SGS, we consider the maximal subgraph of this SGS that has $\ell$ as its only minimal location, and contains no other probabilistic location. For our example (\ref{kaujhfvbdsmh},\ref{kuvakjvk},\ref{ksdjchvkj}), this yields the whole SGS. For each such subgraph, a coding similar to (\ref{kaujhfvbdsmh},\ref{kuvakjvk},\ref{ksdjchvkj}) can be given. The partially ordered execution of the whole SGS is then mapped to a Bayesian network following Definition~\ref{lerifugeroigf}, and the incremental sampling of this Bayesian Network would correspond to the execution of the SGS as a game.

We preferred not to refine our Mixed System model with this additional feature, since, first, it applies only to a restricted class of relations $\cons$, and, second, we believe it to be incompatible with having a parallel composition. 

\subsection{Proof of Lemma~\ref{wjdetyfuuy}}
\label{riuygfttyiohjih}
\begin{proof}\rm  It is enough to prove the result for compressed systems.
  For $i=1,2$, let $\system_i\equiv\system'_i$ and let $\varphi_i$ be
  the bijections defining the two equivalences. We define
	\beqq \varphi(\omega,q_1\join{q_2}) &=&
  \left( (\omega'_1,\omega'_2),q'_1\join{q'_2} \right) \mbox{ where }
  (\omega'_i,q'_i) =\varphi_i(\omega_i,q_i), i=1,2
	\eeqq
	and we have
  to verify that $\varphi$ defines the desired equivalence between
  $\system\eqdef\system_1\mpara\system_2$ and
  $\system'\eqdef\system'_1\mpara\system'_2$. Using the fact that
  $\proba=\proba_1\times\proba_2$, we get
	 \[\bea{rl}
	 \cons_\proba=&\{
(p_1{\join}{p_2},\omega,q_1{\join}{q_2})
\mid
q_1{\compat}{q_2} \,\wedge\,
\omega_1\cons_1{q_1} \,\wedge\, \proba_1(\omega_1){>}0 \,\wedge\,
\omega_2\cons_2{q_2} \,\wedge\,  \proba_2(\omega_2){>}0
\} \\
=& \{
(p,\omega,q_1\join{q_2})
\mid
q_1\compat{q_2} \,\wedge\,
\omega_1\cons_{1\proba}q_1  \,\wedge\,
\omega_2\cons_{2\proba}q_2
\}
\eea
	 \]
	 Thus, for every $(p,\omega,q_1\join{q_2})\in\cons_\proba$, we have
%	 \[
$q'_1=q_1\compat{q_2}=q'_2  \mbox{ and }
\omega'_i\cons_{i\proba}q'_i, i=1,2$,
%\]
whence $\omega'\cons'_{\proba}{q'}$ and $\varphi$ is a bijection. Since $\proba'=\proba'_1\times\proba'_2$ we get $\proba'(\omega')=\proba(\omega)$, which finishes the proof.
\end{proof}

\subsection{Proof of Theorem~\ref{reoifueorwui}}
\label{lertughetpr}
\begin{proof}\rm  
We will repeatedly use notation (\ref{elriuei}). Without loss of generality we can assume that $\system$ is compressed. We first compress $\margin{Y}{\system}$ by considering the following equivalence relation, where $Z=X\setminus Y$ and $q_Y,q_Z$ are valuations for $Y$ and $Z$:
\beqq
\omega'\sim_Y\omega &\mbox{iff}&\forall q_Y:\left\{\bea{c}
		\exists q_Z:\omega\cons(q_Y,q_Z) \\ \Updownarrow \\ \exists q'_Z:\omega'\cons(q_Y,q'_Z)
\eea		\right. ~~ ;\mbox{ let $\omega_Y$ be the equivalence class of $\omega$.}
\eeqq
Let
\[
\cons_Y\eqdef\{(\omega_Y,q_Y)\in\Omega_Y\times{Q_Y}\mid\exists\omega\in\omega_Y:\omega\,\proj{Y}{\cons}\,q_Y\}
\]
 be the associated relation, and let $\proba_Y$ be the compressed probability defined by $\proba_Y(\omega_Y)=\sum_{\omega\in\omega_Y}\proba(\omega)$. Let us denote by
$$\system_Y=(\Omega_Y,\proba_Y,Y,\cons_Y)$$
 the resulting compressed system, and we recall that
$\consist{\Omega_Y}=\{\omega_Y\mid\exists{q_Y}:\omega_Y\cons_Y{q_Y}\}$.
In the sequel, we feel free to identify $\omega_Y\in\Omega_Y$, an element of the set of equivalence classes, with $\omega_Y$ seen as a subset of $\Omega$ saturated for $\sim_Y$. This way, a subset of $\Omega_Y$ can also be interpreted as a subset of $\Omega$.

To prove the theorem, we compare the two probabilistic semantics, namely: which state can be output and what is the outer probability of producing it.
By definition of the sequential composition of kernels, $\margin{Y}{\system}\bayestimes\cond{Y}{\system}$
\begin{enumerate}
	\item samples $\produces{\margin{Y}{\system}}{q_Y}$; and, then
	\item given $q_Y$, samples $(Y{=}q_Y)\mpara\system$.
\end{enumerate}
Regarding the relations governing the nondeterministic choice, the combination of these two steps is identical to $\cons$. Let $q_*$ be such that $\system\leadsto{q_*}$, implying that \mbox{$\margin{Y}{\system}\leadsto{q_{*Y}}$}, where $q_{*Y}\eqdef\proj{Y}{q_*}$.
%Using notation (\ref{ltriguhruiph}), l
Let us evaluate the outer probabilistic score of $q_*$ for the Bayesian network $\margin{Y}{\system}\bayestimes\cond{Y}{\system}$, i.e., the probability that $q_*$ is a possible outcome of sampling $\margin{Y}{\system}\bayestimes\cond{Y}{\system}$. We need to prove that it is equal to the probability that $q_*$ is a possible outcome of $\system$, namely $\consist{\proba}(\cons_{q_*})$---we used notation (\ref{ltriguhruiph}).
 To show this, we note the following:
\begin{enumerate}

	\item To output $q_*$ we first must output $q_{*Y}$, which amounts to selecting $\omega_Y$ such that $\omega_Y\cons_Y{q_{*Y}}$. Using (\ref{jdywetfdewjty}), (\ref{dfksjgksjhjh}) and notation (\ref{ltriguhruiph}), the probabilistic score of $q_{*Y}$, i.e., the probability that $q_{*Y}$ is a possible outcome of  $\margin{Y}{\system}$, is equal to 
	\beq
	\consist{\proba}_Y\bigl((\cons_Y)_{q_{*Y}}\bigr)\, \label{kdkfgkdysj}
	\eeq
	which is $>0$ since $\margin{Y}{\system}\leadsto{q_{*Y}}$.
	
	\item Then, we must select $\omega$ using $\system$, under the additional constraint that $\proj{Y}{q}$ ${=}q_{*Y}$, which requires that we sample $\omega\in\Omega$ under the constraint that $\omega\in\omega_Y$  for some $\omega_Y\in(\cons_Y)_{q_{*Y}}$.
	The corresponding probabilistic score is thus equal to the conditional probability 
	\beq
	\consist{\proba}\bigl(\cons_{q_*}\mid(\cons_Y)_{q_{*Y}}\bigr)\,, \label{shgfkgkhlfk}
	\eeq
	which is well defined since $\consist{\proba}_Y((\cons_Y)_{q_{*Y}})>0$.
	
	\item By (\ref{kslufhsdbkj}), the probabilistic score of $q_*$ is equal to the product of the two scores (\ref{kdkfgkdysj}) and (\ref{shgfkgkhlfk}):
	\[
	\consist{\proba}\bigl(\cons_{q_*}\mid(\cons_Y)_{q_{*Y}}\bigr)\;\consist{\proba}_Y\bigl((\cons_Y)_{q_{*Y}}\bigr)~=~\consist{\proba}\bigl(\cons_{q_*}\;{\cap}\;(\cons_Y)_{q_{*Y}}\bigr)~=~\consist{\proba}\bigl(\cons_{q_*}\bigr)\,,\]
	where the last equality follows from $\cons_{q_*}\subseteq(\cons_Y)_{q_{*Y}}\,.$
\end{enumerate}
This shows that $q_*$ possesses identical probabilistic semantics, for the left and right hand side  of Bayes formula.
\end{proof}

%\subsection{Proof of Lemma~\ref{keleuyo}}
%\label{ltroiyjhioru}
%\begin{proof}\rm 
%Statement~\ref{oeriutgho} is immediate by observing that, first, $\margin{X_1{\cup}{Y}}{\system_1{\times}\system_2}$ on the one hand, and $\system_1\times\margin{Y}{\system_2}$ on the other hand, possess identical probability spaces, namely $(\Omega_1,\proba_1)\times(\Omega_2,\proba_2)$, and, second, they possess identical relations $\proj{X_1\cup{Y}}{\cons_1{\wedge}\cons_2}=\cons_1\wedge\proj{X_1\cup{Y}}{\cons_2}=\cons_1\wedge\proj{Y}{\cons_2}$.
%%
%For statement~\ref{lirtughouhpi}, condition $\system_1\times\margin{Y}{\system_2}\equiv\system_1$ implies that $\system_2$ sets no further constraint on $\system_1$, ensuring that $\system_1\times{\system_2}$ is incremental.
%\end{proof}
\subsection{Proof of Theorem~\ref{eriuhpwiu}}
\label{litrughuihyutrd}
As a prerequisite, we need the following result:
	\begin{lemma}
		\label{keleuyo}
		Let $\system_1$ and $\system_2$ be any two Mixed Systems, and let $Y$ be a set of variables containing $X_1{\cap}X_2$. Then, we have: $
		\margin{X_1{\cup}Y}{\system_1{\mpara}\system_2}\equiv\system_1{\mpara}\margin{Y}{\system_2}
		$.
	\end{lemma}
\begin{proof}\rm 
		This is immediate by observing that, first, $\margin{X_1{\cup}{Y}}{\system_1\mpara\system_2}$ on the one hand, and $\system_1\mpara\margin{Y}{\system_2}$ on the other hand, possess identical probability spaces, namely $(\Omega_1,\proba_1){\times}(\Omega_2,\proba_2)$, and, second, they possess identical relations	$\proj{X_1\cup{Y}}{\cons_1{\wedge}\cons_2}=\cons_1\wedge\proj{X_1\cup{Y}}{\cons_2}=\cons_1\wedge\proj{Y}{\cons_2}$.\eproof
	\end{proof}
	
The proof of Theorem~\ref{eriuhpwiu} relies on the following lemma, which is a corollary of Bayes formula. This lemma provides the basic reasoning step of message passing algorithms:
\begin{lemma}
	\label{elriufek} Let $\system_1$, $\system_2$, and $Y$ be as in Lemma~$\ref{keleuyo}$. Then:
	\beq
	\system_1\mpara\system_2  &\equiv_P& \bigl(\system_1\mpara\margin{Y}{\system_2}\bigr)\bayestimes\,\cond{Y}{\system_2}\,.
	\label{lreiugfepiu}
	\eeq
\end{lemma}
%\subsection{Proof of Corollary~\ref{elriufek}}
%\label{lesirughepui}
\begin{proof}\rm 
For proving formula (\ref{lreiugfepiu}), we first apply Theorem~\ref{reoifueorwui} with $\system$ replaced by $\system_1{\mpara}\system_2$, which yields:
	$\system_1{\mpara}\system_2\equiv_P\margin{X_1{\cup}Y}{\system_1{\mpara}\system_2}\bayestimes\cond{Y}{\system_1{\mpara}\system_2}$. Then, by Lemma~\ref{keleuyo}, 
	\linebreak
	$\margin{X_1{\cup}Y}{\system_1{\mpara}\system_2}\equiv\system_1\mpara\margin{Y}{\system_2}$ and then we conclude by observing that
	$$\bigl(\system_1\mpara\margin{Y}{\system_2}\bigr)\bayestimes\cond{Y}{\system_1{\mpara}\system_2}~\equiv_P~\bigl(\system_1\mpara\margin{Y}{\system_2}\bigr)\bayestimes\cond{Y}{\system_2}\,,$$ since the outcome of $\system_1$ is determined by the left hand factor of ``$\bayestimes$''.\eproof
\end{proof}
Having proved this lemma, the proof of Theorem~\ref{eriuhpwiu} reproduces exactly the reasoning steps establishing the message passing algorithm mapping factor graphs to Bayesian Networks in the classical setting~\cite{Loeliger2004}; thus we only sketch here the argument of the proof. 
\begin{proof}\rm 
Since $\cG_\bsystem$ is a tree, a natural distance can be defined on the set of vertices of $\cG_\bsystem$ by taking the length of the unique path linking two vertices. Select an arbitrary system $\system_o$ as an {origin} and partially order other systems according to their distance to the origin, let $\preceq$ be this partial order. We have thus made $\cG_\bsystem$ a rooted tree, which we can see as a DAG. Then, the following two rules, known as \emph{message passing,} are considered:
\begin{enumerate}
	\item [R1:] Pick $\system\in\cG_\bsystem$, let $\ancestor{\system}$ be its (unique) ancestor in the tree and let $\ancestor{X}$ be the set of common variables of $\ancestor{\system}$ and ${\system}$. Then, let  $\prev{\system\,}$ denote the parallel composition of all strict ancestors of $\system$ in $\cG_\bsystem$ and let $\prev{X}$ be the set of variables of $\prev{\system\,}$. Using Bayes formula, factor $\system$ as
	\beqq
	\system ~~\equiv_P~~
	\margin{\ancestor{X}}{\system}\bayestimes\cond{\ancestor{X}}{\system} ~~\equiv_P~~
	\margin{\prev{X}}{\system}\bayestimes\cond{\prev{X}}{\system}\,,
	\eeqq
	where the second equivalence follows from the fact that additional variables belonging to $\prev{X}\setminus\ancestor{X}$ are not shared with $\system$.
	\item [R2:] Using formula (\ref{lreiugfepiu}) of Lemma~\ref{elriufek}, reorganize $\bsystem$ by rewriting
	\beqq
	\prev{\system\,}\mpara\system &~\equiv_P~&
	\left(\prev{\system\,}\mpara\margin{\prev{X}}{\system}\right)\,\bayestimes\;\cond{\prev{X}}{\system}\,.
	\eeqq
\end{enumerate}
Rules R1 followed by R2 are successively applied starting from the leaves of the tree, down to its root. The result is a Bayesian Network.
\end{proof}
%See Appendix~\ref{lesirughepui} for a proof.
	%\clearpage

\section{Proofs Regarding Mixed Automata}
\subsection{Proof of Lemma~\ref{egfuioehrpo}}
\label{elrgfuilyu}
\begin{proof}\rm 
The result is immediate if both $\system_1$ and $\system'_1$ are compressed, see Definition~\ref{lighlalegfr}. It is thus sufficient to prove the lemma for the following two particular cases: $\system_1$ compresses to $\system'_1$, and the converse.

Consider first the case: $\system_1$ compresses to $\system'_1$. Let $w(\omega_1,\omega_2)$ be the weighting function associated to the lifting $\system_1\NMPlift{\simu}\system_2$, and let $\proba'_1(\omega'_1)=\sum_{\omega_1\in\omega'_1}\proba_1(\omega_1)$ be the relation between $\proba'_1$ and $\proba_1$ in the compression of $\system_1$ to $\system'_1$. Then $w'(\omega'_1,\omega_2)=\sum_{\omega_1\in\omega'_1}w(\omega_1,\omega_2)$ defines the weighting function associated to the lifting $\system'_1\NMPlift{\simu}\system_2$. The other properties required to deduce $\system'_1\NMPlift{\simu}\system_2$ are immediate to prove.

Now, consider the alternative case: $\system'_1$ compresses to $\system_1$, with relation
\beq\bea{c}
\proba_1(\omega_1)=\sum_{\omega'_1\in\omega_1}\proba'_1(\omega'_1)
\eea
\label{ofoihiouhgliuh}
\eeq
 between $\proba'_1$ and $\proba_1$, where $\omega'_1\in\omega_1$ means that $\omega_1$ is the equivalence class of $\omega'_1$ with respect to relation $\sim$ defined in (\ref{eoguheogihio}) when compressing $\system'_1$. This case is slightly more involved since the weighting function $w'(\omega'_1,\omega_2)$ needs to be constructed. We need $w'(\omega'_1,\omega_2)$ to satisfy the following relations:
\begin{equation}\bea{rl}
\forall \omega'_1:&
\proba'_1(\omega'_1)=\sum_{\omega_2}w'(\omega'_1,\omega_2)
\\
\forall \omega_2:&
\proba_2(\omega_2)=\sum_{\omega'_1}w'(\omega'_1,\omega_2)
\\
[2mm]
\forall(\omega'_1,\omega_2;q_1):&
\left[\bea{c}
w'(\omega'_1,\omega_2)>0 \\ \omega'_1\,\cons'_1\,{q_1}\eea\right] \Ra \exists q_2:\left[\bea{c}\omega_2\,\cons_2\,{q_2} \\ q_1\,\simu\,{q_2}\eea\right]
\\
[-3mm]
\wemph{.}
\eea
\label{eruithpeuhpoh}
\end{equation}
Focus first on the first two lines of (\ref{eruithpeuhpoh}). The following calculation shows that
\beqq
w'(\omega'_1,\omega_2) &\eqdef& w(\omega_1,\omega_2) \times \frac{\proba'_1(\omega'_1)}{\proba_1(\omega_1)} \times \uun({\proba_1(\omega_1){>}0})\,,
\eeqq
where $\omega_1$ is such that $\omega'_1\in\omega_1$ and $\uun(B)$ equals $1$ if predicate $B$ is true and $0$ otherwise, yields a weighting function $w' $ satisfying the first two lines of (\ref{eruithpeuhpoh}):
\beqq
\sum_{\omega_2}w'(\omega'_1,\omega_2) &=& \sum_{\omega_2}\left(
w(\omega_1,\omega_2) \times \frac{\proba'_1(\omega'_1)}{\proba_1(\omega_1)} \times \uun({\proba_1(\omega_1){>}0})\right)
\\
&=& \frac{\proba'_1(\omega'_1)}{\proba_1(\omega_1)} \times \uun({\proba_1(\omega_1){>}0}) \times \sum_{\omega_2} w(\omega_1,\omega_2)  ~=~ \proba'_1(\omega'_1)
\\
\sum_{\omega'_1}w'(\omega'_1,\omega_2) &=& \sum_{\omega'_1}\left(
w(\omega_1,\omega_2) \times \frac{\proba'_1(\omega'_1)}{\proba_1(\omega_1)} \times \uun({\proba_1(\omega_1){>}0})\right)
\\
&=& \sum_{\omega_1}\left(
w(\omega_1,\omega_2) \times \frac{1}{\proba_1(\omega_1)} \times \uun({\proba_1(\omega_1){>}0})\right) \underbrace{\sum_{\omega'_1\in\omega_1}\proba'_1(\omega'_1)}_{=\proba_1(\omega_1)}
\\
&=& \sum_{\omega_1}w(\omega_1,\omega_2) ~=~ \proba_2(\omega_2)\,.
\eeqq
We move to the third line of (\ref{eruithpeuhpoh}). The conditions $w'(\omega'_1,\omega_2)>0$ and $\omega'_1\,\cons'_1\,{q_1}$ together imply
 $w(\omega_1,\omega_2)>0$ and $\omega_1\,\cons_1\,{q_1}$ where $\omega_1$ is the equivalence class of $\omega'_1$, i.e., $\omega'_1\in\omega_1$. The right hand side then follows since we have $\system_1\NMPlift{\simu}\system_2$. This finishes the proof.
\end{proof}

\subsection{Proof of Lemma~\ref{glrtukghtrllsdukfg}}
\label{wroifuwgopi}

\begin{proof}\rm 
Set $M'\eqdef M'_1\mpara{M'_2}$ and $M\eqdef M_1\mpara{M_2}$. Define the relation $\leq$ between $Q'$ and $Q$ by: $q'\leq{q}$ iff
$q'_1\leq_1{q_1}$ and $q'_2\leq_2{q_2}$. Let us prove that $\leq$ is a
simulation.
Let $q'$ be such that $\trans{q'}{\action}{\system'}{M'}$ for some consistent $\system'$. Then,
$q'=q'_1\join{q'_2}$ and $\system'=\system'_1\mpara\system'_2$. By
definition of the parallel composition, we have
$\trans{q'_i}{\action_i}{\system'_i}{M'_i}$ for $i=1,2$, with $\action_1\acompat\action_2$ and $\action=\action_1\ajoin\action_2$. Since
$q'_i\leq{q_i}$, we derive the existence (and uniqueness) of
consistent systems $S_i,i=1,2$ such that
$\trans{q_i}{\action_i}{\system_i}{M_i}$. Since $q=q_1\join{q_2}$ we
have $q_1\compat{q_2}$ and, thus, by definition of the parallel
composition, we deduce
\mbox{$\trans{r}{\action}{\system_1\mpara\system_2}{M}$}.
It remains to show that $\system_1\mpara\system_2$ is consistent. To
prove this, remember that $\system'=\system'_1\mpara\system'_2$ is
consistent. Thus, there exist compatible $q'_1$ and $q'_2$ such that
$\produces{\system'_i}{q'_i}, i=1,2$.  By definition of the
simulations $\leq_i$, we deduce that
$\produces{\system_i}{q_i}, i=1,2$, which shows that
$\system_1\mpara\system_2$ is consistent.
\end{proof}

\section{Proofs Regarding the comparison with Probabilistic Automata}

\subsection{Proof of Theorem~\ref{erlgfuierhlpiu} regarding Simple Probabilistic Automata}

\subsubsection{Statement~\ref{reufygeky} of Theorem~\ref{erlgfuierhlpiu}: from SPA to Mixed Automata}
\label{leirugheltorwu}
\begin{proof}\rm 
The sampling of \spa\ $P$ is: if $P$ is in state
$q{\in}Q$, performing $\action{\in}\alphabet$ leads to some target set
of probability distributions over $Q$, of which one is selected,
nondeterministically, and used to draw at random the next state $q'$.

We can reinterpret this sampling as follows: performing
$\action{\in}\alphabet$ while being in state $q{\in}Q$ leads to the
same target set of probability distributions over $Q$, that we use
differently. We form the direct product of all distributions belonging
to the target set and we perform one trial according to this
distribution, i.e., we perform independent random trials for all
probabilities belonging to the target set. This yields a tuple of
candidate values for the next state, of which we select one,
nondeterministically.

Clearly, these two samplings produce identical
outcomes. The latter is the sampling of
\mmdp\
\beq
M_P&=&(\alphabet,\{\xi\},q_0,\ra_P)\,,
\label{lerfgbleuy}
\eeq
defined as follows:
\begin{enumerate}
	\item Alphabet $\alphabet$ of $M_P$ is identical to that of $P$;
\item The unique variable $\xi$ of $M_P$ enumerates the values of $Q$, and initial state $q_0$ is identical to that of $P$; hence, $P$ and $M_P$ possess identical sets of states, related via the identity map;
\item $\ra_P$ maps a pair $(q,\action)\in{Q}{\times}\alphabet$ to the mixed system $S(q)=(\Omega,\Proba,\xi,q,\cons)$, where:
\begin{enumerate}
	\item 	$\Omega$ is the product of $n$ copies of $Q$, where $n$ is the cardinality of the set $\{\proba\mid(q,\action,\proba){\in}\ra\}$; thus, $\omega$ is an $n$-tuple of states: $\omega{=}(q_1,\dots,q_n)$.
  \item \label{lriueliu} $\Proba$ is the product of all probabilities belonging to \mbox{$\{\proba\mid(q,\action,\proba){\in}\ra\}$};
  \item \label{leguihuiliugh} Relation $C$ is defined by
	 $(\omega,q')\in\cons$ if and only if $q'\in\{q_1,\dots,q_n\}$.
\end{enumerate}
\end{enumerate}
So, we map \spa\ $P$ to \mmdp\ ${M_P}$, defined in (\ref{lerfgbleuy}).

\myparagraph{Mapping simulation relations}
Defining simulation relations for $\pa$ requires lifting relations,
from states to distributions over states. The formal definition for
this lifting, as given in Section\,4.1 of~\cite{Seg06-CONCUR},
corresponds to our Definition~\ref{hrgfuihsk}, when restricted to
purely probabilistic mixed systems.
The same holds for the strong simulation relation defined in
Section\,4.2 of the same reference: it is verbatim our
Definition~\ref{def:simulation}, when restricted to purely
probabilistic mixed systems. This proves the part of
Theorem~\ref{erlgfuierhlpiu} regarding simulation.

%\clearpage
\myparagraph{Mapping parallel composition}
We move to parallel composition, for which the reader is referred to~\cite{DBLP:conf/concur/LynchSV03}, Section 3. For $P_1=(\alphabet,Q_1,q_{0,1},\ra_1)$ and $P_2=(\alphabet,Q_2,q_{0,2},\ra_2)$ two PA, their parallel composition is  $P=P_1\mpara{P_2}=(\alphabet,Q_1\times{Q_2},(q_{0,1},q_{0,2}),\ra)$, where
\beq
\trans{(q_1,q_2)}{\action}{\proba_1 {\times} \proba_2}{}
 &\mbox{ iff }&
 \trans{q_i}{\action}{\proba_i}{i} \mbox{ for }i=1,2
\label{rtogtrilko}
\eeq
So, on one hand we consider the \mmdp\ $M_P$.
On the other hand, we consider the parallel composition of their images $M_{P_1}$ and $M_{P_2}$, namely $M=M_{P_1}\mpara M_{P_2}=(\alphabet,\{\xi_1,\xi_2\},(q_{0,1},q_{0,2}),\ra_{12})$. In $M$,
the state space is the domain of the pair $(\xi_1,\xi_2)$, namely $Q_1\times{Q_2}$,
and, since there is no shared variable between the two \mmdps, the transition relation $\ra_{12}$ is given by:
\beq
\trans{(q_1,q_2)}{\action}{\system_1 {\mpara} \system_2}{12}
&\mbox{ iff }& \trans{q_i}{\action}{\system_i}{i} \mbox{ for }i=1,2
\label{toguihtuiokuygf}
\eeq
We thus need to show that
\beq
\mbox{$M_P$ and $M$ are simulation equivalent.}
\label{gtiohrgtio}
\eeq
We will actually show that the identity relation between the two state spaces (both are equal to $Q_1\times{Q_2}$) is a simulation relation in both directions.

Observe first that (\ref{rtogtrilko}) and (\ref{toguihtuiokuygf}) differ in that the former involves a nondeterministic transition relation, whereas the latter involves a deterministic transition function, mapping states to mixed systems.
Pick $(q_1,q_2)\in{Q_1}\times{Q_2}$ and consider a transition for $M_P$:
\[
\trans{(q_1,q_2)}{\action}{S}{M_P}=((\Omega,\Proba),\xi,(q_{1},q_{2}),\cons)
\]
where we have, for $S$:
\begin{itemize}
	\item $\Omega$ is the product of $n_1$ copies of $Q_1$ and $n_2$ copies of $Q_2$, where, for $i=1,2$, $n_i$ is the cardinality of the set $\{\proba_i\mid(q_i,\action,\proba_i)\in\ra_i\}$, so that $\omega$ identifies $n_1\times{n_2}$-tuple of states: $\omega=(q_{11},\dots,q_{1n_1};q_{21},\dots,q_{2n_2})$;
	\item $\Proba$ is the product of all probabilities belonging to set \mbox{$\{\proba_1\times\proba_2\mid(q_i,\action,\proba_i)\in\ra_i\}$};
	\item $\xi$ has domain $Q_1\times{Q_2}$;
	\item $(\omega,(q''_1,q''_2))\in\cons$ if and only if \[
	(q''_1,q''_2)\in\{(q_{1i_1},q_{2i_2})\mid i_1\in\{1,\dots,n_1\} \mbox{ and } i_2\in\{1,\dots,n_2\}\}\,.
	\]
\end{itemize}
Next, pick $(q_1,q_2)\in{Q_1}\times{Q_2}$ and consider a transition for $M$, see (\ref{toguihtuiokuygf}). We need to detail what $S_1\mpara{S_2}=((\Omega' ,\Proba' ),\xi' ,(q_1,q_2),\cons')$ is. We have, for $S_1\mpara{S_2}$:
\begin{itemize}
	\item $\Omega'$ is still the product of $n_1$ copies of $Q_1$ and $n_2$ copies of $Q_2$;
	\item $\Proba'$ is the product $\Proba_1\times\Proba_2$, where $\Proba_i$ is the product of all probabilities belonging to set \mbox{$\{\proba_i\mid(q_i,\action,\proba_i)\in\,\ra_i\}$};
	\item $\xi'$ has domain $Q_1\times{Q_2}$;
	\item $(\omega,(q'_1,q'_2))\in\cons'$ if and only if \[
	(q'_1,q'_2)\in\{(q_{1i_1},q_{2i_2})\mid i_1\in\{1,\dots,n_1\} \mbox{ and } i_2\in\{1,\dots,n_2\}\}\,.
	\]
\end{itemize}
By associativity of $\times$, $\Proba'=\Proba$, whereas other items for $S$ on the one hand and other items for $S_1\mpara{S_2}$ on the other hand, are synctatically identical. Thus (\ref{gtiohrgtio}) follows.
\end{proof}

\subsubsection{Statement~\ref{reoiufui} of Theorem~\ref{erlgfuierhlpiu}: from Mixed Automata to SPA}
\label{loitrghkliughip}
\begin{proof}\rm 
Consider the following reverse mapping $M{\mapsto}{P_M}$, from {\mmdps} to \spa:
\begin{enumerate}
	\item \label{loeruighweruio} The alphabet $\alphabet$ of $P_M$ is identical to that of $M$;
\item The set of states $Q$ of $P_M$ is equal to the set of states of $M$, namely the domain of its set $X$ of variables;

\item \label{eorigfuhoiu} For $\system=(\Omega,\proba,X,p,\cons)$,  decompose  relation
$
\{(\omega,q)|\omega\cons{q}\}$ as $\bigcup_{\psi\in\Psi_\cons}\mbox{graph}(\psi)
$, where $\Psi_\cons$ denotes the set of all partial functions $\Omega\ra{Q}$, mapping each $\omega\in\exists{q}.\cons$ to some $q$ such that $\omega\cons{q}$. Then, we consider, for each $\psi\in\Psi_\cons$, the measure defined by $\psi[\proba](q)\eqdef\proba(\psi^{-1}(q))$, where $\psi^{-1}(q)=\{\omega|\psi(\omega){=}q\}$ ($\psi[\proba]$ is the image of $\proba$ by $\psi$), and we renormalize it by considering
\[
\frac{\psi[\proba]}{\psi[\proba](Q)}\,,
\]
 thus obtaining a probability distribution over $Q$. This defines a subset $\bP_\system\subseteq\Probas(Q)$ of probability distributions.

\item \label{oweriuthwpiweru} The transition relation of $P_M$ is defined as follows:
\beq
\ra_{P_M} &=&
\left\{
(p,\action,\mu) \mid \exists\system:(p,\action,\system)\in\;\ra_M \mbox{ and } \mu\in\bP_\system
\right\}
\label{lierughpeih}
\eeq
\end{enumerate}
Consider two \mmdps\ $M,M'$ and let $\leq$ be a simulation relation between their state spaces $Q$ and $Q	'$: $q\leq{q'}$ and $\trans{q}{\action}{\system}{M}$ imply the existence of $\system'\in\Systems(Q')$  such that  $\system\NMPlift{\leq}\system'$ and $\trans{q'}{\action}{\system'}{M'}$. We need to show that the same relation $\leq\;\subseteq{Q}{\times}{Q'}$ is also a simulation relation for \spa. Let $\trans{q}{\action}{\mu}{P_M}$ be a transition of \spa\ $P_M$. By (\ref{lierughpeih}), there exists a Mixed System $\system$ such that
$\trans{q}{\action}{\system}{M}$ and $\mu\in\bP_\system$. Since $\leq$ is a simulation relation for \mmdps, there exists $\system'\in\Systems(Q')$  such that  $\system\NMPlift{\leq}\system'$ and $\trans{q'}{\action}{\system'}{M'}$. Now, $\system\NMPlift{\leq}\system'$ expands as follows: There exists a weighting function $w:\Omega\times\Omega'\ra[0,1]$ such that the following two conditions hold:
\begin{enumerate}
	\item \label{peroiguto} For every triple $(\omega,\omega';q)$ such that $w(\omega,\omega')>0$ and $\omega\cons{q}$, there exists $q'$ such that $\omega'\cons'{q'}$ and $q\leq{q'}$;
	\item $w$ projects to $\proba$ and $\proba'$, respectively.
\end{enumerate}
Let $\psi\in\Psi_\cons$ be the selection function giving rise to $\mu$ following step~\ref{eorigfuhoiu}, meaning that $\mu$ is obtained by renormalizing $\psi[\proba]$. Select any $\omega\in\exists{q}.\cons$ and let $q=\psi(\omega)$. Select any $\omega'$ such that $w(\omega,\omega')>0$
and assign to it one $q'$ such that
$\omega'\cons'{q'}$ and $q\leq{q'}$ (such an $q'$ exists by the above Condition~\ref{peroiguto}). This selection procedure defines a selection function $\psi':\exists{q'}\cons'\ra{Q'}$, mapping the $\omega'$ of the above Condition~\ref{peroiguto} to $q'$, which in turn defines a probability distribution $\mu'$, obtained by renormalizing $\psi'[\proba']$. Consider the following weighting function over $Q{\times}Q'$:
\beqq
v &=& (\psi,\psi').w\,, \mbox { which expands as} \\
v(q,q') &=& w\{(\hat{\omega},\hat{\omega}')\mid\psi(\hat{\omega})=q,\psi'(\hat{\omega}')=q'\}
\eeqq
In particular $v(q,q')\geq{w(\omega,\omega')}>0$ by construction of $\psi,\psi'$, and $v$. Then, $v$ projects to $\mu$, and to $\mu'$:
\beqq
\forall q:
\sum_{q'}v(q,q')&=&\sum_{q'}w\{(\hat{\omega},\hat{\omega}')\mid\psi(\hat{\omega})=q,\psi'(\hat{\omega}')=q'\}
\\
&=& \sum_{\omega'}w\{(\hat{\omega},\hat{\omega}')\mid\psi(\hat{\omega})=q\}=\mu(q)
\eeqq
and
\beqq
\forall q':
\sum_{q}v(q,q')&=&\sum_{q}w\{(\hat{\omega},\hat{\omega}')\mid\psi(\hat{\omega})=q,\psi'(\hat{\omega}')=q'\}
\\
&=& \sum_{\omega}w\{(\hat{\omega},\hat{\omega}')\mid\psi'(\hat{\omega'})=q'\}=\mu'(q')
\eeqq
To summarize, we have constructed a probability distribution $\mu'$ such that \mbox{$\mu\leq^{\Probas}\mu'$} and $\trans{q'}{\action}{\mu'}{P_{M'}}$, showing that $\leq$ was also a simulation relation for \spa.
\end{proof}
To complete our proof, it remains to show the following lemma:
\begin{lemma}
	\label{eoriuhwpoiu} There is no mapping $M\mapsto{P_M}$ that preserves the parallel composition.
\end{lemma}
To support the above claim, we consider the following counter-example, where {$\system(\mem)$} indicates that $\system$ has previous state $\mem$:
\begin{example}\rm 
	\label{oeriguheol}
	Let $X=\{x_1,x,x_2\}$ be a set of three variables with finite domains $Q_{x_1},Q_{x},Q_{x_2}$. Consider the two systems $\system_i(p_i)=(\Omega_i,\proba_i,X_i,p_i,\cons_i), i=1,2$, where: $X_1=\{x_1,x\}$, $X_2=\{x,x_2\}$; $p_1\compat{p_2}$; $\Omega_i=Q_i$ with $Q_1=Q_{x_1}{\times}Q_{x}$ and $Q_2=Q_{x}{\times}Q_{x_2}$; $\proba_i$ is a probability over $\Omega_i$; and $\omega_i\cons_i{q_i}$ iff $\omega_i={q_i}$. Define
	\beq
	\bx:\Omega_1\uplus\Omega_2\ra{Q_x}, \mbox{ such that }\left\{\bea{l}
	\bx(\omega_1)=q \;\mbox{ if }\omega_1=(q_1,q) \\
	\bx(\omega_2)=q' \mbox{ if }\omega_2=(q',q_2)
	\eea\right.
	\label{elriguferhog}
	\eeq
	 System $\system_1$ amounts to defining the pair $(x_1,x)$ as random variables with joint distribution $\proba_1$; similarly, $\system_2$ amounts to defining the pair $(x,x_2)$ as random variables with joint distribution $\proba_2$.	We assume that the set of all
	 $q{\in}Q_x$ such that $\proba_1(Q_1{\times}\{q\}){>}0$ $\mbox{and }\proba_2(\{q\}{\times}Q_2){>}0$ is non empty.
	Forming the composition $\system_1{\mpara}\system_2$ yields the system
	$\system(p){=}(\Omega,\proba,X,p,\cons)$, where $X=X_1{\cup}X_2=\{x_1,x,x_2\}$, $Q=Q_{x_1}{\times}Q_{x}{\times}Q_{x_2}$, $p=p_1\join{p_2}$, $\Omega=\Omega_1\times\Omega_2$, $\proba=\proba_1\times\proba_2$, and $\omega\cons{(q_1,q,q_2)}$ iff $\omega_1\cons_1(q_1,q)$ and $\omega_2\cons_2(q,q_2)$.
	According to Definition~\ref{slergiuhpiu}, the sampling of $\system$ is the following:
	draw $(\omega_1,\omega_2)$ at random with the conditional distribution
$
\proba_1\times\proba_2\bigl((\omega_1,\omega_2)| \bx(\omega_1){=}\bx(\omega_2)\bigr)
$,
	where the map $\bx$ was defined in (\ref{elriguferhog});
	the resulting $(\omega_1,\omega_2)$ uniquely defines $(q_1,q,q_2)\in{Q}$ (no nondeterminism).
In words, the parallel composition $\system_1\mpara\system_2$ amounts to making the triple of variables $(x_1,x,x_2)$ to be random with the joint distribution $
\proba_1\times\proba_2\bigl((\omega_1,\omega_2)\mid \bx(\omega_1)=\bx(\omega_2)\bigr)
$.

Next, consider the \mmdp\ $M{=}(\{\action\},{X},q_0,\ra)$, where $X{=}\{x_1,x,x_2\}$, set $Q$ of states is defined accordingly $Q{=}Q_{x_1}{\times}Q_{x}{\times}Q_{x_2}$, and $\ra$ maps, through action $\action$, any state $p{\in}Q$ to the above system $\system(p)$. Similarly, we consider the two Mixed Automata $M_i=(\{\action\},{X_i},q_{i,0},\ra_i),i{=}1,2$, where $X_i$ is as above, $q_{i,0}$ is the projection of $q_0$ on $Q_i$, and $\ra_i$ maps, through action $\action$, any state $p_i{\in}Q_i$ to the above system $\system_i(p_i)$. We have $M=M_1{\mpara}M_2$.

The only candidate way of mapping $M_i$ to a \spa\ is by considering the two \spa\ $P_i$ with sets of states $Q_i$ and transition relation $\trans{p_i}{\action}{\proba_i}{i}$, where $\proba_i$ was defined above. Now, $P_1{\mpara}P_2$ has transition relation  $\trans{p}{\action}{\proba_1\times\proba_2}{}$, which reflects no interaction between the two \spa, so it cannot represent $M_1{\mpara}M_2$.
%\eproof
%
\end{example}

\subsection{Proof of Theorem~\ref{elriguhui} regarding Probabilistic Automata}
\label{klueirghirugh}
\begin{proof}\rm 
We consider the mapping $P\mapsto{M_P}=(\{1\},X,q_0,\ra_{M_P})$, from \pa\ to Mixed Automata, defined as follows:
\begin{enumerate}
	\item Alphabet $\{1\}$ is the trivial singleton (the particular element does not matter);
	\item $X=\{\xi_\alphabet,\xi_Q\}$, where the variables $\xi_\alphabet$ and  $\xi_Q$ enumerate $\alphabet$ and $Q$;
	\item \label{lseuioghu} Transition $\ra_{M_P}$ maps state $p$ to system $\system(p)=((\Omega,\proba),X,p,\cons)$, where
\begin{itemize}
	\item $\Omega=(\alphabet{\times}Q)^n$, where $n$ is the cardinal of the image of $p$ by transition $\ra$;
	\item $\proba$ is the product of all the distributions selected by transition $\ra$ starting from $p$;
	\item $\cons$ is the nondeterministic selection of one component of $\omega$.
\end{itemize}
\end{enumerate}
	We only need to prove the positive statement related to simulation. Consider a simulation relation for \pa\ $q\leq{q'}$. We need to prove that $\leq$ is also a simulation relation for \mmdps. Let $\mu$ be such that $(q,\mu)\in\;\ra$. Since $\leq$ is a simulation relation for \pa, there exists $\mu'$ such that $(q',\mu')\in\;\ra'$ and $\mu\leq^\Probas\mu'$.
	Let $\system$ and $\system'$ be the mixed systems to which $q$ and $q'$ are mapped by step~\ref{lseuioghu} of the mapping $P\mapsto{M_P}$. We have to prove that
	$\system\NMPlift{\leq}\system'$. For each $\mu$ such that $(q,\mu)\in\;\ra$, let
	the function $\chi:\Probas(Q)\ra\Probas(Q')$ select one $\mu'$ such that $(q',\mu')\in\;\ra'$ and $\mu\leq^\Probas\mu'$ and let $v_\mu$ be a weighting function associated to relation $\mu\leq^\Probas\mu'$. The following weighting function
	\beqq
	w(\omega,\omega') &\eqdef&
	\prod_{\mu:(q,\mu)\in\;\ra}v_\mu(q_\mu,q'_\mu)
	\eeqq
	where $(q_\mu,q'_\mu)\in{Q{\times}Q'}$, solves the problem.
	\end{proof}

	\section{{Extending Mixed Systems to continuous probabilities}}
\label{reoiughofieurf}
In this appendix, we indicate how to relax the restriction that the considered probability spaces should all be discrete and we discuss technical difficulties. A recommended reference on probability theory is~\cite{DellacherieMeyer1978}. The reader is invited to compare the following writing with the corresponding material of Section~\ref{uwtrdytrd}. We begin with some notations and prerequisites.

\subsection{{Notations and prerequisites on probability theory}}
\label{eliguhpiu}
For $P$ and $Q$ two sets, $P\times{Q}$ their product, and $A\subseteq{P}\times{Q}$, we denote by $\proj{P}{A}$ the projection of $A$ over $P$.

\myparagraph{Probability spaces}
$(\Omega,\cF,\proba)$ shall generically denote a probability space, i.e., $\Omega$ is a set, $\cF$ is a \sigalg\ over $\Omega$ (i.e., a subset of $2^\Omega$, containing $\emptyset$ and stable under complement, and countable unions and intersections), and $\proba$ is a probability (i.e., a countably additive function, from \sigalg\ $\cF$ to $[0,1]$, such that $\proba(\emptyset)=0$ and $\proba(\Omega)=1$). Let $p:(\Omega,\cF)\mapsto\{0,1\}$ be a measurable predicate, say that $p$ \emph{holds almost everywhere} if $\proba\{\omega\mid p(\omega)=1\}=1$.
For a measurable function $f:(\Omega,\cF)\mapsto(\bR_+,\cL)$, where $\cL$ is the Borel \sigalg\ over $\bR_+$, we write 
\beqq
\bE(f)&\eqdef&\int f(\omega) \proba(d\omega)\,.
\eeqq
For $(\Omega_i,\cF_i,\proba_i)_{i=1,2}$ two probability spaces, $\cF_1\times\cF_2$ is defined as the smallest \sigalg\ over $\Omega_1\times\Omega_2$ making the two projections measurable, and $\proba_1\times\proba_2$ shall denote the cartesian product of the two probabilities, characterized by $(\proba_1\times\proba_2)(A_1\times{A_2})=\proba_1(A_1)\proba_2(A_2)$, where $A_i\in\cF_i$. Infinite products of probabilities $\proba=\prod_{i\in{I}}\proba_i$ with arbitrary index set $I$ can even be defined; they are characterized by the equalities
$
\proba\left(
\prod_{i\in{I}}A_i
\right) = \prod_{i\in{I}}\proba_i(A_i)
$,
where all but a finite number of $A_i$ are equal to $\Omega_i$.

\myparagraph{Conditional expectations and conditional probabilities}
For $\cG\subseteq\cF$ a sub-\sigalg\ of $\cF$, and $X:(\Omega,\cF)\mapsto\bR_+$, measurable, there exists $Y:(\Omega,\cG)\mapsto\bR_+$, measurable, such that $\bE(X{\times}Z)=\bE(Y{\times}Z)$ for any measurable $Z:(\Omega,\cG)\mapsto\bR_+$. $Y$ satisfying the above properties is almost surely unique: $\proba(Y'\neq{Y})=0$ for any two such random variables. $Y$ is called the 
\beq
\mbox{\emph{conditional expectation of $X$ given $\cG$,} written $\bE(X\mid\cG)$.}
\label{liughpriu}
\eeq
 For $A\in\cF$, let $\uun_A$ denote the characteristic function of set $A$, which equals $1$ on $A$ and $0$ elsewhere; then, we write 
\beq
\proba(A\mid\cG)&\eqdef&\bE(\uun_A\mid\cG).
\label{leriukgjpofyugfio}
\eeq
If $B\in\cF$ satisfies $\proba(B)>0$ and $\cF_B=\{\emptyset,\Omega,B,\Omega{\setminus}{B}\}$ is the smallest \sigalg\ containing set $B$, then, the conditional expectation $f\eqdef\proba(A\mid\cF_B)$ is such that $f(\omega)=\proba(A\mid{B})=\frac{\proba(A{\cap}B)}{\proba(B)}$ for almost every $\omega\in{B}$. To show this, we form 
\[\bea{c}
\bE(\proba(A\mid\cG)\times\uun_B)=\bE\left(\frac{\proba(A{\cap}B)}{\proba(B)}\times\uun_B\right)
=\frac{\proba(A{\cap}B)}{\proba(B)}\times\underbrace{\bE\left(\uun_B\right)}_{=\proba(B)}=\proba(A{\cap}B)
\eea\]
with a similar result for $\uun_{\Omega-B}$, showing that the characterization of the conditional expectation is satisfied.
This establishes the link between conditional expectation and conditional probability in its elementary setting. 

For $\cG_2,\cG_1\subseteq\cF$ two sub-\sigalg{s}, and $X:(\Omega,\cF)\mapsto\bR_+$, measurable, we write $\bE(X\mid\cG_1\mid\cG_2)\eqdef\bE(\,\bE(X\mid\cG_1)\mid\cG_2)$. Let $B\in\cF$ and $\cG\subseteq\cF$, and let $\cF_B$ be the smallest \sigalg\ containing the set $B$; we write 
\beq
\proba(A\mid \cF_B\mid\cG)&\eqdef&\bE(\,\proba(A\mid\cF_B)\mid\cG)=\bE(\,\bE(\uun_A\mid\cF_B)\mid\cG)\,.
\label{liughltkuygfk}
\eeq

\myparagraph{Disintegration} 
Consider $(\Omega,\cF,\proba)$ and $\cG\subseteq\cF$ as before.
So far we have defined $\proba(A\mid\cG)$ as a $\cG$-measurable random variable, for a given $A\in\cF$. Can we take it as a \emph{transition probability} $P(\omega,A)$, i.e., a map such that $A\mapsto P(\omega,A)$ is a probability for $\omega$ fixed, and $\omega\mapsto P(\omega,A)$ is $\cG$-measurable for $A$ fixed? Here is the formalization:

\begin{definition}[{\cite{DellacherieMeyer1978,halmos1976measure}}]
	\label{weuyfgwekuy} 
Call \emph{disintegration}\footnote{Depending on the authors, disintegration is also called \emph{regular version of the conditional expectation}.} $\proba(A\mid\cG)$, where $A$ ranges over $\cF$, a map
$(\omega,A)\mapsto{P}(A\mid\omega)$ from $\Omega\times\cF$ to $[0,1]$ such that:
\begin{enumerate}
	\item \label{lieruhfpeioru} For $A$ fixed, $\omega\mapsto{P}(A\mid\omega)$ is $\cG$-measurable, and ${P}(A\mid\cdot)$ is a version of the conditional expectation $\proba(A\mid\cG)$; 
	\item For $\omega$ fixed, $A\mapsto{P}(A\mid\omega)$ is a probability.
	%\smallskip
\end{enumerate}
If $P(A\mid\omega)$ and $P'(A\mid\omega)$ are two such regular versions, then, probabilities $P(\cdot\mid\omega)$ and $P'(\cdot\mid\omega)$ must be equal outside a set of $\omega$ of $\proba$-probability zero.
\end{definition}

\myparagraph{Existence of a disintegration}
The existence of a disintegration is not always guaranteed. It is obvious for discrete probability spaces. It is not true in general, however, see, e.g.,~\cite{blackwell1956}. Failure to exist typically occurs when working with measurable spaces completed with subsets of zero probability sets. 
The existence of a a disintegration is only guaranteed under specific topological properties for the underlying set. Jirina Theorem is an example of broad sufficient topological conditions for the existence of regular versions for conditional expectations, see~\cite{DellacherieMeyer1978,halmos1976measure}. The Blackwell spaces, however, are adequate tools for this, so we introduce them next. For the following material, the reader is referred to~\cite{blackwell1956}.
%A \emph{Polish space} is a topological space that is homeomorphic to a complete metric space that has a countable dense subset;  products and disjoint unions of countably many Polish spaces are Polish. A topological space is a \emph{Lusin space} if it is homeomorphic to a Borel subset of a compact metric space; every Polish space is Lusin; products and disjoint unions of countably many Lusin spaces are Lusin. A \emph{Suslin space} is the image of a Polish space under a continuous mapping. We thus have the following inclusion between classes: Polish $\subset$ Lusin $\subset$ Suslin.

If $\Omega$ is a metric space, the smallest \sigalg\ containing all open sets of $\Omega$ is its \emph{Borel \sigalg,} denoted by $\cB$. Borel \sigalg\ $\cB$ is called \emph{separable} if there is a sequence $B_n{\in}\cB$ such that $\cB$ is the smallest Borel \sigalg\ containing all $B_n$. In particular, if $\Omega$ is a separable metric space, its Borel \sigalg\ is separable. The \emph{atoms} of $\cB$ are the sets $B\in\cB$ such that no proper nonempty subset of it belongs to $\cB$. Any two nonidentical atoms are disjoint and every Borel set  is a union of atoms. 

A metric space $A$ will be called \emph{analytic} if $A$ is the continuous image of the set of irrational numbers. The following properties hold, showing which cases are covered by this notion:
\begin{enumerate}
	\item If $A_n$ is a sequence of analytic sets in a metric space $\Omega$, then $\bigcup_nA_n$, $\bigcap_nA_n$ if nonempty, the product space $A_1{\times}A_2$ and the infinite product space $A_1{\times}A_2{\times}\cdots$, are
analytic sets.
\item  If $A$ is analytic, so is every Borel subset of $A$.
\item Every Borel set of the Euclidean $n$-space is analytic.
\item If $A,B$ are disjoint analytic subsets of a metric space $\Omega$, there is a Borel set $D$ of $\Omega$ such that $D\supset{A}$ and $D\cap{B}=\emptyset$.
\item If $f$ is a Borel-measurable mapping of an analytic set $A$ into a separable metric space $Q$, then $f(A)$, the range of $f$, is an analytic set.
\end{enumerate}
Pairs $(\Omega,\cB)$, where $\Omega$ is analytic and $\cB$ is its Borel \sigalg, are called \emph{Blackwell spaces}.\footnote{They are actually called ``Lusin spaces'' in~\cite{blackwell1956}, but the term ``Backwell spaces'' has been used since then in the literature to avoid the confusion with the hierarchy of Polish topological spaces.} The following two results are proved in~\cite{blackwell1956}:
\begin{theorem}
	\label{keirufgher}  For $(\Omega,\cB,\proba)$ a Blackwell space:
	\begin{enumerate}
	\item Two separable sub-\sigalg{s} of $\cB$ with the same atoms  are identical.
		\item For $\proba$ any probability on $(\Omega,\cB,\proba)$ and $\cB'$ any separable sub-\sigalg\ of $\cB$, there exists a disintegration for $\proba(A\mid\cB')$. 
	\end{enumerate}
\end{theorem}
In the following we will assume (unless otherwise stated) that all considered measurable spaces are Blackwell, so that Theorem~\ref{keirufgher} can be applied.

%\myparagraph{Blackwell spaces}
%By a theorem of David Blackwell~\cite{blackwell1956}, see also~\cite{SPS_1972__6__159_0} and~\cite{ASENS_1973_4_6_4_459_0} page 464, the class of \emph{Blackwell spaces} offer regular versions of conditional distributions as well as the possibility of comparing \sigalg\ by comparing their \emph{atoms,} i.e., measurable sets containing no strict measurable subset.

\subsection{{Definition and basic properties}}
\myparagraph{Relations}
Upper case letters $X,Y,Z$ shall denote finite sets of \emph{variables}, and variables are denoted by corresponding lower case letters $x,y,z\dots$ Let the domain of $x$ be denoted by $Q_x$ and be equipped with a \sigalg\ $\cG_x$; the domain of $X$ is $Q_X\eqdef\prod_{x\in{X}}Q_x$, equipped with the product \sigalg\ $\cG_X=\prod_{x\in{X}}\cG_x$. We will consider \emph{equations} (also called \emph{relations} or \emph{constraints}): an equation on $X$ identifies with its set of solutions, i.e., a measurable subset of $Q_X$; if $Y\subseteq{X}$, an equation on $Y$ can be seen as an equation on $X$. We consider \emph{systems of equations}, which are sets of equations implicitly composed via intersection.

\begin{definition}[Mixed System]
  \label{erfgeiuy} \it
	A \emph{Mixed System} is a tuple 
	\beq
		\system &=& \bigl((\Omega,\cF,\proba),(Q_x,\cG_x)_{x\in{X}},\cons\,\bigr),
	\label{keurygfekrygu}
	\eeq
	where $(\Omega,\cF,\proba)$ is a private probability space; $(Q_x,\cG_x)_{x\in{X}}$ is a finite set of measurable state spaces with product $
	(Q,\cG)\eqdef\prod_{x\in{X}}(Q_x,\cG_x)$, and $\cons\in{\cF{\times}\cG}$ is a measurable relation over ${\Omega{\times}{Q}}$. In the sequel, we also write $\omega\cons{q}$ to mean $(\omega,q)\in\cons$, and we identify the set of variables $X$ with the measurable state space $(Q_x,\cG_x)_{x\in{X}}$ it defines, thus we write
	\beq
	\system &=&\bigl((\Omega,\cF,\proba),{X},\cons\,\bigr),
	\label{lfeiyugeroiu}
	\eeq
for short instead of $(\ref{keurygfekrygu})$.
\end{definition} 

Defining the semantics of Mixed Systems in the general case requires some care, as the following example shows.
\begin{example}\rm [discussing consistency]
	\label{lkujygkregohro} Let $X$ and $Y$ be two real random variables with continuous joint distribution $\proba$. Formally, $\Omega=\bR^2$, $\cF$ is the Lebesgue \sigalg\ over $\Omega$, $\proba$ is a continuous probability over $(\Omega,\cF)$ and $X$ and $Y$ are the first and second coordinates of $\bR^2$. For $y$ a given value for $Y$, consider $\cons=\{(\omega,y)\mid\omega\in\Omega\}$ completing the definition of Mixed System $\system=\bigl((\Omega,\cF,\proba),{X},\cons\,\bigr)$. The intuition is that $\system$ models the conditional distribution of $(X,Y)$ given that $Y=y$. We would like this to be a consistent system, despite $\proba(Y{=}y)=0$. Thus, elementary Definition~\ref{slergiuhpiu} for the operational semantics cannot be used since it would lead to considering system $\system$ as inconsistent.
	
	For this case, the correction is easily guessed. The aim is that prior probability $\proba$ should be replaced by the posterior conditional distribution $\proba(\cdot\mid y)$, rather a disintegration for it. This amounts to making $y$ ``variable'' by considering a disintegration $\proba(A\mid\cF_Y)$ according to Definition~\ref{weuyfgwekuy}, where $\cF_Y\subset\cF$ is the \sigalg\ generated by random variable $Y$. Recall that the existence of such a disintegration is subject to topological conditions, see the comment following Definition~\ref{weuyfgwekuy}---such conditions are satisfied by this example. Then, we take $\consist{\proba}=\proba(\cdot\mid y)$ by taking the corresponding disintegration.
\eproof\end{example}

How can we extend this to general Mixed Systems? Informally, how can we make relation $\cons$ ``variable''? 
\begin{definition}[consistency and sampling]
	\label{elriugheilru} 
	Mixed System $\system$ is called \emph{consistent} if the following conditions hold:
	\begin{enumerate}
		\item \label{olitrughu} There exists a sub-\sigalg\ $\cH\subseteq\cF$ such that a disintegration $\proba(\cdot\mid\cH)$ exists; we denote by $\atom$ a generic atom of $\cH$, thus conditional probability $\proba(\cdot\mid\cH)$ becomes a function of atom $\atom$, so we write it $\proba(\cdot\mid{\atom})$;
		\item \label{trghoiejihojtloi} There exists a measurable relation $\Cons\in\cF\times\cG$ such that 
		\begin{enumerate}
				\item \label{etprwghtprgu}
		Relation $\cons$ takes the form $\cons=\Cons\cap({\atom}\times{Q})$	for some atom $\atom$ of $\cH$;
			\item \label{lritughuiyjui} $\proba(\consist{\Omega}\mid{\atom})>0$ where $\consist{\Omega}\eqdef$ \mbox{$\{\omega\mid\exists{q}:\omega\Cons{q}\}$}.
\end{enumerate}
	\end{enumerate}
	If $\system$ is consistent, define $\consist{\proba}$ by 
	\beq
	\consist{\proba}(A) &\eqdef& \frac{\proba(A\cap\consist{\Omega}\mid{\atom})}{\proba(\consist{\Omega}\mid{\atom})}
	\label{troihyjtpijpol}
	\eeq
	The \emph{sampling} of $\system$ consists in: $(1)$ drawing $\omega$ at random using $\consist{\proba}$, and $(2)$ nondeterministically selecting $q$ such that $\omega\cons{q}$. This two-step procedure is denoted by $\produces{\system}{q}$.
\end{definition}

The ``variable embedding'' of $\cons$ is the relation $\Cons$, from which $\cons$ is retrieved by selecting the atom $w$ at step~\ref{etprwghtprgu}.
\begin{example}\rm 
	\label{loerufgehro} 
	Consider the \plang\ program ``$S_1\,\progpara\,S_2\,\progpara\,S_3\,\progpara\,S_4$'' of the introduction. Now, the white noise model for $\prog{w}$ in $\prog{Noise}$ is truly Gaussian (or any other distribution on $\bR$, possibly continuous). The prior probability of this system is 
	\beq
	\mbox{prior proba}:\left\{\bea{l}
	\mathit{rf}_n\sim{\bf Bernoulli}(10^{-6}) \\ v_n\sim\mu \\ \mbox{by semantic convention, $\mathit{rf}$ and $w$ are independent}
	\eea\right.
	\label{deujhysetfduyt}
	\eeq
	and relation $\cons$ is the following system of equations:
	\beq	
	\cons&:&
	\left\{\bea{l}
	\prog{\keyw{observe}}\; u,y \\
x_0 = c_x \,,\, v_0=c_v \,,\, f_0 = \mfalse \\
x_n = \varphi(u_n,x_{n-1}) \\
y_n = \mbox{if } f_n \mbox{ then } \psi(x_n,v_n) \mbox{ else } x_n \\
f_n = (\mathit{rf}_n \mbox{ or } f_{n-1}) \mbox{ and not } \mathit{bk}_{n} 
	\eea\right.
	\label{elriufgheroui}
	\eeq
	The following observation is the key to handle the model (\ref{deujhysetfduyt},\ref{elriufgheroui}): if we forget for a while the first constraint $\bemph{\bf observe}\;y$ in (\ref{elriufgheroui}), then the resulting dynamical system can be seen as an input/output system with inputs $u,v,{\it rf}$, of which $u$ is a measured input, whereas $v,{\it rf}$ are random inputs: there is nothing unusual. In this i/o system, the prior probability is not subject to any constraint, hence the posterior probability equals the prior. 
	
	The difficulty comes with the consideration of the output constraint $\prog{\keyw{observe}}\;y$. This suggests taking for the instrumental \sigalg\ $\cH$ the \sigalg\ generated by $y$. Accordingly, we partition (\ref{elriufgheroui}) as
	\beq	
	\cons&:&
	\left\{\bea{rclll}
	&&\bemph{\bf observe}\;y &&\mbox{defining }\cH \mbox{, the \sigalg\ generated by }y, \\
	&&&&\mbox{whose atom is represented by a value for }y.\\
	y_n &=& f(v_n,\mathit{rf}_n) &&\mbox{defining }\Cons \mbox{ and consistency set } \consist{\Omega}, \mbox{ equal to }\Omega
	\eea\right.
	\label{elriufheruil}
	\eeq
	where $f$ is the function resulting from computing $y$ from the pair $(w,{\it rf})$ by using system of equations (\ref{elriufgheroui}) in which the first equation has been deleted (other variables are also computed). This defines the auxiliary relation $\Cons$ and we have $\consist{\Omega}=\Omega$. The \sigalg\ $\cF$ is generated by the pair $(w,{\it rf})$ of random variables, and $\cH$ is the \sigalg\ generated by $y\eqdef f(w,{\it rf})$. Atoms of $\cH$ consist of any reachable value for $y$. Then, $\proba(\cdot\mid\cH)=\proba(\cdot\mid{y})$ is the conditional distribution of the pair $(w,{\it rf})$ given a reachable value for $y$, and $\consist{\Omega}=\Omega$, showing that the considered system is consistent.\eproof
\end{example}
As a side result, the discussion of this example suggests how sampling can be performed in practice for \plang\ programs involving continuous distributions.
\end{document}